%% file: main.tex
\pdfoutput=1
\documentclass[letterpaper]{article}
\usepackage[totalwidth=480pt, totalheight=680pt]{geometry}

\usepackage{algorithm}
\usepackage{algorithmicx}
\usepackage{algpseudocode}
\usepackage{subcaption}

\usepackage{enumerate}

\usepackage{amsmath}
\usepackage{amssymb}
\usepackage{amsthm}
\usepackage{pifont}
\usepackage{booktabs}
\usepackage{xcolor}
\definecolor{darkgreen}{rgb}{0,0.5,0}
\usepackage{hyperref}
\hypersetup{
    unicode=false,          
    colorlinks=true,        
    linkcolor=red,          
    citecolor=darkgreen,    
    filecolor=magenta,      
    urlcolor=blue           
}
\usepackage[capitalize, nameinlink]{cleveref}

\usepackage{thm-restate}
\theoremstyle{plain}
\newtheorem{theorem}{Theorem}

\newtheorem{lemma}[theorem]{Lemma}

\theoremstyle{definition}
\newtheorem{definition}[theorem]{Definition}

\theoremstyle{remark}

\newtheorem*{remark*}{Remark}

\usepackage{tikz}
\usetikzlibrary{calc, graphs, graphs.standard, shapes, arrows, arrows.meta, positioning, decorations.pathreplacing, decorations.markings, decorations.pathmorphing, fit, matrix, patterns, shapes.misc, tikzmark}

\newcommand{\R}{\mathbb{R}}
\newcommand{\E}{\mathbb{E}}
\newcommand{\cA}{\mathcal{A}}
\newcommand{\cB}{\mathcal{B}}
\newcommand{\cI}{\mathcal{I}}
\newcommand{\cJ}{\mathcal{J}}

\newcommand{\cE}{\mathcal{E}}
\newcommand{\cO}{\mathcal{O}}
\newcommand{\cX}{\mathcal{X}}
\newcommand{\skel}{\mathrm{skel}}
\newcommand{\Pa}{\texttt{Pa}}

\newcommand{\Anc}{\texttt{Anc}}
\newcommand{\Des}{\texttt{Des}}
\newcommand{\OPT}{\texttt{OPT}}
\newcommand{\cost}{\texttt{cost}}

\newcommand{\wt}{\widetilde}
\DeclareMathOperator*{\argmax}{argmax}
\DeclareMathOperator*{\argmin}{argmin}

\title{New metrics and search algorithms for weighted causal DAGs}
\author{
Davin Choo\thanks{Equal contribution}\\
National University of Singapore\\
\texttt{davin@u.nus.edu}
\and
Kirankumar Shiragur\footnotemark[1]\\
Broad Institute of MIT and Harvard\\
\texttt{shiragur@stanford.edu}
}
\date{}

\begin{document}

\maketitle

\input{abstract}

\input{introduction}
\input{preliminaries}
\input{results}
\input{techniques}
\input{experiments}
\input{conclusion}

\section*{Acknowledgements}
This research/project is supported by the National Research Foundation, Singapore under its AI Singapore Programme (AISG Award No: AISG-PhD/2021-08-013).
Part of this work was done while the authors were visiting the Simons Institute for the Theory of Computing.
We would like to thank Saravanan Kandasamy, Jiaqi Zhang, and the ICML reviewers for valuable feedback and discussions.

\bibliography{refs}
\bibliographystyle{alpha}

\newpage
\appendix

\input{appendix-adaptive-versus-nonadaptive}
\input{appendix-meek-rules}
\input{appendix-additional-known-results}
\input{appendix-generalized-cost}
\input{appendix-blackbox-combination}
\input{appendix-proofs}
\input{appendix-experiments}

\end{document}

%% file: abstract.tex
\begin{abstract}
Recovering causal relationships from data is an important problem.
Using observational data, one can typically only recover causal graphs up to a Markov equivalence class and additional assumptions or interventional data are needed for complete recovery.
In this work, under some standard assumptions, we study causal graph discovery via \emph{adaptive interventions with node-dependent interventional costs}.
For this setting, we show that no algorithm can achieve an approximation guarantee that is asymptotically better than linear in the number of vertices with respect to the verification number; a well-established benchmark for adaptive search algorithms. 
Motivated by this negative result, we define a \emph{new benchmark} that captures the worst-case interventional cost for any search algorithm.
Furthermore, with respect to this new benchmark, we provide adaptive search algorithms that achieve logarithmic approximations under various settings: atomic, bounded size interventions and generalized cost objectives.
\end{abstract}

%% file: introduction.tex
\section{Introduction}
Causal discovery is a fundamental problem that has found applications in a wide range of fields, including biology/medicine/genetics \cite{king2004functional,cho2016reconstructing, tian2016bayesian, sverchkov2017review,rotmensch2017learning,pingault2018using, de2019combining}, epidemiology, philosophy \cite{reichenbach1956direction,woodward2005making,eberhardt2007interventions}, and econometrics \cite{hoover1990logic,rubin2006estimating}.  In most of these applications, directed acyclic graphs (DAGs) are used to represent the causal relationships and the goal is to recover the underlying causal graph from data.
It is well known that using observational data, the causal structure can only be learned up to its Markov equivalence class (MEC) and additional assumptions or interventional data is required for the recovery of the ground truth causal graph.
Here, we focus our attention on causal discovery using interventions.

There is a rich literature on causal discovery from interventional data, which can be broadly classified into two categories: adaptive \cite{shanmugam2015learning,greenewald2019sample,squires2020active,choo2022verification,choo2023subset} versus non-adaptive \cite{eberhardt2005number,eberhardt2006n,eberhardt2010causal,hu2014randomized} approaches.
Given an essential graph, \emph{non-adaptive} algorithms have to decide beforehand a collection of interventions such that \emph{any} plausible causal graph can be recovered while \emph{adaptive} algorithms can decide on interventions sequentially while using information gleaned from past interventions.
Adaptive algorithms are powerful and in some cases, the interventional cost of an optimal adaptive algorithm is exponentially better than any non-adaptive algorithms\footnote{For tree causal graphs, an adaptive algorithm only needs $\cO(\log n)$ interventions to recover it while any adaptive algorithm requires $\Omega(n)$ interventions in some cases. See \cref{sec:adaptive-is-exponentially-stronger}.}.
While the non-adaptive setting is pretty well understood even in the most general setting of node dependent vertex costs, researchers have only recently made progress on the adaptive front in the special case of unit vertex costs.
Unfortunately, unit vertex costs fail to capture many real world scenarios where performing interventions can have varying costs (e.g.\ it is less costly to force someone to sleep 8 hours than to force someone to run 10 miles), are unethical (e.g.\ force someone to smoke), or even practically impossible.
See \cite{king2004functional,sverchkov2017review,ness2018bayesian,lindgren2018experimental} for more applications of causal learning in settings where interventions have different costs.

\paragraph{Problem setup}
Motivated by the power of adaptivity and broad applicability of varying costs, in this work, we study causal discovery via \emph{adaptive} interventions with the goal of recovering the true underlying causal graph given the observational MEC while minimizing the total interventional cost when vertices may have \emph{differing} interventional cost.

Under standard assumptions of causal sufficiency, faithfulness and infinite sample regime, in addition to the \emph{search} problem defined above, recent works \cite{squires2020active,choo2022verification,choo2023subset} have also studied a related fundamental problem for the adaptive setting known as the \emph{verification} problem.
Given a MEC of an unknown ground truth causal graph $G^*$ and a graph $G$ from the MEC, the goal of the verification problem is determine whether $G$ is $G^*$.
By plugging in $G$ with $G^*$ in the verification problem, we see that the optimal solution to the verification is a natural lower bound for the search problem.
We denote the minimum \emph{size} and minimum \emph{cost} solutions to the verification problem as $\nu(G^*)$ and $\overline{\nu}(G^*)$ respectively.

For the special case of unit cost at each vertex, where $\nu(G^*) = \overline{\nu}(G^*)$, \cite{choo2022verification} recently gave an adaptive search algorithm that recovers $G^*$ by performing at most $\cO(\log n \cdot \nu(G^*))$ atomic interventions\footnote{Interventions that only involve a single vertex each.}, which is only a logarithmic factor worse than necessary.
Furthermore, they also argue that no algorithm can achieve an asymptotically better approximation ratio than $\cO(\log n)$ with respect to the $\nu(G^*)$ for all the causal graphs $G^*$.
In light of these results, it is natural to ask if such results also hold when vertices have different interventional costs.

While efficient algorithms for the verification and \emph{non-adaptive} search problems with varying vertex costs, and adaptive search problem for \emph{unit vertex costs} are known, to the best of our knowledge, there is no existing efficient \emph{adaptive search algorithm for varying vertex costs}.
Existing approaches for the unweighted setting do \emph{not} extend to the weighted setting due to two major difficulties: proving lower bounds for the benchmark, and designing algorithms that are competitive with it.
For the lower bound, existing methods only have guarantees with respect to the clique numbers of chain components and are oblivious to individual vertex costs.
On the other hand, existing adaptive search algorithms do not account for vertex weights\footnote{For instance, the algorithm of \cite{choo2022verification} searches for clique separators and intervenes on all the vertices in these clique separators.However, we show that (see \cref{thm:interventional-metric-lower-bound}) one cannot always intervene on the costliest vertex in a clique if we hope to have any theoretical guarantees; this is reflected in one of our algorithmic subroutines (see \cref{alg:dangling-subroutine}).}.
In fact, we can even show that the previously considered benchmark of the verification number is no longer meaningful in the context of weighted causal graphs.
More formally, we prove that no algorithm (even with infinite computational power) can achieve an asymptotically better approximation than $\cO(n)$ with respect to the verification cost $\overline{\nu}(G^*)$ for all ground truth causal graphs on $n$ nodes.
Therefore, $\overline{\nu}(G^*)$ is too strong and an unreasonable benchmark\footnote{A recent work on subset verification and search \cite{choo2023subset} also remarked that comparing against an algorithm that \emph{knows} $G^*$ can be overly pessimistic, and suggested that one should ``compare against the ``best'' algorithm that does \emph{not} know $G^*$''. This is consistent with our formulation of taking the maximum over all DAGs within the same Markov equivalence class.} to compare against in the weighted setting. Motivated by this negative result, 
we propose the following new benchmark
\[
\overline{\nu}^{\max}(G^*)
= \max_{G \in [G^*]} \overline{\nu}(G)
\]
which captures the intuition that any algorithm has to grapple with the worst-case causal graph in the given MEC\footnote{Our benchmark differs from the notion of separating systems studied in the non-adaptive search literature. See \cref{sec:benchmark-differs-from-separating-system}.}.
Using this new benchmark, we then provide adaptive search algorithms that are competitive against the $\overline{\nu}^{\max}(G^*)$.

\paragraph{Our main contributions are summarized as follows:}
\begin{enumerate}
    \item We argue that $\overline{\nu}(G^*)$ is not a good benchmark.
    \item Define a new benchmark $\overline{\nu}^{\max}(G^*)$ that captures the worst case interventional cost for any search algorithm.
    \item Provide an adaptive search algorithm that is $\cO(\log^2 n)$ competitive to $\overline{\nu}^{\max}(G^*)$ in the atomic setting.
    \item Extend our search results to bounded size interventions and for generalized cost function that enables an explicit trade-off between the \emph{number} and \emph{cost} of interventions, with $\nu^{\max}(G^*)$ and $\overline{\nu}^{\max}(G^*)$ being special cases.
\end{enumerate}

\paragraph{Outline of paper}
We give preliminaries and related work in \cref{sec:preliminaries}.
Results are stated in \cref{sec:results} and we provide a proof sketch of these results in \cref{sec:techniques}.
Some empirical results are shown in \cref{sec:experiments} and source code is provided in the supplementary materials.
Full proofs and further experimental details are provided in the appendix.

%% file: preliminaries.tex
\section{Preliminaries}
\label{sec:preliminaries}

We write $\{1, \ldots, n\}$ as $[n]$ and hide absolute constant multiplicative factors in $n$ using standard asymptotic notations.
For any set $A$, we denote its powerset by $2^A$.
Throughout, we denote the (unknown) ground truth DAG by $G^*$.

\subsection{Graph preliminaries}

Let $G = (V,E)$ be a graph on $|V| = n$ vertices.
We use $V(G)$, $E(G)$ and $A(G) \subseteq E(G)$ to denote its vertices, edges, and oriented arcs respectively.
$G$ is said to be directed or fully oriented if $A(G) = E(G)$, and partially oriented otherwise.
For $u,v \in V$, we write $u \sim v$ if these vertices are connected and $u \not\sim v$ otherwise.
We use $u \to v$ or $u \gets v$ to specify the arc directions.
For any subset $V' \subseteq V$ and $E' \subseteq E$, $G[V']$ and $G[E']$ denote the vertex-induced and edge-induced subgraphs respectively. Consider a vertex $v \in V$ in a directed graph, let $\Pa(v), \Anc(v), \Des(v)$ denote the parents, ancestors and descendants of $v$ respectively.

The \emph{skeleton} $\skel(G)$ of a (partially oriented) graph $G$ refers to the underlying graph where all edges are made undirected.
A \emph{v-structure} refers to three distinct vertices $u,v,w \in V$ such that $u \to v \gets w$ and $u \not\sim w$.
The cycle is directed if at least one of the edges is directed and all directed arcs are in the same direction along the cycle.
A partially directed graph is a \emph{chain graph} if it contains no directed cycle.
In the undirected graph $G[E \setminus A]$ obtained by removing all arcs from a chain graph $G$, each connected component is called a \emph{chain component}.
We use $CC(G)$ to denote the set of chain components, where each $H \in CC(G)$ is a subgraph of $G$ and $V = \dot\cup_{H \in CC(G)} V(H)$. For any partially directed graph, an \emph{acyclic completion} or \emph{consistent extension} refers to an assignment of edge directions to unoriented edges such that the resulting fully directed graph has no directed cycles; we say that a DAG $G$ is \emph{consistent} with a partially directed graph $H$ if $G$ is an acyclic completion of $H$.

DAGs are fully oriented chain graphs, where vertices represent random variables and the joint probability density $f$ factorizes according to the Markov property:
$
f(v_1, \ldots, v_n) = \prod_{i=1}^n f(v_i \mid \Pa(v))
$.
We can associate a (not necessarily unique) \emph{valid permutation} $\pi : V \to [n]$ to any (partially directed) DAG such that oriented arcs $(u,v)$ satisfy $\pi(u) < \pi(v)$ and unoriented arcs $\{u,v\}$ can be oriented as $u \to v$ without forming directed cycles when $\pi(u) < \pi(v)$.
A DAG is called a \emph{moral DAG} if it has no v-structures, in which case its essential graph is just its skeleton.
Moral DAGs have a unique source node (a node without incoming arcs), and any subgraph of it is also a moral DAG.

For any DAG $G$, we denote its \emph{Markov equivalence class} (MEC) by $[G]$ and \emph{essential graph} by $\cE(G)$.
DAGs in the same MEC $[G]$ have the same skeleton and essential graph $\cE(G)$ is a partially directed graph such that an arc $u \to v$ is directed if $u \to v$ in \emph{every} DAG in MEC $[G]$, and an edge $u \sim v$ is undirected if there exists two DAGs $G_1, G_2 \in [G]$ such that $u \to v$ in $G_1$ and $v \to u$ in $G_2$.
It is known that two graphs are Markov equivalent if and only if they have the same skeleton and v-structures \cite{verma1990,andersson1997characterization}.
An arc $u \to v$ is a \emph{covered edge} \cite{chickering2013transformational} if $\Pa(u) = \Pa(v) \setminus \{u\}$.

We now give a definition and result for graph separators.

\begin{definition}[$\alpha$-separator and $\alpha$-clique separator, Definition 19 from \cite{choo2022verification}]
Let $A,B,C$ be a partition of the vertices $V$ of a graph $G = (V,E)$.
We say that $C$ is an \emph{$\alpha$-separator} if no edge joins a vertex in $A$ with a vertex in $B$ and $|A|, |B| \leq \alpha \cdot |V|$. We call $C$ is an \emph{$\alpha$-clique separator} if it is an \emph{$\alpha$-separator} and a clique.
\end{definition}

\begin{theorem}[\cite{gilbert1984separatorchordal}, instantiated for unweighted graphs]
\label{thm:chordal-separator}
Let $G = (V,E)$ be a chordal graph with $|V| \geq 2$ and $p$ vertices in its largest clique.
There exists a $1/2$-clique-separator $C$ involving at most $p-1$ vertices.
The clique $C$ can be computed in $\cO(|E|)$ time.
\end{theorem}

\subsection{Interventions and verifying sets}

An \emph{intervention} $S \subseteq V$ is an experiment where all variables $s \in S$ are forcefully set to some value, independent of the underlying causal structure.
An intervention is \emph{atomic} if $|S| = 1$ and \emph{bounded} if $|S| \leq k$ for some $k>0$; observational data is a special case where $S = \emptyset$.
The effect of interventions is formally captured by Pearl's do-calculus \cite{pearl2009causality}.
We call any $\cI \subseteq 2^V$ a \emph{intervention set}.
An \emph{ideal intervention} on $S \subseteq V$ in $G$ induces an interventional graph $G_S$ where all incoming arcs to vertices $v \in S$ are removed \cite{eberhardt2005number}.
It is known that intervening on $S$ allows us to infer the edge orientation of any edge cut by $S$ and $V \setminus S$ \cite{eberhardt2007causation,hyttinen2013experiment,hu2014randomized,shanmugam2015learning,kocaoglu2017cost}.
For ideal interventions, an $\cI$-essential graph $\cE_{\cI}(G)$ of $G$ is the essential graph representing the Markov equivalence class of graphs whose interventional graphs for each intervention is Markov equivalent to $G_S$ for any intervention $S \in \cI$.
In \cref{sec:additional-known-results}, we give some well-known properties about interventional essential graphs.
Here, we highlight one such result that we will later use: intervening on a node $v$ in a moral DAG will orient any arcs $u \to w$ where $u$ is an ancestor of $v$ and $w$ is a descendant of $v$.

\begin{lemma}[Lemma 34 of \cite{choo2023subset}]
\label{lem:middle}
Let $G = (V,E)$ be a moral DAG.
Intervening on vertex $w$ orients all edges $u \to v$ with $w \in \Des(u) \cap \Anc(v)$.
\end{lemma}

A \emph{verifying set} $\cI$ for a DAG $G \in [G^*]$ is an intervention set that fully orients $G$ from $\cE(G^*)$, possibly with repeated applications of Meek rules (see \cref{sec:appendix-meek-rules}).
In other words, for any graph $G = (V,E)$ and any verifying set $\cI$ of $G$, we have $\cE_{\cI}(G)[V'] = G[V']$ for \emph{any} subset of vertices $V' \subseteq V$.
Furthermore, if $\cI$ is a verifying set for $G$, then $\cI \cup S$ is also a verifying set for $G$ for any additional intervention $S \subseteq V$.

\begin{definition}[Minimum size/cost verifying set]
\label{def:min-verifying-set}
Let $w$ be a weight function on intervention sets.
An intervention set $\cI$ is called a verifying set for a DAG $G^*$ if $\cE_{\cI}(G^*) = G^*$.
$\cI$ is a \emph{minimum size (resp.\ cost) verifying set} if $\cE_{\cI'}(G^*) \neq G^*$ for any $|\cI'| < |\cI|$ (resp.\ for any $w(\cI') < w(\cI)$).
\end{definition}

Fix a DAG $G$ and some upper bound $k \geq 1$ on the intervention size.
Then, the \emph{minimum verification number} $\nu_k(G)$ and the \emph{minimum verification cost} $\overline{\nu}_k(G)$ denote the size/cost of the minimum size/cost verifying set respectively.
Note that atomic interventions are a special case of bounded size interventions with $k = 1$.

Similar to \cite{kocaoglu2017cost,ghassami2018budgeted,lindgren2018experimental,addanki2020efficient}, we consider additive vertex costs where each $v \in V$ has an associated intervention cost $w(v)$ in this work.
The cost of an intervention $S \subseteq V$ is simply the sum of the vertices involved and the cost of an intervention set $\cI \subseteq 2^V$ is the sum of the intervention costs, i.e.\ $w(\cI) = \sum_{S \in \cI} w(S) = \sum_{S \in \cI} \sum_{v \in S} w(v)$. Since treating a bounded size intervention as $k$ individual atomic interventions can only recover more information, we aim to optimize the following generalized cost function to explicitly trade-off between the cost and size of the intervention set:
\begin{equation}
\label{eq:generalized-cost}
\max_{G \in [G^*]} \min_{\substack{\text{$\cI$ is a bounded}\\\text{size verifying}\\\text{set for $G$}}} \alpha \cdot w(\cI) + \beta \cdot |\cI| \quad \text{where $\alpha, \beta \geq 0$}
\end{equation}
Fix any integer $k \geq 1$ and DAG $G \in [G^*]$.
Minimizing \cref{eq:generalized-cost} yields $\overline{\nu}_k(G)$ when $\alpha = 1$ and $\beta = 0$ and $\nu_k(G)$ when $\alpha = 0$ and $\beta = 1$.
Prior work \cite{choo2022verification} studied the version of \cref{eq:generalized-cost} without the maximization over all DAGs in the Markov equivalence class for the verification problem, but not the search problem.

For any bounded size verification set $\cI \subseteq 2^V$, we write $\cost(\cI, \alpha, \beta, k) = \alpha \cdot w(\cI) + \beta \cdot |\cI|$ to denote its cost relative to \cref{eq:generalized-cost}.
For any deterministic adaptive search algorithm $A$ that produces intervention set $\cI$ for causal graph $G^*$, we define $\cost(A, G^*, \alpha, \beta, k) = \cost(\cI, \alpha, \beta, k)$.
For any randomized adaptive search algorithms, $\cost(A, G^*, \alpha, \beta, k)$ refers to \emph{expected cost}, where the expectation is over all the internal random choices made by $A$.
When restricting to atomic interventions with $k = 1$, we simply write $\cost(\cI, \alpha, \beta)$ and $\cost(A, G^*, \alpha, \beta)$.

\subsection{Related work in causal graph discovery}
\label{sec:related-work}

\cite{hyttinen2013experiment} was the first to apply the notion of separating systems from the combinatorics literature to causal discovery via \emph{non-adaptive} atomic interventions.
This was later extended to interventions of bounded size in the adaptive setting by \cite{hu2014randomized,shanmugam2015learning}.
Meanwhile, \cite{ghassami2018budgeted} studied the problem of maximizing the number of oriented edges given a fixed budget of non-adaptive atomic interventions.

There has been a flurry of works that explored adaptive search algorithms to fully orient a given essential graph while minimizing the number of interventions used \cite{he2008active,hu2014randomized,shanmugam2015learning,kocaoglu2017cost, lindgren2018experimental, greenewald2019sample,squires2020active,choo2022verification}.
More recently, \cite{choo2023subset} studied the problem of adaptive \emph{subset} search problem where one is only interested in learning the orientations of a subset of target edges.

In the context of \emph{weighted} interventions, one of the earliest works in the setting of additive vertex costs is \cite{kocaoglu2017cost}, where they show how to compute minimum cost non-adaptive bounded size interventions in polynomial time.
When the maximum number of interventions is fixed and one has to find the minimum cost intervention set, \cite{lindgren2018experimental} showed that it is NP-hard and provided a constant factor approximation algorithm.

For lower bounds, prior works such as \cite{squires2020active,porwal2021almost,choo2022verification} studied bounds for the verification number, where \cite{choo2022verification} eventually gave a complete characterization of $\nu_1(G^*)$ via the minimum vertex cover of the covered edges of $G^*$.

In the presence of latents (i.e.\ causal insufficiency), a common causal graph discovery objective is to recover an \emph{ancestral graph} \cite{richardson2002ancestral} instead of a DAG.
\cite{addanki2020efficient} recently studied this problem using non-adaptive interventions under additive vertex costs.

%% file: results.tex
\section{Results}
\label{sec:results}

Here we state all our main results of the paper.
Our first result suggests that comparing against $\overline{\nu}_1(G^*)$ may be too pessimistic for weighted causal graphs as we show that one cannot outperform an approximation of $|V(G^*)| = n$ in the worst case.
\cref{thm:n-apx-ratio} is information-theoretic and holds even for algorithms that have infinite computation power.

\begin{restatable}{mytheorem}{napxratio}
\label{thm:n-apx-ratio}
For any adaptive search algorithm $A$, deterministic or randomized, there exists a weighted causal graph $G^*$ such that $\cost(A, G^*, 1, 0) \in \Omega(n \cdot \overline{\nu}_1(G^*))$.
\end{restatable}

Recently, in the context of adaptive \emph{subset} search on \emph{unweighted} causal graphs, \cite{choo2023subset} showed that comparing against $\nu_1(G^*)$ in the presence of an \emph{adaptive} adversary\footnote{An adaptive adversary observes the interventions made by an adaptive algorithm and is allowed to ``change its mind'' by choosing the ground truth DAG among the set of all DAGs that are consistent with the already revealed information. We remark that \cref{thm:n-apx-ratio} holds even in the presence of a \emph{non-adaptive} adversary, and thus is a stronger result in this aspect.} leads to pessimistic bounds.
Instead, they propose to compare against a benchmark that does \emph{not} know $G^*$.
Independently motivated by \cref{thm:n-apx-ratio}, we propose the following natural benchmark metric that to compare search algorithms against:
\begin{equation}
\label{eq:new-metric}
\overline{\nu}^{\max}_k(G^*)
= \max_{G \in [G^*]} \overline{\nu}_k(G)
\quad \text{for any integer $k \geq 1$}
\end{equation}

As discussed in the introduction, our new benchmark captures the worst case cost for any optimal algorithm over the DAGs corresponding to a given essential graph.
That is, $\max_{G \in [G^*]} \cost(\texttt{ALG}, G, 1, 0, k) \geq \nu^{\max}_k(G^*) \geq \nu_k(G^*)$ for any fixed adaptive search algorithm $\texttt{ALG}$.
This benchmark also resolves the earlier raised concerns in \cite{choo2023subset} of ``comparing against the ``best'' algorithm that does \emph{not} know $G^*$''.

We next present an adaptive algorithm that is competitive with respect to \cref{eq:new-metric} when searching over weighted causal graphs using adaptive interventions.

\begin{restatable}{mytheorem}{weightedsearch}
\label{thm:weighted-search}
Fix an essential graph $\cE(G^*)$ corresponding to an unknown weighted causal DAG $G^*$.
\cref{alg:weighted-search} is a deterministic and adaptive algorithm that computes an atomic intervention set $\cI$ such that $\cE_{\cI}(G^*) = G^*$ and $w(\cI) \in \cO \left( \log^2(n) \cdot \overline{\nu}^{\max}_1(G^*) \right)$.
Ignoring the time spent implementing the actual interventions, \cref{alg:weighted-search} runs in $\cO(n \cdot \log^2(n) \cdot d \cdot m)$ time, where $d$ and $m$ are the degeneracy and number of edges of $\skel(\cE(G^*))$ respectively.
\end{restatable}

\cref{thm:weighted-search} is the first competitive \emph{adaptive} algorithm for the weighted setting.
We later show that one can combine \cref{alg:weighted-search} with a simple algorithm in a black-box manner to get $w(\cI) \in \cO(\min\{ n \cdot \overline{\nu}(G^*), \log^2(n) \cdot \overline{\nu}^{\max}_1(G^*) \})$.
The closest comparable result for \emph{weighted} graph search is the \emph{non-adaptive} search algorithm of \cite{lindgren2018experimental} discussed in \cref{sec:related-work}.
However, note that the size of a separating system in the non-adaptive setting can be much larger than $\overline{\nu}^{\max}_1(G^*)$ even when all vertices have unit weight: in the case where the essential graph is a path on $n$ vertices, $\overline{\nu}^{\max}_1(G^*) = \nu^{\max}_1(G^*) = 1$ while any separating system on this path has size $\Omega(n)$; see \cref{sec:benchmark-differs-from-separating-system}.

By tweaking the algorithm of \cref{thm:weighted-search} appropriately, our next result provides competitive guarantees with respect to the generalized cost function \cref{eq:generalized-cost}.

\begin{restatable}{mytheorem}{weightedsearchgeneralized}
\label{thm:weighted-search-generalized}
Fix an essential graph $\cE(G^*)$ corresponding to an unknown weighted causal DAG $G^*$.
Suppose $\cI^*_1$ and $\cI^*_k$ are an atomic and bounded size verifying sets minimizing \cref{eq:generalized-cost} such that $\cost(\cI^*_1, \alpha, \beta, 1) = \OPT_1$ and $\cost(\cI^*_k, \alpha, \beta, k) = \OPT_k$.
Then, \cref{alg:weighted-search-generalized} runs in polynomial time and computes a bounded size intervention set $\cI$ in a deterministic and adaptive manner such that $\cE_{\cI}(G^*) = G^*$, and\\
1. $\cost(\cI, \alpha, \beta, 1) \in \cO \left( \log^2 n \cdot \OPT_1 \right)$\\
2. $\cost(\cI, \alpha, \beta, k) \in \cO \left( \log n \cdot (\log n + \log k) \cdot \OPT_k \right)$.
\end{restatable}

We remark that \cref{alg:weighted-search} is a special case of \cref{alg:weighted-search-generalized} (given in \cref{sec:appendix-generalized-cost}) with $\alpha = 1$, $\beta = 0$, and $k = 1$.

\textbf{Why study bounded size interventions?}\\
One may be able to reduce the \emph{number} of interventions performed if one is allowed to intervene on more than one vertex per intervention.
For instance, to fully recover the orientations of a clique on $n$ nodes, it is known that $\Omega(n)$ atomic interventions are required.
However, if bounded size interventions are allowed, the lower bound is only $\Omega(n/k)$ and $\tilde{\cO}(n/k)$ interventions suffice \cite{shanmugam2015learning}.
As interventions take up actual wall-clock time and adaptivity demands sequentiality in the decision process, the ability to perform bounded size interventions (ideally in parallel) is particularly important for time-sensitive scenarios.

\textbf{Significance of our new metric $\overline{\nu}^{\max}(G^*)$}\\
As the previous benchmark of $\overline{\nu}(G^*)$ is overly pessimistic, many algorithms will ``look the same'' (albeit all with terrible competitive ratios) when compared against $\overline{\nu}(G^*)$ and it is natural to ask if there is a \emph{meaningful} comparison that differentiates them.
The new benchmark $\overline{\nu}^{\max}(G^*)$ serves this purpose: an algorithm that is more competitive to $\overline{\nu}^{\max}(G^*)$ would have better worst-case guarantees.
Intuitively, $\overline{\nu}^{\max}(G^*)$ shifts the comparisons away from an idealistic ``how much will an oracle that knows $G^*$ pay?" to a weaker ``how much will the best possible algorithm that only knows $[G^*]$ pay?''.
The latter question is more realistic/reasonable, and as we have argued, more meaningful in problem instances where vertices have differing costs.

Many adaptive search algorithms guarantee an $\tilde{\cO}(n)$ approximation to $\overline{\nu}(G^*)$ which only implies an $\tilde{\cO}(n)$ approximation to $\overline{\nu}^{\max}(G^*)$, while \cref{alg:weighted-search} provably obtains a logarithmic competitive ratio to $\overline{\nu}^{\max}(G^*)$.
For instance, the following naive algorithm incurs a cost of $\cO(n \cdot \overline{\nu}(G^*))$, but does not yield meaningful guarantees against $\overline{\nu}^{\max}(G^*)$: intervene on vertices one-by-one in an ascending weight ordering until the entire graph is oriented.
The proof of $\cO(n \cdot \overline{\nu}(G^*))$ is straightforward: the weight $w(v_{final})$ of the final intervened vertex $v_{final}$ is a lower bound for $\overline{\nu}(G^*)$, and we intervened at most $n$ vertices, each of cost lower than $w(v_{final})$, before $v_{final}$.
In \cref{sec:appendix-blackbox-combination}, we formally show how to combine the this naive algorithm with our algorithms of \cref{thm:weighted-search} and \cref{thm:weighted-search-generalized} in a black-box manner to retain the guarantees against $\overline{\nu}^{\max}(G^*)$ whilst \emph{simultaneously} ensuring that at most $\cO(n \cdot \overline{\nu}(G^*))$ cost is incurred.
The high-level idea is to simulate both algorithms in parallel until the causal graph is recovered.

%% file: techniques.tex
\section{Techniques}
\label{sec:techniques}

Here, we give some high-level technical ideas behind our algorithmic results (\cref{thm:weighted-search} and \cref{thm:weighted-search-generalized}).
We first describe how to lower bound the benchmark $\overline{\nu}^{\max}$ before giving our atomic adaptive algorithm (\cref{alg:weighted-search}).
Then, we explain how to tweak \cref{alg:weighted-search} to handle the generalized cost function with bounded size interventions.

\subsection{Lower bounding the benchmark}
\label{sec:lower-bounds}

For any interventional essential graph, we know that interventions within a chain component do not help to recover arcs within another chain component \cite{hauser2014two}.
Using this fact along with the proof strategy of \cite{choo2022verification} for lower bounding $\nu_1(G^*)$, we can show the following lower bound for $\overline{\nu}^{\max}_1(G^*)$.

\begin{restatable}{mytheorem}{interventionalmetriclowerbound}
\label{thm:interventional-metric-lower-bound}
For any DAG $G^*$ and its essential graph $\cE(G^*)$, we have
\[
\overline{\nu}^{\max}_1(G^*) \geq
\max_{\cI \subseteq V} \left\{ \sum_{\substack{H \in CC(\cE_{\cI}(G^*))\\|V(H)| \geq 2}} \max \left\{ \zeta^{(1)}_{\cI,H} , \zeta^{(2)}_{\cI,H} \right\} \right\}
\]
where we maximize over atomic intervention sets $\cI \subseteq V$,
\[
\zeta^{(1)}_{\cI,H} = \frac{1}{2} \cdot \max_{\text{clique } C \in H} \left\{ w(V(C)) - \max_{v \in V(C)} \left\{ w(v) \right\} \right\}
\]
and
\[
\zeta^{(2)}_{\cI,H} = \frac{1}{2} \cdot \max_{v \in V(H)} \bigg\{ \min \left\{ w(v), \gamma_{H,v} \right\} \bigg\}
\]
where $\gamma_{H,v} = \sum_{i=1}^t \max_{\substack{\text{clique } C_i:\\V(C_i) \subseteq V_i \cap N_H(v)}} \left\{ w(V(C_i)) \right\}$ with $V_1, \ldots, V_t \subseteq V(H)$ being vertex sets of the $t \geq 1$ disjoint connected components in $H[V(H) \setminus \{v\}]$.
\end{restatable}

The two cases of \cref{thm:interventional-metric-lower-bound} are pictorially illustrated by \cref{fig:case1} and \cref{fig:case2} respectively: we lower bound the $\zeta^{(1)}$ and $\zeta^{(2)}$ terms via the minimum cost vertex cover of the covered edges constructed in each figure.
In $\zeta^{(1)}$, $w(V(C)) - \max_{v \in V(C)} w(v)$ corresponds to the sum of the weight of all clique vertices \emph{except} the costliest one.
In $\zeta^{(2)}$, we check whether it is cheaper to intervene on a particular vertex $v$ or a clique in each connected component ``dangling'' from $v$.
The proof for both cases relies on being able to pick a ``worst case ordering'' on the vertices that are consistent with the given essential graph $\cE_{\cI}(G^*)$.

To argue that we can always fix such an ordering of our choice, we combine a ``patching'' result of \cite{choo2023subset} (see the second point of \cref{thm:moral-patching}) with the ``maximal clique picking'' procedure of \cite{wienobst2021polynomial} from the literature of MEC size counting.
Informally, \cite{wienobst2021polynomial} showed that all possible DAGs consistent with any given essential graph can be generated by repeating procedure: Picking a maximal clique $C$ to be the prefix maximal clique; orient all incident edges out of $C$; apply Meek rules until convergence; repeat.

\begin{restatable}{mylemma}{cliquepicking}
\label{lem:clique-picking}
Fix an interventional essential graph $\cE_{\cI}(G^*)$ corresponding to an arbitrary moral DAG $G^*$ and intervention set $\cI \subseteq 2^V$.
For any clique $C$ (not necessarily maximal) in any chain component of $\cE_{\cI}(G^*)$ and any permutation ordering $\pi$ on the vertices $V(C)$ of $C$, there exists a DAG $G$ consistent with $\cE_{\cI}(G^*)$ such that $u \to v$ if and only if $\pi(u) < \pi(v)$ for any two clique vertices $u,v \in V(C)$.
\end{restatable}

Roughly speaking, given a (not necessarily maximal) clique $C'$ and an ordering $\pi$, \cref{lem:clique-picking} follows by first picking a maximal clique containing $C'$ to be the prefix via ``maximal clique picking'', and then picking the vertices within $C'$ one by one according to $\pi$ via ``patching''.

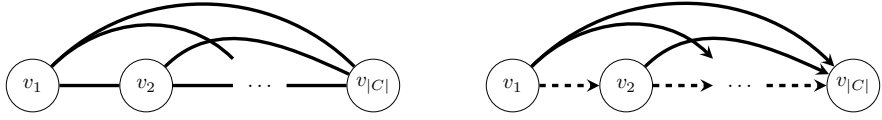
\begin{figure}[tb]
    \centering
    \resizebox{0.7\linewidth}{!}{%
    \input{figures/case1}
    }
    \caption{Consider clique $C$ involving vertices $v_1, v_2, \ldots, v_{|C|}$ with $w(v_1) \geq w(v_2) \geq \ldots \geq w(v_{|C|})$. By \cref{lem:clique-picking}, there exists a DAG consistent with this essential graph by choosing any vertex ordering of our choice within $C$ (see right figure). Covered edges $\{v_1 \to v_2, v_2 \to v_3, \ldots, v_{|C|-1} \to v_{|C|}\}$ are dashed.
    \vspace{-10pt}
    }
    \label{fig:case1}
\end{figure}

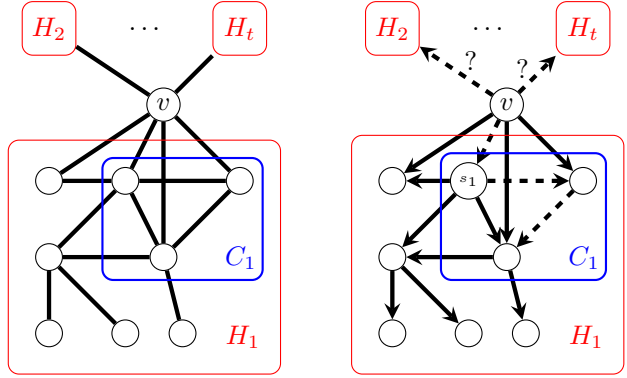
\begin{figure}[tb]
    \centering
    \resizebox{0.5\linewidth}{!}{%
    \input{figures/case2}
    }
    \caption{Consider the chain component $H$ with vertex $v$ and ``dangling'' connected components $H_1, H_2, \ldots, H_t$ in $H[V(H) \setminus \{v\}]$. Suppose $C_1$ is the costliest clique within $H[V(H_1) \cap N_H(v)]$. By \cref{lem:clique-picking}, there exists a DAG consistent with this essential graph by letting $v$ be the prefix within $H$ and then letting $C_1$ be the prefix within $H_1$ and choosing any vertex ordering of our choice within $C_1$ (see right figure). Covered edges are dashed. For any connected component $H_i$ with source node $s_i$, the arc $v \to s_i$ is a covered edge while $v \to u$ is \emph{not} a covered edge, for $u \in V(H_i) \setminus \{s_i\}$.
    \vspace{-10pt}
    }
    \label{fig:case2}
\end{figure}

Since vertex costs are additive and intervening on vertices in each bounded intervention set atomically can only recover more information about the causal graph, our next result provide lower bounds of the benchmark for bounded size interventions with its atomic counterpart.

\begin{restatable}{theorem}{relatek}
\label{thm:relate-k}
For any DAG $G^*$ and integer $k \geq 1$, $\overline{\nu}^{\max}_k(G^*) \geq \overline{\nu}^{\max}_1(G^*)$ and $\nu^{\max}_k(G^*) \geq \lceil \nu^{\max}_1(G^*) / k \rceil$.
\end{restatable}

\subsection{A competitive adaptive search algorithm}
\label{sec:upper-bounds}

Here, we present an adaptive search algorithm (\cref{alg:weighted-search}) that is competitive with respect to the lower bounds we presented in the previous section and proves \cref{thm:weighted-search}.
A known result of \cite{choo2023subset} (see \cref{thm:moral-patching}) allows us to ignore all oriented arcs in an interventional essential graph, without loss of generality, we can always assume that the underlying causal graph is a moral DAG. Given an essential graph of a moral DAG, \cref{alg:weighted-search} adaptively computes and outputs an atomic intervention set with cost competitive to $\overline{\nu}^{\max}_1(G^*)$.

\begin{algorithm}[tb]
\begin{algorithmic}[1]
\caption{Atomic weighted adaptive search.}
\label{alg:weighted-search}
    \Statex \textbf{Input}: Essential graph $\cE(G^*)$ of a moral DAG $G^*$ and weight function $w : V \to \R$.
    \Statex \textbf{Output}: Atomic intervention set $\cI$ s.t.\ $\cE_{\cI}(G^*) = G^*$.
    \State Initialize $i=0$ and $\cI_0 = \emptyset$.
    \While{$\cE_{\cI_{i}}(G^*)$ still has unoriented edges}
        \State Initialize $\cJ_i \gets \emptyset$
        \For{$H \in CC(\cE_{\cI_{i}}(G^*))$ of size $|H| \geq 2$}
            \State Find 1/2-clique separator $K_H$ via \cref{thm:chordal-separator}.
            \State Let $S = \{\{v\} : v \in V(K_H)) \setminus \{v_H\}\}$ be an atomic intervention set \emph{without} the costliest vertex
            \Statex\hspace{\algorithmicindent}\hspace{\algorithmicindent}$v_H = \argmax_{v \in V(K_H)} w(v)$.
            \State Intervene on $S$ and add $S$ to $\cJ_i$.
            \State Let $Z_{v_H} \in CC(\cE_{\cI_i \cup S}(G^*))$ be the chain component containing $v_H$ \emph{after} intervening on $S$.
            \If{$v_H$ is \emph{not} singleton in $Z_{v_H}$}
                \State Add output of $\texttt{ResolveDangling}$ to $\cJ_i$.
            \EndIf
        \EndFor
        \State Update $\cI_{i+1} \gets \cI_i \cup \cJ_i$ and $i \gets i+1$.
    \EndWhile
    \State \textbf{Return} $\cI_i$
\end{algorithmic}
\end{algorithm}

\begin{algorithm}[tb]
\begin{algorithmic}[1]
\caption{\texttt{ResolveDangling}}
\label{alg:dangling-subroutine}
    \Statex \textbf{Input}: Interventional essential graph $\cE_{\cI'}(G^*)$ for some intervention $\cI' \subseteq 2^V$, weight function $w : V \to \R$, and a chain component $H$ of $\cE_{\cI'}(G^*)$ that contains vertex $v \in V(H)$ with $t$ disjoint connected components $H_1, \ldots, H_t$ in $H[V(H) \setminus \{v\}]$.
    \Statex \textbf{Output}: Atomic intervention set $\cI$ such that all the outgoing edges of $v$ within $H$ are oriented in $\cE_{\cI \cup \cI'}(H)$.
    \State Initialize $\cI \gets \emptyset$.
    \If{$w(v) \leq \sum_{i=1}^t \max_{\text{clique $C$ in $H_i \cap N_H(v)$}} w(C)$}
        \State Intervene on $v$; Set $\cI \gets \cI \cup \{\{v\}\}$.
    \Else
        \For{$i \in \{1, \ldots, t\}$}
            \State Initialize $V' \gets V(H_i) \cap N_H(v)$.
            \While{$\skel(\cE_{\cI \cup \cI'}(G^*))[V']$ is \emph{not} a clique}
                \State Find a 1/2-clique separator $K$ of $H_i[V']$.
                \State Intervene on $K$ atomically; Add $V(K)$ to $\cI$.
                \State By \cref{lem:at-most-one-incoming}, find the chain component $Q$ with only incoming arcs into $K$ (if it exists).
                \If{$Q$ exists}
                    \State Set $V' \gets V(Q)$.
                \Else
                    \State Set $V' \gets \emptyset$.
                    \State \textbf{break}
                \EndIf
            \EndWhile
            \State Intervene on $V'$ atomically; Set $\cI \gets \cI \cup V'$.
        \EndFor
    \EndIf
    \State \textbf{Return} $\cI$
\end{algorithmic}
\end{algorithm}

On a high level, \cref{alg:weighted-search} is rather similar to the algorithm of \cite{choo2022verification}: both algorithms repeatedly apply \cref{thm:chordal-separator} to compute 1/2-clique separators $K_H$ so that we can break up the chain components and recurse on smaller sized chain components.
Such an approach is useful because the lower bound of \cref{thm:interventional-metric-lower-bound} ensures that the interventions done in each disjoint chain component can be summed up together to compare against $\overline{\nu}^{\max}_1(G^*)$.

Unfortunately, we \emph{cannot} fully intervene on \emph{all} vertices in the clique separators unlike the \emph{unweighted} adaptive search algorithm of \cite{choo2022verification}.
In the weighted setting, the costliest vertex $v_H$ in a clique separator $K_H$ may be enormous cost\footnote{In the extreme case, consider the example where $w(v_H) \gg \sum_{v \in V \setminus\{v_H\}} w(v)$. In this example, intervening on \emph{everything but $v_H$} in an atomic fashion trivially recovers \emph{any} DAG in $[G^*]$, i.e.\ $\overline{\nu}^{\max}_1(G^*) \leq \sum_{v \in V \setminus\{v_H\}} w(v) \ll w(v_H)$.} and the first case of \cref{thm:interventional-metric-lower-bound} only guarantees that we can remain competitive by intervening on \emph{all but $v_H$}.
As there may be connected components ``dangling off $v_H$'', we invoke \texttt{ResolveDangling} (\cref{alg:dangling-subroutine}) to ensure that the partites induced by the 1/2-clique separators will indeed be separated while we use the second case of \cref{thm:interventional-metric-lower-bound} to bound the cost of interventions used\footnote{If the computed clique separator in Step 8 of \cref{alg:weighted-search} involves only 1 node, $S = \emptyset$ and we break up the partites via \texttt{ResolveDangling} (\cref{alg:dangling-subroutine}).}.
Denoting an iteration of the while loop in \cref{alg:weighted-search} as a \emph{phase}, we show the following two lemmas about \cref{alg:weighted-search} whose combinations directly yields \cref{thm:weighted-search}.

\begin{restatable}{mylemma}{lognphasessuffice}
\label{lem:logn-phases-suffice}
\cref{alg:weighted-search} terminates after $\cO(\log n)$ phases.
\end{restatable}

\begin{restatable}{mylemma}{costofeachphase}
\label{lem:cost-of-each-phase}
Each phase in \cref{alg:weighted-search} incurs a cost of $\cO(\log(n) \cdot \overline{\nu}^{\max}_1(G^*))$.
\end{restatable}

The first logarithmic factor in \cref{lem:logn-phases-suffice} is due to the halving of the size of the chain components in each phase while the second logarithmic factor in \cref{lem:cost-of-each-phase} is due to the subroutine \texttt{ResolveDangling}, which tries to find a prefix clique in each dangling component. The description of the subroutine \texttt{ResolveDangling} is provided in \cref{alg:dangling-subroutine} and the guarantees of this subroutine are summarized in the following lemma.

\begin{restatable}{mylemma}{danglingsubroutine}
\label{lem:dangling-subroutine}
Fix an interventional essential graph $\cE_{\cI'}(G^*)$ corresponding to an unknown weighted causal moral DAG $G^*$ and some intervention $\cI' \subseteq 2^V$.
Let $H$ be a chain component of $\cE_{\cI'}(G^*)$ containing a vertex $v \in V(H)$.
Then, \cref{alg:dangling-subroutine} returns an atomic intervention set $\cI$ such that all the outgoing edges of $v$ within $H$ are oriented in $\cE_{\cI' \cup \cI}(H)$ and $w(\cI) \in \cO(\log n \cdot \overline{\nu}^{\max}_1(G^*))$.
\end{restatable}

In the remainder, we briefly discuss the technical idea behind \cref{alg:dangling-subroutine}. Let $\pi$ be an arbitrary consistent ordering of vertices corresponding to the unknown underlying DAG $G^*$. Suppose there are $t$ disjoint connected components $H_1, \ldots, H_t$ after removing $v_H$, then the cost incurred by \cref{alg:dangling-subroutine} is made competitive by using the second case of \cref{thm:interventional-metric-lower-bound}.
See \cref{fig:trace} for an illustration.
If $w(v_H)$ is at most the sum of weights of the heaviest clique across all $\{H_i \cap N_H(v)\}_{i \in [t]}$, then we can simply intervene on $v_H$ to disconnect the partites.
Otherwise, within each disjoint connected component $H_i$, we will search for and intervene on the source vertex $u_i = \argmin_{u \in V(H_i) \cap N_H(v)} \pi(u)$ within each $H_i$.
As we will search for $u_i$ using 1/2-clique separators, \cref{lem:at-most-one-incoming} guarantees we find it in $\cO(\log n)$ iterations.

\begin{restatable}{mylemma}{atmostoneincoming}
\label{lem:at-most-one-incoming}
Let $\cE_{\cI}(G)$ be the interventional essential graph of a moral DAG $G = (V,E)$ with respect to intervention set $\cI \subseteq V$.
Fix any chain component $H \in CC(\cE_{\cI}(G))$ and vertex $v \in V(H)$.
If $v$ is the source node of $H$, then there are no chain components of $\cE_{\cI \cup \{v\}}(H)$ with only incoming arcs into $v$ in $G$.
Otherwise, if $v$ is not the source node of $H$, then there is exactly one chain component of $\cE_{\cI \cup \{v\}}(H)$ with only incoming arcs into $v$ in $G$.
Furthermore, without further interventions, we can decide if such a chain component exist (and find it) in polynomial time.
\end{restatable}

Consider an arbitrary connected component $H_i$ amongst $H_1, \ldots, H_t$ with source vertex $u_i$.
If $\pi(v_H) < \pi(u_i)$, \cref{lem:middle} tells us that intervening on $u_i$ will orient all $v_H \to z$ arcs for $z \in H_i$ which disconnects $H_i$ from $v_H$, and thus from other components $H_j$.
Meanwhile, if $\pi(v_H) > \pi(u_i)$, then there is an arc from $H_i$ to $v_H$.
Note that intervening on $u_i$ may not disconnect $H_i$ from $v_H$, but it will disconnect $H_i$ from the other components\footnote{Without loss of generality, suppose $\pi(u_1) = \min_{i \in \{1, \ldots, t\}} \pi(u_i)$ and $\pi(v_H) > \pi(u_1)$. Orienting the arc $u_1 \to v_H$ triggers Meek rule R1 to orient all $v_H \to z$ arcs for $z \not\in H_1$, thus disconnecting the $H_i$'s from each other.}.
Nonetheless, we can still conclude that the resulting connected component has size at most halved since $H_i$ was part of a partite resulting from a 1/2-clique separator -- it may include at most an additional vertex $v_H$ but will not include $u_i$ (since we intervene on $u_i$).
We prove this formally in the appendix.

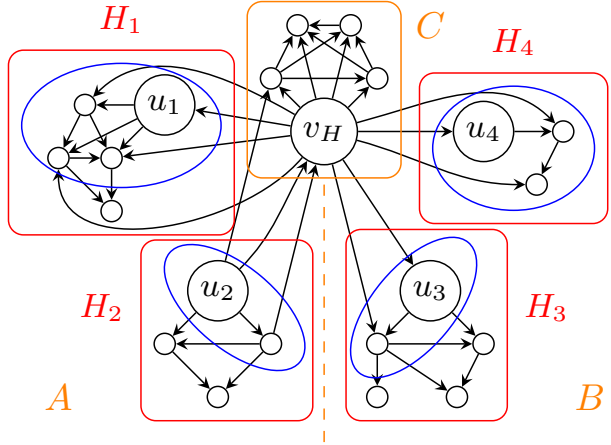
\begin{figure}[tb]
    \centering
    \resizebox{0.5\linewidth}{!}{%
    \input{figures/trace}
    }
    \caption{
    Consider the moral DAG $G^*$ above where $C$ is a 1/2-clique separator with vertices in $A = V(H_1) \cup V(H_2)$ and $B = V(H_3) \cup V(H_4)$, and $v_H$ is the costliest vertex in $C$.
    If we were to intervene on every single vertex in $C$, as per the algorithm of \cite{choo2022verification}, then the partites $A$ and $B$ will be disconnected.
    However, $v_H$ may be very costly and \cref{thm:interventional-metric-lower-bound} only gives approximation guarantees when intervening on $\cI = V(C) \setminus \{v_H\}$.
    Since the incident edges of $v_H$ may remain unoriented in $\cE_{\cI}(G^*)$, the partites may still be connected, e.g. the arcs $u_2 \to v_H \to u_3$ remain unoriented in $\cE_{\cI}(G^*)$.
    We say that connected components $H_1$, $H_2$, $H_3$, and $H_4$ are ``dangling'' from $v_H$ in $\cE_{\cI}(G^*)$.
    By \cref{lem:middle}, it suffices to intervene on all the source vertices $u_i$ in each $H_i$, and thus \texttt{ResolveDangling} searches for $u_i$ amongst the neighbors of $v_H$ in each $H_i$ (the blue ellipses).
    \vspace{-10pt}
    }
    \label{fig:trace}
\end{figure}

\subsection{Handling the generalized cost objective}
\label{sec:handling-generalized-cost}

To handle the generalized cost objective of \cref{eq:generalized-cost}, we make three algorithmic tweaks to the algorithms presented in the previous section.
Firstly, we change the condition of Line 2 in \cref{alg:dangling-subroutine} to account for the $\alpha$-$\beta$ trade-off in \cref{eq:generalized-cost}.
Secondly, to compute bounded size interventions to follow orient a clique, we apply the labelling scheme of \cref{lem:labelling-scheme} to use bounded sized interventions when intervening on cliques via the subroutine \texttt{CliqueIntervention} (\cref{alg:clique-intervention-subroutine}) with guarantees given in \cref{lem:clique-intervention-subroutine-generalized}.
Finally, when searching for a prefix clique in the while-loop \cref{alg:dangling-subroutine}, if one directly applies \texttt{CliqueIntervention} on each clique separator $K_{H_i}$, then one can show an $\cO(\log^2 n \cdot \log k)$ approximation.
To obtain an $\cO(\log n \cdot (\log n + \log k))$ approximation, we show that it suffices to partition $V(K_{H_i})$ into groups of size at most $k$ and intervening on them.
Note that this will \emph{not} necessarily orient all internal edges of $V(K_{H_i})$ but is sufficient for the purposes of locating a prefix clique.

\begin{lemma}[\cite{shanmugam2015learning}]
\label{lem:labelling-scheme}
Let $(n,k,a)$ be parameters where $k \leq n/2$.
There is a polynomial time labeling scheme that produces distinct $\ell$ length labels for all elements in $[n]$ using letters from the integer alphabet $\{0\} \cup [a]$ where $\ell = \lceil \log_a n \rceil$.
In every digit (or position), any integer letter is used at most $\lceil n/a \rceil$ times.
This labelling scheme is a separating system: for any $i,j \in [n]$, there exists some digit $d \in [\ell]$ where the labels of $i$ and $j$ differ.
\end{lemma}

\begin{restatable}{mylemma}{cliqueinterventionsubroutinegeneralized}
\label{lem:clique-intervention-subroutine-generalized}
Given a set of clique vertices $V(C) \subseteq V$ and integer $k \geq 1$, \cref{alg:clique-intervention-subroutine} returns a set $S \subseteq 2^{V(C)}$ such that each partite in $S$ has at most $k$ vertices.
When $k = 1$, $|S| = |V(C)|$ and each vertex appears exactly once in $S$.
When $k > 1$, $|S| \in \cO(\log k \cdot |V(C)| / k)$ and each vertex appears at most $\cO(\log k)$ times in $S$.
\end{restatable}

In terms of analysis, to lower bound \cref{eq:generalized-cost}, we individually lower bound the cost and size terms.
For instance,
\[
\max_{G \in [G^*]} \min_{\substack{\text{$\cI$ is a bounded}\\\text{size verifying}\\\text{set for $G$}}} \alpha \cdot w(\cI) + \beta \cdot |\cI|
\geq \max_{G \in [G^*]} \min_{\substack{\text{$\cI$ is a bounded}\\\text{size verifying}\\\text{set for $G$}}} \alpha \cdot w(\cI)
= \alpha \cdot \overline{\nu}^{\max}_k(G^*) \;.
\]
Similarly, $\beta \cdot \nu^{\max}_k(G^*)$ is also a lower bound.
Then, we can further use \cref{thm:relate-k} to lower bound $\overline{\nu}^{\max}_k(G^*)$ and $\nu^{\max}_k(G^*)$ via $\overline{\nu}^{\max}_1(G^*)$ and $\nu^{\max}_1(G^*)$ respectively.
The additional $\log k$ term for non-atomic interventions occurs because of the multiplicity of vertices in the output of \texttt{CliqueIntervention} (see \cref{lem:clique-intervention-subroutine-generalized}).

In summary, our tweaked algorithm for the generalized cost objective has $\cO(\log n)$ phases, similar to \cref{alg:weighted-search}, and we incur a cost of $\cO((\log n + \log k) \cdot \OPT_k)$ in each phase.
A description of the tweaked algorithm and a more detailed analysis of it is provided in the appendix.

%% file: figures/case1.tex
\begin{tikzpicture}
%
%
\node[draw, circle, minimum size=25pt, inner sep=2pt] at (0,0) (u1) {$v_1$};
\node[draw, circle, minimum size=25pt, inner sep=2pt, right=of u1] (u2) {$v_2$};
\node[minimum size=25pt, inner sep=2pt, right=of u2] (udots) {$\ldots$};
\node[draw, circle, minimum size=25pt, inner sep=2pt, right=of udots] (un) {$v_{|C|}$};
\draw[ultra thick] (u1) -- (u2);
\draw[ultra thick] (u1) to[out=45,in=135] (udots);
\draw[ultra thick] (u1) to[out=45,in=135] (un);
\draw[ultra thick] (u2) -- (udots);
\draw[ultra thick] (u2) to[out=45,in=155] (un);
\draw[ultra thick] (udots) -- (un);

%
%
\node[draw, circle, minimum size=25pt, inner sep=2pt] at (8,0) (d1) {$v_1$};
\node[draw, circle, minimum size=25pt, inner sep=2pt, right=of d1] (d2) {$v_2$};
\node[minimum size=25pt, inner sep=2pt, right=of d2] (ddots) {$\ldots$};
\node[draw, circle, minimum size=25pt, inner sep=2pt, right=of ddots] (dn) {$v_{|C|}$};
\draw[ultra thick, -stealth, dashed] (d1) -- (d2);
\draw[ultra thick, -stealth] (d1) to[out=45,in=135] (ddots);
\draw[ultra thick, -stealth] (d1) to[out=45,in=135] (dn);
\draw[ultra thick, -stealth, dashed] (d2) -- (ddots);
\draw[ultra thick, -stealth] (d2) to[out=45,in=155] (dn);
\draw[ultra thick, -stealth, dashed] (ddots) -- (dn);
\end{tikzpicture}

%% file: figures/case2.tex
\begin{tikzpicture}
%
%
\node[draw, circle, minimum size=10pt, inner sep=2pt] at (0,0) (uv) {$v$};
\node[draw, red, rectangle, rounded corners, minimum size=20pt, inner sep=2pt] at (-1.5,1) (uv2) {$H_2$};
\node[draw, red, rectangle, rounded corners, minimum size=20pt, inner sep=2pt] at (1,1) (uvt) {$H_t$};
\node[] at ($(uv2)!0.5!(uvt)$) (uvdots) {$\ldots$};
\node[draw, circle, minimum size=10pt, inner sep=2pt] at ($(uv) + (-0.5,-1)$) (u1) {};
\node[draw, circle, minimum size=10pt, inner sep=2pt] at ($(uv) + (1,-1)$) (u2) {};
\node[draw, circle, minimum size=10pt, inner sep=2pt] at ($(uv) + (0,-2)$) (u3) {};
\node[draw, circle, minimum size=10pt, inner sep=2pt] at ($(uv) + (-1.5,-1)$) (u4) {};
\node[draw, circle, minimum size=10pt, inner sep=2pt] at ($(uv) + (-1.5,-2)$) (u5) {};
\node[draw, circle, minimum size=10pt, inner sep=2pt] at ($(uv) + (0.25,-3)$) (u6) {};
\node[draw, circle, minimum size=10pt, inner sep=2pt] at ($(uv) + (-1.5,-3)$) (u7) {};
\node[draw, circle, minimum size=10pt, inner sep=2pt] at ($(uv) + (-0.5,-3)$) (u8) {};
\draw[ultra thick] (uv) -- (uv2);
\draw[ultra thick] (uv) -- (uvt);
\draw[ultra thick] (uv) -- (u1);
\draw[ultra thick] (uv) -- (u2);
\draw[ultra thick] (uv) -- (u3);
\draw[ultra thick] (uv) -- (u4);
\draw[ultra thick] (u1) -- (u2);
\draw[ultra thick] (u1) -- (u3);
\draw[ultra thick] (u1) -- (u4);
\draw[ultra thick] (u1) -- (u5);
\draw[ultra thick] (u2) -- (u3);
\draw[ultra thick] (u3) -- (u5);
\draw[ultra thick] (u3) -- (u6);
\draw[ultra thick] (u5) -- (u7);
\draw[ultra thick] (u5) -- (u8);

\node[fit=(u1)(u2)(u3)(u4)(u5)(u6)(u7)(u8), draw, red, rounded corners, inner sep=10pt] (uv1) {};
\node[red] at ($(uv1.south east) + (-0.5,0.5)$) {$H_1$};
\node[fit=(u1)(u2)(u3), draw, thick, blue, rounded corners] (uc1) {};
\node[blue] at ($(uc1.south east) + (-0.3,0.3)$) {$C_1$};

%
%
\node[draw, circle, minimum size=10pt, inner sep=2pt] at (4.5,0) (dv) {$v$};
\node[draw, red, rectangle, rounded corners, minimum size=20pt, inner sep=2pt] at (3,1) (dv2) {$H_2$};
\node[draw, red, rectangle, rounded corners, minimum size=20pt, inner sep=2pt] at (5.5,1) (dvt) {$H_t$};
\node[] at ($(dv2)!0.5!(dvt)$) (dvdots) {$\ldots$};
\node[draw, circle, minimum size=10pt, inner sep=2pt] at ($(dv) + (-0.5,-1)$) (d1) {\tiny $s_1$};
\node[draw, circle, minimum size=10pt, inner sep=2pt] at ($(dv) + (1,-1)$) (d2) {};
\node[draw, circle, minimum size=10pt, inner sep=2pt] at ($(dv) + (0,-2)$) (d3) {};
\node[draw, circle, minimum size=10pt, inner sep=2pt] at ($(dv) + (-1.5,-1)$) (d4) {};
\node[draw, circle, minimum size=10pt, inner sep=2pt] at ($(dv) + (-1.5,-2)$) (d5) {};
\node[draw, circle, minimum size=10pt, inner sep=2pt] at ($(dv) + (0.25,-3)$) (d6) {};
\node[draw, circle, minimum size=10pt, inner sep=2pt] at ($(dv) + (-1.5,-3)$) (d7) {};
\node[draw, circle, minimum size=10pt, inner sep=2pt] at ($(dv) + (-0.5,-3)$) (d8) {};
\draw[ultra thick, -stealth, dashed] (dv) -- node[above, pos=0.3]{$?$} (dv2);
\draw[ultra thick, -stealth, dashed] (dv) -- node[above, pos=0.1]{$?$} (dvt);
\draw[ultra thick, -stealth, dashed] (dv) -- (d1);
\draw[ultra thick, -stealth] (dv) -- (d2);
\draw[ultra thick, -stealth] (dv) -- (d3);
\draw[ultra thick, -stealth] (dv) -- (d4);
\draw[ultra thick, -stealth, dashed] (d1) -- (d2);
\draw[ultra thick, -stealth] (d1) -- (d3);
\draw[ultra thick, -stealth] (d1) -- (d4);
\draw[ultra thick, -stealth] (d1) -- (d5);
\draw[ultra thick, -stealth, dashed] (d2) -- (d3);
\draw[ultra thick, -stealth] (d3) -- (d5);
\draw[ultra thick, -stealth] (d3) -- (d6);
\draw[ultra thick, -stealth] (d5) -- (d7);
\draw[ultra thick, -stealth] (d5) -- (d8);

\node[fit=(d1)(d2)(d3)(d4)(d5)(d6)(d7)(d8), draw, red, rounded corners, inner sep=10pt] (dv1) {};
\node[red] at ($(dv1.south east) + (-0.5,0.5)$) {$H_1$};
\node[fit=(d1)(d2)(d3), draw, thick, blue, rounded corners] (dc1) {};
\node[blue] at ($(dc1.south east) + (-0.3,0.3)$) {$C_1$};
\end{tikzpicture}

%% file: figures/trace.tex
\begin{tikzpicture}
\node[draw, circle, inner sep=2pt] at (0,0) (vh) {\small $v_H$};
\node[draw, circle, inner sep=2pt] at ($(vh) + (-0.5,0.5)$) (k2) {};
\node[draw, circle, inner sep=2pt] at ($(vh) + (0.5,0.5)$) (k3) {};
\node[draw, circle, inner sep=2pt] at ($(vh) + (0.25,1)$) (k4) {};
\node[draw, circle, inner sep=2pt] at ($(vh) + (-0.25,1)$) (k5) {};

\draw[-stealth] (vh) -- (k2);
\draw[-stealth] (vh) -- (k3);
\draw[-stealth] (vh) -- (k4);
\draw[-stealth] (vh) -- (k5);
\draw[-stealth] (k2) -- (k3);
\draw[-stealth] (k2) -- (k4);
\draw[-stealth] (k2) -- (k5);
\draw[-stealth] (k3) -- (k4);
\draw[-stealth] (k3) -- (k5);
\draw[-stealth] (k4) -- (k5);

\node[draw, circle, inner sep=2pt] at ($(vh) + (-1.5,0.25)$) (u1) {\small $u_1$};
\node[draw, circle, inner sep=2pt] at ($(u1) + (-0.75,0)$) (u12) {};
\node[draw, circle, inner sep=2pt] at ($(u1) + (-1,-0.5)$) (u13) {};
\node[draw, circle, inner sep=2pt] at ($(u1) + (-0.5,-0.5)$) (u14) {};
\node[draw, circle, inner sep=2pt] at ($(u1) + (-0.5,-1)$) (u15) {};

\draw[-stealth] (u1) -- (u12);
\draw[-stealth] (u1) -- (u13);
\draw[-stealth] (u1) -- (u14);
\draw[-stealth] (u12) -- (u13);
\draw[-stealth] (u12) -- (u14);
\draw[-stealth] (u13) -- (u14);
\draw[-stealth] (u13) -- (u15);
\draw[-stealth] (u14) -- (u15);

\node[draw, circle, inner sep=2pt] at ($(vh) + (-1,-1.5)$) (u2) {\small $u_2$};
\node[draw, circle, inner sep=2pt] at ($(u2) + (0.5,-0.5)$) (u22) {};
\node[draw, circle, inner sep=2pt] at ($(u2) + (-0.5,-0.5)$) (u23) {};
\node[draw, circle, inner sep=2pt] at ($(u2) + (0,-1)$) (u24) {};

\draw[-stealth] (u2) -- (u22);
\draw[-stealth] (u2) -- (u23);
\draw[-stealth] (u22) -- (u24);
\draw[-stealth] (u22) -- (u23);
\draw[-stealth] (u23) -- (u24);

\node[draw, circle, inner sep=2pt] at ($(vh) + (1,-1.5)$) (u3) {\small $u_3$};
\node[draw, circle, inner sep=2pt] at ($(u3) + (-0.5,-0.5)$) (u32) {};
\node[draw, circle, inner sep=2pt] at ($(u3) + (-0.5,-1)$) (u33) {};
\node[draw, circle, inner sep=2pt] at ($(u3) + (0.5,-0.5)$) (u34) {};
\node[draw, circle, inner sep=2pt] at ($(u3) + (0.25,-1)$) (u35) {};

\draw[-stealth] (u3) -- (u32);
\draw[-stealth] (u3) -- (u34);
\draw[-stealth] (u32) -- (u33);
\draw[-stealth] (u32) -- (u34);
\draw[-stealth] (u32) -- (u35);
\draw[-stealth] (u34) -- (u35);

\node[draw, circle, inner sep=2pt] at ($(vh) + (1.5,0)$) (u4) {\small $u_4$};
\node[draw, circle, inner sep=2pt] at ($(u4) + (0.75,0)$) (u42) {};
\node[draw, circle, inner sep=2pt] at ($(u4) + (0.5,-0.5)$) (u43) {};

\draw[-stealth] (u4) -- (u42);
\draw[-stealth] (u42) -- (u43);

\draw[-stealth] (vh) -- (u1);
\draw[-stealth] (vh) to[out=150,in=45] (u12);
\draw[-stealth] (vh) to[out=230,in=270] (u13);
\draw[-stealth] (vh) -- (u14);

\draw[-stealth] (u2) to[out=45,in=240] (vh);
\draw[-stealth] (u2) -- (k2);
\draw[-stealth] (u22) -- (vh);

\draw[-stealth] (vh) -- (u3);
\draw[-stealth] (vh) -- (u32);

\draw[-stealth] (vh) -- (u4);
\draw[-stealth] (vh) to[out=15,in=135] (u42);
\draw[-stealth] (vh) to[out=345,in=180] (u43);

\draw[dashed, orange] ($(vh) + (0,-0.5)$) -- ($(vh) + (0,-3)$);
\node[orange] at ($(vh) + (1,1)$) {$C$};
\node[orange] at ($(vh) + (-2.5,-2.5)$) {$A$};
\node[orange] at ($(vh) + (2.5,-2.5)$) {$B$};

\node[draw, blue, fit=(u1)(u12)(u13)(u14), ellipse, inner xsep=-1pt, inner ysep=-1pt, yshift=-1pt] (h1_neighbors) {};
\node[draw, blue,fit=(u2)(u22), ellipse, rotate=140, inner ysep=-5pt] (h2_neighbors) {};
\node[draw, blue,fit=(u3)(u32), ellipse, rotate=50, inner ysep=-5pt] (h3_neighbors) {};
\node[draw, blue,fit=(u4)(u42)(u43), ellipse, inner xsep=-1pt, inner ysep=-1pt] (h4_neighbors) {};

\node[draw, orange, fit=(vh)(k2)(k3)(k4)(k5), rounded corners] (clique-separator) {};
\node[draw, red, fit=(u1)(u12)(u13)(u14)(u15)(h1_neighbors), rounded corners] (h1) {};
\node[draw, red, fit=(u2)(u22)(u23)(u24)(h2_neighbors), rounded corners] (h2) {};
\node[draw, red, fit=(u3)(u32)(u33)(u34)(u35)(h3_neighbors), rounded corners] (h3) {};
\node[draw, red, fit=(u4)(u42)(u43)(h4_neighbors), rounded corners] (h4) {};

\node[thick, red] at ($(h1_neighbors) + (0,1)$) {\small $H_1$};
\node[thick, red] at ($(h2_neighbors) + (-1.25,0)$) {\small $H_2$};
\node[thick, red] at ($(h3_neighbors) + (1.25,0)$) {\small $H_3$};
\node[thick, red] at ($(h4_neighbors) + (0,1)$) {\small $H_4$};
\end{tikzpicture}

%% file: experiments.tex
\section{Experiments}
\label{sec:experiments}

Since \texttt{ALG} (\cref{alg:weighted-search}) is a special case of \texttt{ALG-GENERALIZED} (\cref{alg:weighted-search-generalized}) when $\alpha = 0$, $\beta = 1$, and $k = 1$, we implement and benchmark \texttt{ALG-GENERALIZED} against a synthetic dataset.

We modified the experimental setup used by \cite{squires2020active,choo2022verification,choo2023subset} to run on \emph{weighted} causal graphs and measure the generalized cost incurred for varying $\alpha$ and $\beta$ values.
We ran experiments for $\alpha \in \{0,1\}$ and $\beta = 1$ on two different types of weight classes for a graph on $n$ vertices:
\begin{description}
    \item[Type 1] The weight of each vertex is independently sampled from an exponential distribution $\exp(n^2)$ with parameter $n^2$. This is to simulate the setting where there is a spread in the costs of the vertices.
    \item[Type 2] A randomly chosen $p=0.1$ fraction of vertices are assigned weight $n^2$ while the others are assigned weight $1$. This is to simulate the setting where there are a few randomly chosen high cost vertices.
\end{description}

We have 4 sets of experiments in total and
\cref{fig:some-experiments} shows a subset of them. More experimental details and results are given in \cref{sec:appendix-experiments}, where we also investigate the impact of size for bounded size interventions.

\begin{figure}[tb]
\centering
\begin{subfigure}[b]{0.47\linewidth}
    \centering
    \includegraphics[width=\linewidth]{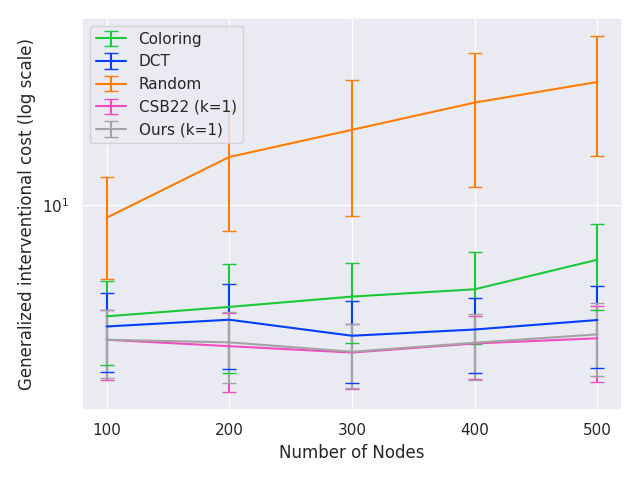}
    \caption{Type 1, $\alpha = 0$, $\beta = 1$}
\end{subfigure}
\begin{subfigure}[b]{0.47\linewidth}
    \centering
    \includegraphics[width=\linewidth]{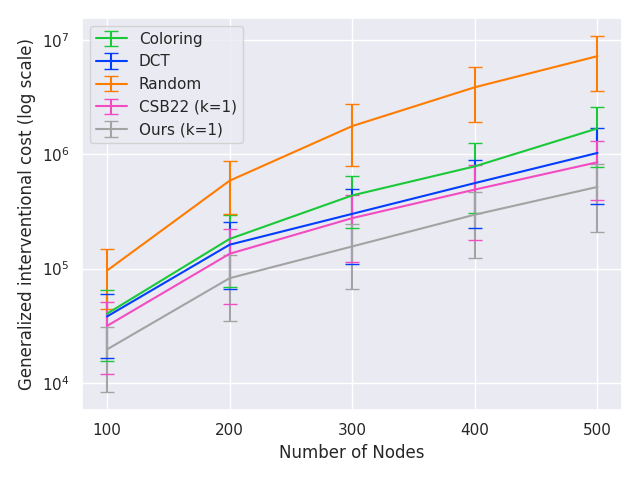}
    \caption{Type 1, $\alpha = 1$, $\beta = 1$}
\end{subfigure}
\\
\begin{subfigure}[b]{0.47\linewidth}
    \centering
    \includegraphics[width=\linewidth]{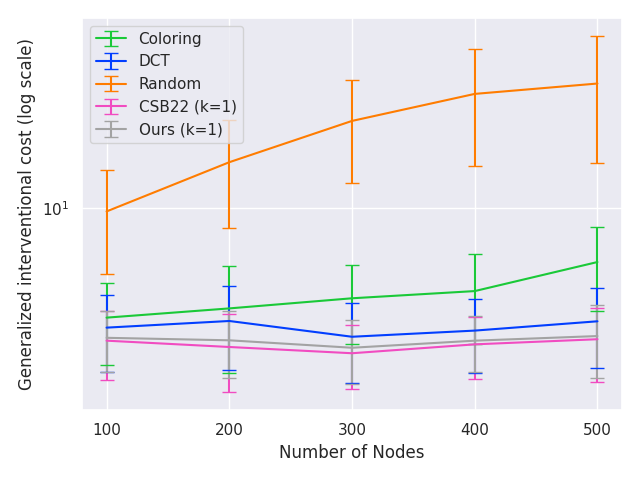}
    \caption{Type 2, $\alpha = 0$, $\beta = 1$}
\end{subfigure}
\begin{subfigure}[b]{0.47\linewidth}
    \centering
    \includegraphics[width=\linewidth]{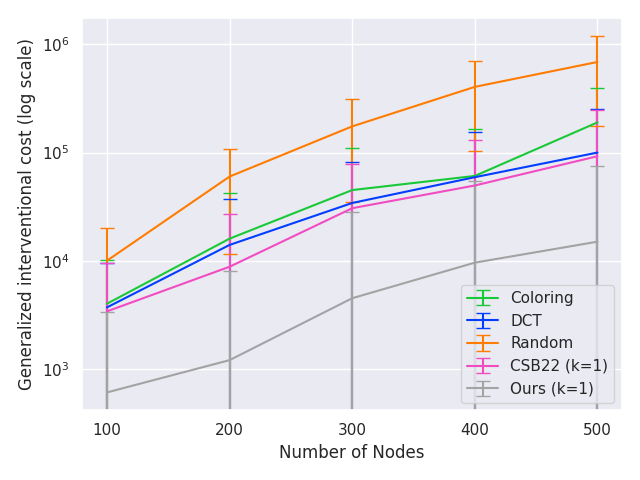}
    \caption{Type 2, $\alpha = 1$, $\beta = 1$}
\end{subfigure}
\caption{Experimental results for atomic interventions (log scale)}
\label{fig:some-experiments}
\end{figure}

\subsection{Qualitative discussion of experimental results}

For any intervention set $\cI \subseteq 2^V$ that fully orients the given causal graph, the Y-axis measures the generalized cost $\alpha \cdot w(\cI) + \beta \cdot |\cI|$.
So, fixing either $\alpha$ or $\beta$, and scaling the other will recover any possible observable trend (ignoring the magnitude of the values on the Y-axis).
As our experiments were for atomic interventions, the parameter setting of $(\alpha, \beta) = (0,1)$ precisely recovers the unweighted atomic intervention setting, so \cref{fig:some-experiments} attempts to illustrate what happens when we set $\alpha = 0$ and $\alpha = 1$.

When $\alpha = 0$, the generalized cost function is simply the number of interventions used that the other state-of-the-art methods were designed for.
Here, \texttt{ALG-GENERALIZED} incurs a similar cost despite having additional overheads to ensure theoretical guarantees for general $\alpha \geq 0$.

For $\alpha > 0$, the generalized cost function is affected by the vertex weights, and \texttt{ALG-GENERALIZED} incurs noticeably less generalized cost than the others already when $\alpha = 1$ (note that the plot is in log scale).
This gap will only increase as we increase the value of $\alpha$ to make the generalized cost put more weightage on the total additive vertex cost of the intervention $\cI$.

As our experimental instances were randomly generated, it does look like existing algorithms, such as \cite{choo2022verification}, is competitive with our weight-sensitive algorithm \texttt{ALG-GENERALIZED} on such random instances, even though they are oblivious to vertex weights.
However, we can easily create many instances where these algorithms performs arbitrarily worse.
For instance, consider the star graph $G^*$ on $n$ nodes where the leaves have weight 1 and the centroid has weight $w \gg n$; imagine $w = n^{10000}$.
On $G^*$, \cite{choo2022verification} will intervene on the centroid, incurring $w$ while \texttt{ALG-GENERALIZED} will never intervene on the centroid and in the worst case intervene on all the leaves (paying at most $n-1$) to fully orient $G^*$ from $\cE(G^*)$.

In terms of running time, \texttt{ALG-GENERALIZED} has a similar running time\footnote{\texttt{ALG-GENERALIZED} is faster than all benchmarked algorithms except \cite{choo2022verification}. This is expected as both are based on 1/2-clique separators but \texttt{ALG-GENERALIZED} has additional computational overheads to handle dangling components.} as the other state-of-the-art algorithms across all experiments.

%% file: conclusion.tex
\section{Conclusion and future directions}

In our work, we make standard assumptions of causal sufficiency, faithfulness, and infinite sample regime.
As these assumptions may be too strong in some practical settings, one should view our work as providing theoretical foundations to the \emph{feasibility} of the weighted search problem (i.e.\ what \emph{can} be done in an optimistic setting) and it is of paramount practical importance to weaken/remove such assumptions in future work. In addition, we also state some possible future directions that we think are interesting:

\begin{enumerate}
    \item Understand the optimal approximation ratio with respect to our new benchmark $\overline{\nu}^{\max}(G^*) = \max_{G \in [G^*]} \overline{\nu}(G)$.
        \cite{choo2022verification} tells us that a $\log(n)$ factor in approximation is \emph{unavoidable} even in the unweighted case, but is $\log^2(n)$ necessary in the weighted case?
        Is there an $\Omega(\log^2 n)$ lower bound construction, or is our analysis too loose, or is there another algorithm that achieves $\log n$ approximation in the weighted case?
    \item Design subset search algorithms \`{a} la \cite{choo2023subset} that are competitive with respect to $\overline{\nu}^{\max}(G^*)$, for both unweighted and weighted settings.
    \item Provide an efficient algorithm to compute $\overline{\nu}^{\max}(G^*)$.
        We remark that, for a given $G^*$, efficient computation of $\nu_1(G^*)$ and $\overline{\nu}_1(G^*)$ are known \cite{choo2022verification}.
\end{enumerate}

%% file: appendix-adaptive-versus-nonadaptive.tex
\section{Adaptive versus non-adaptive interventions}
\label{sec:appendix-adaptive-versus-nonadaptive}

Separating systems are the central mathematical objects for non-adaptive intervention design.
Roughly speaking, a separating system on a set of elements is a collection of subsets such that for every pair of elements from the set, there exists at least one subset which contains exactly one element from the pair.

Instead of all pairs of elements, let us consider the (typically smaller) $G$-separating system for a given graph $G$.
It is known \cite{kocaoglu2017cost} that the optimal non-adaptive intervention set to learn a moral DAG $G^*$ is a $\skel(G^*)$-separating system.

\begin{definition}[$G$-separating system; Definition 3 of \cite{kocaoglu2017cost}]
Given an undirected graph $G = (V,E)$, a set of subsets $\cI \subseteq 2^V$ is a $G$-separating system if for every edge $\{u,v\} in E$, there exists $I \in \cI$ such that either ($u \in I_i$ and $v \not\in I_i$) or ($u \not\in I_i$ and $v \in I_i$).
\end{definition}

\begin{theorem}[Theorem 1 of \cite{kocaoglu2017cost}]
\label{thm:g-separating-system}
For any undirected graph $G$, an intervention set $\cI$ learns \emph{every} causal graph $D$ with $\skel(D) = G$ if and only if $\cI$ is a $G$-separating system.
\end{theorem}

\paragraph{Path example}
Consider an essential graph which is an undirected path on $n$ vertices.
There are $n$ possible DAGs corresponding to this Markov equivalence class, each of which can be uniquely identified by picking one of the vertices as a source and orienting all edges away from it.
By \cref{thm:g-separating-system}, we see that $\Omega(n)$ atomic interventions are necessary.

\subsection{Adaptive can be exponentially stronger}
\label{sec:adaptive-is-exponentially-stronger}

Consider an essential graph which is an undirected path on $n$ vertices described above where we know that one has intervene on at least $\Omega(n)$ vertices using non-adaptive atomic interventions.
If we allow adaptive interventions, $\cO(\log n)$ atomic interventions suffice by simulating binary search: intervene on the ``center'' vertex to uncover its incident edge orientations; orient one half using Meek rule R1; repeat.

\subsection{New benchmark is different from separating system}
\label{sec:benchmark-differs-from-separating-system}

Recall our newly proposed metric:
$
\overline{\nu}^{\max}_k(G^*)
= \max_{G \in [G^*]} \overline{\nu}_k(G)
$
for any integer $k \geq 1$.

Consider an essential graph which is an undirected path on $n$ vertices described above where we know that one has intervene on at least $\Omega(n)$ vertices using non-adaptive atomic interventions.
Under our newly proposed metric, $\overline{\nu}^{\max}_1(G^*) = \nu^{\max}_1(G^*) = 1$ since intervening on the source vertex always suffices to fully orient the entire DAG.

%% file: appendix-meek-rules.tex
\section{Meek rules}
\label{sec:appendix-meek-rules}

\paragraph{Remark}
This section of well-known facts is adapted from the appendices of \cite{choo2022verification,choo2023subset}.

Meek rules are a set of 4 edge orientation rules that are sound and complete with respect to any given set of arcs that has a consistent DAG extension \cite{meek1995}.
Given any edge orientation information, one can always repeatedly apply Meek rules till a fixed point to maximize the number of oriented arcs.

\begin{definition}[Consistent extension]
A set of arcs is said to have a \emph{consistent DAG extension} $\pi$ for a graph $G$ if there exists a permutation on the vertices such that (i) every edge $\{u,v\}$ in $G$ is oriented $u \to v$ whenever $\pi(u) < \pi(v)$, (ii) there is no directed cycle, (iii) all the given arcs are present.
\end{definition}

\begin{definition}[The four Meek rules \cite{meek1995}, see \ref{fig:meek-rules} for an illustration]
\hspace{0pt}
\begin{description}
    \item [R1] Edge $\{a,b\} \in E \setminus A$ is oriented as $a \to b$ if $\exists$ $c \in V$ such that $c \to a$ and $c \not\sim b$.
    \item [R2] Edge $\{a,b\} \in E \setminus A$ is oriented as $a \to b$ if $\exists$ $c \in V$ such that $a \to c \to b$.
    \item [R3] Edge $\{a,b\} \in E \setminus A$ is oriented as $a \to b$ if $\exists$ $c,d \in V$ such that $d \sim a \sim c$, $d \to b \gets c$, and $c \not\sim d$.
    \item [R4] Edge $\{a,b\} \in E \setminus A$ is oriented as $a \to b$ if $\exists$ $c,d \in V$ such that $d \sim a \sim c$, $d \to c \to b$, and $b \not\sim d$.
\end{description}
\end{definition}

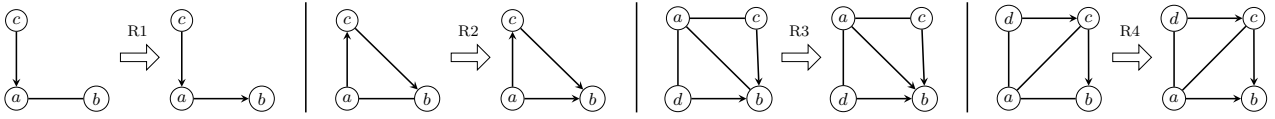
\begin{figure}[htbp]
\centering
\input{figures/meek-rules}
\caption{An illustration of the four Meek rules}
\label{fig:meek-rules}
\end{figure}

There exists an algorithm (Algorithm 2 of \cite{wienobst2021extendability}) that runs in $\cO(d \cdot |E|)$ time and computes the closure under Meek rules, where $d$ is the degeneracy of the graph skeleton\footnote{A $d$-degenerate graph is an undirected graph in which every subgraph has a vertex of degree at most $d$. Note that the degeneracy of a graph is typically smaller than the maximum degree of the graph.}.

%% file: figures/meek-rules.tex
\resizebox{\linewidth}{!}{%
\begin{tikzpicture}
%
%
\node[draw, circle, inner sep=2pt] at (0,0) (R1a-before) {\small $a$};
\node[draw, circle, inner sep=2pt, right=of R1a-before] (R1b-before) {\small $b$};
\node[draw, circle, inner sep=2pt, above=of R1a-before](R1c-before) {\small $c$};
\draw[thick, -stealth] (R1c-before) -- (R1a-before);
\draw[thick] (R1a-before) -- (R1b-before);

\node[draw, circle, inner sep=2pt] at (3,0) (R1a-after) {\small $a$};
\node[draw, circle, inner sep=2pt, right=of R1a-after] (R1b-after) {\small $b$};
\node[draw, circle, inner sep=2pt, above=of R1a-after](R1c-after) {\small $c$};
\draw[thick, -stealth] (R1c-after) -- (R1a-after);
\draw[thick, -stealth] (R1a-after) -- (R1b-after);

\node[single arrow, draw, minimum height=2em, single arrow head extend=1ex, inner sep=2pt] at (2.2,0.75) (R1arrow) {};
\node[above=5pt of R1arrow] {\footnotesize R1};

%
%
\node[draw, circle, inner sep=2pt] at (6,0) (R2a-before) {\small $a$};
\node[draw, circle, inner sep=2pt, right=of R2a-before] (R2b-before) {\small $b$};
\node[draw, circle, inner sep=2pt, above=of R2a-before](R2c-before) {\small $c$};
\draw[thick, -stealth] (R2a-before) -- (R2c-before);
\draw[thick, -stealth] (R2c-before) -- (R2b-before);
\draw[thick] (R2a-before) -- (R2b-before);

\node[draw, circle, inner sep=2pt] at (9,0) (R2a-after) {\small $a$};
\node[draw, circle, inner sep=2pt, right=of R2a-after] (R2b-after) {\small $b$};
\node[draw, circle, inner sep=2pt, above=of R2a-after](R2c-after) {\small $c$};
\draw[thick, -stealth] (R2a-after) -- (R2c-after);
\draw[thick, -stealth] (R2c-after) -- (R2b-after);
\draw[thick, -stealth] (R2a-after) -- (R2b-after);

\node[single arrow, draw, minimum height=2em, single arrow head extend=1ex, inner sep=2pt] at (8.2,0.75) (R2arrow) {};
\node[above=5pt of R2arrow] {\footnotesize R2};

%
%
\node[draw, circle, inner sep=2pt] at (12,0) (R3d-before) {\small $d$};
\node[draw, circle, inner sep=2pt, above=of R3d-before](R3a-before) {\small $a$};
\node[draw, circle, inner sep=2pt, right=of R3a-before] (R3c-before) {\small $c$};
\node[draw, circle, inner sep=2pt, right=of R3d-before](R3b-before) {\small $b$};
\draw[thick, -stealth] (R3c-before) -- (R3b-before);
\draw[thick, -stealth] (R3d-before) -- (R3b-before);
\draw[thick] (R3c-before) -- (R3a-before) -- (R3d-before);
\draw[thick] (R3a-before) -- (R3b-before);

\node[draw, circle, inner sep=2pt] at (15,0) (R3d-after) {\small $d$};
\node[draw, circle, inner sep=2pt, above=of R3d-after](R3a-after) {\small $a$};
\node[draw, circle, inner sep=2pt, right=of R3a-after] (R3c-after) {\small $c$};
\node[draw, circle, inner sep=2pt, right=of R3d-after](R3b-after) {\small $b$};
\draw[thick, -stealth] (R3c-after) -- (R3b-after);
\draw[thick, -stealth] (R3d-after) -- (R3b-after);
\draw[thick] (R3c-after) -- (R3a-after) -- (R3d-after);
\draw[thick, -stealth] (R3a-after) -- (R3b-after);

\node[single arrow, draw, minimum height=2em, single arrow head extend=1ex, inner sep=2pt] at (14.2,0.75) (R3arrow) {};
\node[above=5pt of R3arrow] {\footnotesize R3};

%
%
\node[draw, circle, inner sep=2pt] at (18,0) (R4a-before) {\small $a$};
\node[draw, circle, inner sep=2pt, above=of R4a-before](R4d-before) {\small $d$};
\node[draw, circle, inner sep=2pt, right=of R4d-before] (R4c-before) {\small $c$};
\node[draw, circle, inner sep=2pt, right=of R4a-before](R4b-before) {\small $b$};
\draw[thick, -stealth] (R4d-before) -- (R4c-before);
\draw[thick, -stealth] (R4c-before) -- (R4b-before);
\draw[thick] (R4d-before) -- (R4a-before) -- (R4c-before);
\draw[thick] (R4a-before) -- (R4b-before);

\node[draw, circle, inner sep=2pt] at (21,0) (R4a-after) {\small $a$};
\node[draw, circle, inner sep=2pt, above=of R4a-after](R4d-after) {\small $d$};
\node[draw, circle, inner sep=2pt, right=of R4d-after] (R4c-after) {\small $c$};
\node[draw, circle, inner sep=2pt, right=of R4a-after](R4b-after) {\small $b$};
\draw[thick, -stealth] (R4d-after) -- (R4c-after);
\draw[thick, -stealth] (R4c-after) -- (R4b-after);
\draw[thick] (R4d-after) -- (R4a-after) -- (R4c-after);
\draw[thick, -stealth] (R4a-after) -- (R4b-after);

\node[single arrow, draw, minimum height=2em, single arrow head extend=1ex, inner sep=2pt] at (20.2,0.75) (R4arrow) {};
\node[above=5pt of R4arrow] {\footnotesize R4};

\draw[thick] (5.25,1.75) -- (5.25,-0.25);
\draw[thick] (11.25,1.75) -- (11.25,-0.25);
\draw[thick] (17.25,1.75) -- (17.25,-0.25);
\end{tikzpicture}
}

%% file: appendix-additional-known-results.tex
\section{Additional known results}
\label{sec:additional-known-results}

\begin{lemma}[Yao's lemma \cite{yao1977probabilistic}]
\label{lem:yao}
Let $\cA$ be the space of all possible deterministic algorithms over probability distribution $p$, and $\cX$ be the space of problem inputs over probability distribution $q$.
Denote probability distributions over $\cA$ and $\cX$ by $p_a$ and $q_x$ respectively.
Then,
\[
\max_{x \in \cX} \E_p[c(A,x)] \geq \min_{a \in \cA} \E_q[c(a,X)]
\]
\end{lemma}

In other words, \cref{lem:yao} tells us that in order to lower bound the cost of any randomized algorithm, it suffices to find a ``bad'' input distribution such that any deterministic incurs a high cost.

\begin{lemma}[Modified lemma 1 of \cite{hauser2014two}; Appendix B of \cite{choo2022verification}]
\label{lem:hauser-bulmann-strengthened}
Let $\cI \subseteq 2^V$ be an intervention set.
Consider the $\cI$-essential graph $\cE_{\cI}(G^*)$ of some DAG $G^*$ and let $H \in CC(\cE_{\cI}(G^*))$ be one of its chain components.
Then, for any additional interventional set $\cI' \subseteq 2^V$ such that $\cI \cap \cI' = \emptyset$, we have
\[
\cE_{\cI \cup \cI'}(G^*)[V(H)] = \cE_{\{S \cap V(H)~:~S \in \cI'\}}(G^*[V(H)]).
\]
\end{lemma}

\begin{lemma}[Lemma 21 of \cite{choo2022verification}]
\label{lem:strengthened-lb}
Fix an essential graph $\cE(G^*)$ and $G \in [G^*]$.
Then,
\[
\nu_1(G) \geq \max_{\cI \subseteq V} \sum_{H \in CC(\cE_{\cI}(G^*))} \left\lfloor \frac{\omega(H)}{2} \right\rfloor
\]
\end{lemma}

\begin{theorem}[\cite{choo2023subset}]
\label{thm:moral-patching}
For any intervention set $\cI \subseteq 2^V$, define $R(G, \cI) = A(\cE_{\cI}(G)) \subseteq E$ as the set of oriented arcs in the $\cI$-essential graph of a DAG $G$ and define $G^{\cI} = G[E \setminus R(G,\cI)]$ as the \emph{fully directed} subgraph DAG induced by the \emph{unoriented arcs} in $G$, where $G^{\emptyset}$ is the graph obtained after removing all the oriented arcs in the observational essential graph due to v-structures.
Then, for any DAG $G = (V,E)$ and intervention sets $\cA, \cB \subseteq 2^V$,
\begin{enumerate}
    \item \textbf{``Suffices to study moral DAGs''}: $R(G,\cA \cup \cB) = R(G^{\cA},\cB) \;\dot\cup\ R(G^{\cB},\cA) \;\dot\cup\; (R(G,\cA) \cap R(G,\cB))$
    \item \textbf{``Patching''}: Any acyclic completion of $\cE(G^{\cA})$ can be combined with $R(G,\cA)$ to obtain a valid DAG that belongs to both $\cE(G)$ and $\cE_{\cA}(G)$.
\end{enumerate}
\end{theorem}

The first point of \cref{thm:moral-patching} justifies why it suffices to only study verification and adaptive search via ideal interventions on moral DAGs: since $R(G,\cI) = R(G^{\emptyset},\cI) \;\dot\cup\; R(G, \emptyset)$, any oriented arcs in the observational graph can be removed \emph{before performing any interventions} as the optimality of the solution is unaffected.

The second point of \cref{thm:moral-patching} tells us one can freely orient any chain component within any interventional essential graph and still be able to find a consistent DAG within the equivalence class.
This is useful in the lower bound analysis of our proposed benchmark later.

\begin{definition}[Separation of covered edges; Definition 8 of \cite{choo2022verification}]
We say that an intervention $S \subseteq V$ \emph{separates} a covered edge $u \sim v$ if $|\{u,v\} \cap S| = 1$.
That is, \emph{exactly} one of the endpoints is intervened by $S$.
We say that an intervention set $\cI$ separates a covered edge $u \sim v$ if there exists $S \in \cI$ that separates $u \sim v$.
\end{definition}

\begin{theorem}[Theorem 9 of \cite{choo2022verification}]
\label{thm:necessary-and-sufficient}
An intervention set $\cI$ is an atomic verifying set for DAG $G$ if and only if $\cI$ separates every covered edges of $G$.
\end{theorem}

\begin{theorem}[Theorem 12 of \cite{choo2022verification}]
\label{thm:nu-k-to-1}
For any DAG $G$ and integer $k \geq 1$, $\nu_k(G) \geq \lceil \nu_1(G)/k \rceil$.
\end{theorem}

\begin{theorem}[Proposition 3 of \cite{eberhardt2006n}; Theorem 4 of \cite{shanmugam2015learning}; Lemma 17 of \cite{choo2022verification}]
\label{thm:clique-covered-edges-and-lower-bound}
If a DAG $G$ is a clique on $n \geq 3$ vertices $v_1, v_2, \ldots, v_n$ with $\pi(v_1) < \pi(v_2) < \ldots < \pi(v_n)$, then $v_1 \to v_2, \ldots, v_{n-1} \to v_n$ are covered edges of $G$.
Using atomic interventions to orient $G$, $n-1$ adaptive interventions are necessary in the worst case.
Using interventions involving at most $k \geq 1$ vertices each to orient $G$, $\frac{n}{2k}$ randomized adaptive interventions are necessary.
In any case, to orient $G$, the total number of variables being intervened upon is at least $n/2$.
\end{theorem}

Note that \cref{thm:clique-covered-edges-and-lower-bound} holds even if the clique is just a subgraph of a larger causal DAG as long as there is no non-clique vertex $u$ such that $\pi(v_i) < \pi(u) < \pi(v_j)$ and $v_i \to u \to v_j$ for any two clique vertices $v_i$ and $v_j$ with $i < j$.
Within the proofs of \cref{thm:interventional-metric-lower-bound} and \cref{thm:interventional-metric-lower-bound-generalized}, we rely on this observation in combination with \cref{lem:clique-picking}, which make the clique a prefix within an ordering of interest.
This allows us to lower bound our benchmark since our benchmark cares about the worst case ordering.

The next result of \cite{wienobst2021polynomial} is used together with \cref{thm:moral-patching} to argue that we can always pick an unoriented clique (not necessarily maximal) to be the prefix of a given interventional essential graph.

\begin{definition}[Acyclic moral orientation]
An \emph{acyclic moral orientation} is a complete orientation of a partially directed DAG such that it does not create a new v-structure.
\end{definition}

\begin{lemma}[Maximal clique picking; \cite{wienobst2021polynomial}]
\label{lem:maximal-clique-picking}
Every acyclic moral orientation of an undirected graph can be represented by a topological ordering which starts with a maximal clique.
\end{lemma}

The next result is a lemma in the appendix of \cite{choo2023subset} that is used to prove \cref{lem:middle}.
We will later use it to prove a generalization of \cref{lem:middle} that holds for bounded size interventions.

\begin{lemma}
\label{lem:intermediate-direct-arcs-exist}
Let $G = (V,E)$ be a moral DAG.
If $u \to v$ in $G$, then $u \to w$ in $G$ for any two vertices $u,v \in V$ and for all $w \in \Des(u) \cap \Anc(v)$.
\end{lemma}

%% file: appendix-generalized-cost.tex
\section{Handling the generalized cost objective}
\label{sec:appendix-generalized-cost}

As discussed in \cref{sec:handling-generalized-cost}, we make algorithmic tweaks to account for \cref{eq:generalized-cost} and bounded size interventions: see the {\color{blue}blue} lines in \texttt{ALG-GENERALIZED} (\cref{alg:weighted-search-generalized}) and \texttt{ResolveDanglingGeneralized} (\cref{alg:dangling-subroutine-generalized}).

\begin{algorithm}[h]
\begin{algorithmic}[1]
\caption{\texttt{ALG-GENERALIZED}. A weighted adaptive search competitive with respect to \cref{eq:generalized-cost}.}
\label{alg:weighted-search-generalized}
    \Statex \textbf{Input}: Essential graph $\cE(G^*)$ of a moral DAG $G^*$, weight function $w : V \to \R$, and integer $k \geq 1$.
    \Statex \textbf{Output}: Bounded size intervention set $\cI$ such that $\cE_{\cI}(G^*) = G^*$.
    \State Initialize $i=0$ and $\cI_0 = \emptyset$.
    \While{$\cE_{\cI_{i}}(G^*)$ still has unoriented edges}
        \State Initialize $\cJ_i \gets \emptyset$
        \For{$H \in CC(\cE_{\cI_{i}}(G^*))$ of size $|H| \geq 2$}
            \State Find 1/2-clique separator $K_H$ via \cref{thm:chordal-separator}.
            \State Denote $v_H$ as the costliest vertex $v_H = \argmax_{v \in V(K_H)} w(v)$.
            \State {\color{blue}Let $S$ be the intervention set output by \texttt{CliqueIntervention} on the subclique $V(K_H) \setminus \{v_H\}$
            \Statex\hspace{\algorithmicindent}\hspace{\algorithmicindent}\emph{without} $v_H$.}
            \State Intervene on $S$ and add $S$ to $\cJ_i$.
            \State Let $Z_{v_H} \in CC(\cE_{\cI_i \cup \cJ_i}(G^*))$ be the chain component containing $v_H$ \emph{after} intervening on $S$.
            \If{$v_H$ is \emph{not} singleton in $Z_{v_H}$}
                \State {\color{blue}Add output of $\texttt{ResolveDanglingGeneralized}$ to $\cJ_i$.}
            \EndIf
        \State Update $\cI_{i+1} \gets \cI_i \cup \cJ_i$ and $i \gets i+1$.
        \EndFor
    \EndWhile
    \State \textbf{Return} $\cI_i$
\end{algorithmic}
\end{algorithm}

\begin{algorithm}[h]
\begin{algorithmic}[1]
\caption{\texttt{ResolveDanglingGeneralized}. A subroutine for \texttt{ALG-GENERALIZED}.}
\label{alg:dangling-subroutine-generalized}
    \State \textbf{Input}: Interventional essential graph $\cE_{\cI'}(G^*)$ for some intervention $\cI' \subseteq 2^V$, weight function $w : V \to \R$, a chain component $H$ of $\cE_{\cI'}(G^*)$ that contains vertex $v \in V(H)$ with $t$ disjoint connected components $H_1, \ldots, H_t$ in $H[V(H) \setminus \{v\}]$, and an integer $k \geq 1$.
    \State \textbf{Output}: Bounded size intervention set $\cI$ such that the $t$ components are mutually disjoint in $\cE_{\cI \cup \cI'}(H)$.
    \State Initialize $\cI \gets \emptyset$.
    \If{{\color{blue}$\alpha \cdot w(v) + \beta \leq \sum_{i=1}^t \max_{\text{clique $C$ in $H_i \cap N_H(v)$}} \alpha \cdot w(C) + \beta \cdot |V(C)|$}}
        \State Intervene on $v$ and set $\cI \gets \{\{v\}\}$.
    \Else
        \For{$i \in \{1, \ldots, t\}$}
            \State Initialize $V' \gets V(H_i) \cap N_H(v)$.
            {\color{blue}
            \While{$\skel(\cE(G^*))[V']$ is \emph{not} a clique, or $|V'| > k$}
                \State Find a 1/2-clique separator $K$ of $H_i[V']$ using \cref{thm:chordal-separator}.
                \State Arbitrarily partition the vertices of $K$ into sets $S_1, \ldots, S_{\lceil |V(K)|/k \rceil} \subseteq V$, each involving at most
                \Statex\hspace{\algorithmicindent}\hspace{\algorithmicindent}\hspace{\algorithmicindent}$k$ vertices.
                \State Intervene on the vertices in the sets $\mathcal{S} = \{S_1, \ldots, S_{\lceil |V(K)|/k \rceil}\}$ and add $S_1, \ldots, S_{\lceil |V(K)|/k \rceil}$ to $\cI$.
                \State By \cref{lem:s_source}, identify $S_{source} \in \mathcal{S}$ which is the set containing the source node of $K$.
                \State By \cref{lem:at-most-one-incoming-generalized}, determine if there exists a chain component $Q$ with only incoming arcs to $S_{source}$.
                \Statex\hspace{\algorithmicindent}\hspace{\algorithmicindent}\hspace{\algorithmicindent}If so, find it.
                \If{$Q$ exists}
                    \State Restrict $V'$ to $V(Q)$.
                \Else
                    \State Set $V' \gets V(S_{source})$.
                \EndIf
            \EndWhile
            \State Let $S$ be the intervention set output by \texttt{CliqueIntervention} on the clique $H_i[V']$ involving at most 
            \Statex\hspace{\algorithmicindent}\hspace{\algorithmicindent}$k$ vertices.
            \State Intervene on $S$ and add $S$ to $\cI$.
            }
        \EndFor
    \EndIf
    \State \textbf{Return} $\cI$
\end{algorithmic}
\end{algorithm}

\begin{algorithm}[h]
\begin{algorithmic}[1]
\caption{\texttt{CliqueIntervention}. A labelling subroutine based on \cref{lem:labelling-scheme}.}
\label{alg:clique-intervention-subroutine}
    \Statex \textbf{Input}: A set of clique vertices $C \subseteq V$, integer $k \geq 1$.
    \Statex \textbf{Output}: Partition $S$ of $C$.
    \If{$k = 1$}
        \State Define $\cI = \{\{v\} : v \in C\}$
    \Else
        \State Define $k' = \min\{k, |C|/2\}$, $a = \lceil |C|/k' \rceil \geq 2$, and $\ell = \lceil \log_a |C| \rceil$. Compute labelling scheme on $C$ with
        \Statex\hspace{\algorithmicindent}$(|C|, k, a)$ via \cref{lem:labelling-scheme} and define $\cI = \{S_{x,y}\}_{x \in [\ell], y \in [a]}$, where $S_{x,y} \subseteq Q$ is the subset of vertices whose
        \Statex\hspace{\algorithmicindent}$x^{th}$ letter in the label is $y$.
    \EndIf
    \State \textbf{Return} $\cI$
\end{algorithmic}
\end{algorithm}

The correctness of \cref{alg:dangling-subroutine-generalized} relies on \cref{lem:at-most-one-incoming} \cref{lem:s_source}, and \cref{lem:at-most-one-incoming-generalized}.
Note that \cref{alg:dangling-subroutine-generalized} does \emph{not} attempt to fully orient the edges within the 1/2-clique separator while searching for a prefix clique of size at most $k$.
Since we are \emph{not} guaranteed to know the source node of $K$ so we cannot hope to directly apply \cref{lem:at-most-one-incoming}, and thus we need to prove generalized version of \cref{lem:at-most-one-incoming-generalized} to justify why \texttt{ResolveDanglingGeneralized} terminates after $\cO(\log n)$ iterations.
In fact, \cref{lem:at-most-one-incoming} is the special case where the clique $K$ is a single vertex.

\begin{restatable}{mylemma}{ssource}
\label{lem:s_source}
Consider any arbitrary directed clique $G = (V,E)$ and any integer $k \geq 1$.
Without loss of generality, $V = \{v_1, \ldots, v_n\}$ and $\pi(v_1) < \ldots < \pi(v_n)$, i.e.\ $v_1$ is the source of $G$.
Suppose we arbitrarily partition the vertex set into sets $\mathcal{S} = \{S_1, \ldots, S_{\lceil n/k \rceil}\}$, each of size at most $k$.
Then, the set $S_{source} \in \mathcal{S}$ containing $v_1$ is the \emph{unique} set in $\mathcal{S}$ that has a vertex without any incoming arcs from the other sets.
\end{restatable}

\begin{restatable}{mylemma}{atmostoneincominggeneralized}
\label{lem:at-most-one-incoming-generalized}
Let $\cE_{\cI}(G)$ be the interventional essential graph of a moral DAG $G = (V,E)$ with respect to intervention set $\cI \subseteq 2^V$.
Fix any chain component $H \in CC(\cE_{\cI}(G))$ and let $K$ be an arbitrary clique in $H$.
If $K$ contains the source node of $H$, then there are no chain components of $\cE_{\cI \cup \{V(K)\}}(H)$ with only incoming arcs into $K$ in $G$.
Otherwise, if $K$ does not contain the source node of $H$, then there is exactly one chain component of $\cE_{\cI \cup \{V(K)\}}(H)$ with only incoming arcs into $K$ in $G$.
Furthermore, without further interventions, we can decide if such a chain component exist (and find it) in polynomial time.
\end{restatable}

To obtain bounded size interventions for intervening on cliques, we invoke \cref{lem:labelling-scheme} through the subroutine \texttt{CliqueIntervention} (\cref{alg:clique-intervention-subroutine}).
\cref{lem:clique-intervention-subroutine-generalized} states the guarantees of \texttt{CliqueIntervention}.

\cliqueinterventionsubroutinegeneralized*

Analogous to \cref{thm:interventional-metric-lower-bound} and \cref{lem:dangling-subroutine}, we prove \cref{thm:interventional-metric-lower-bound-generalized} and \cref{lem:dangling-subroutine-generalized} for our tweaked algorithm with respect to the generalized cost \cref{eq:generalized-cost}.

\begin{restatable}{mytheorem}{interventionalmetriclowerboundgeneralized}
\label{thm:interventional-metric-lower-bound-generalized}
Fix an essential graph $\cE(G^*)$ corresponding to an unknown weighted causal DAG $G^*$.
Suppose $\cI^*_1$ and $\cI^*_k$ are an atomic and bounded size intervention sets minimizing \cref{eq:generalized-cost} such that $\cE_{\cI^*_1}(G^*) = \cE_{\cI^*_k}(G^*) = G^*$, $\cost(\cI^*_1, \alpha, \beta, 1) = \OPT_1$, and $\cost(\cI^*_k, \alpha, \beta, k) = \OPT_k$.
Then, maximizing over intervention sets $\cI \subseteq V$, we have
\[
\OPT_1 \geq
\max_{\substack{\cI \subseteq 2^V\\\text{$\cI$ atomic}}} \left\{ \sum_{\substack{H \in CC(\cE_{\cI}(G^*))\\|V(H)| \geq 2}} \max \left\{  \zeta^{(3)}_{\cI,H} , \zeta^{(4)}_{\cI,H} \right\} \right\}
\quad \text{and} \quad
\OPT_k \geq
\max_{\substack{\cI \subseteq 2^V\\\text{$\cI$ bounded size}}} \left\{ \sum_{\substack{H \in CC(\cE_{\cI}(G^*))\\|V(H)| \geq 2}} \max \left\{  \zeta^{(5)}_{\cI,H} , \zeta^{(6)}_{\cI,H} \right\} \right\}
\]
where
\begin{align*}
\zeta^{(3)}_{\cI,H}
&= \frac{1}{2} \cdot \max_{\text{clique } C \in H} \left\{ \alpha \cdot \left( w(V(C)) - \max_{v \in V(C)} \left\{ w(v) \right\} \right) + \beta \cdot |V(C)| \right\} \;,\\
\zeta^{(4)}_{\cI,H}
&= \frac{1}{2} \cdot \max_{v \in V(H)} \left\{ \min \left\{ \alpha \cdot w(v) + \beta, \sum_{i=1}^t \max_{\substack{\text{clique } C_i:\\V(C_i) \subseteq V_i \cap N_H(v)}} \left\{ \alpha \cdot w(V(C_i)) + \beta \cdot |V(C_i)| \right\} \right\} \right\} \;,\\
\zeta^{(5)}_{\cI,H}
&= \frac{1}{2} \cdot \max_{\text{clique } C \in H} \left\{ \alpha \cdot \left( w(V(C)) - \max_{v \in V(C)} \left\{ w(v) \right\} \right) + \frac{\beta}{k} \cdot |V(C)| \right\} \;,\\
\zeta^{(6)}_{\cI,H}
&= \frac{1}{2} \cdot \max_{v \in V(H)} \left\{ \min \left\{ \alpha \cdot w(v) + \frac{\beta}{k}, \sum_{i=1}^t \max_{\substack{\text{clique } C_i:\\V(C_i) \subseteq V_i \cap N_H(v)}} \left\{ \alpha \cdot w(V(C_i)) + \frac{\beta}{k} \cdot |V(C_i)| \right\} \right\} \right\} \;,
\end{align*}

and $V_1, \ldots, V_t \subseteq V(H)$ are vertex sets of the $t \geq 1$ disjoint connected components in $H[V(H) \setminus \{v\}]$ in $\zeta^{(4)}_{\cI,H}$ and $\zeta^{(6)}_{\cI,H}$.
\end{restatable}

\begin{restatable}{mylemma}{danglingsubroutinegeneralized}
\label{lem:dangling-subroutine-generalized}
Fix an interventional essential graph $\cE_{\cI'}(G^*)$ corresponding to an unknown weighted causal moral DAG $G^*$ and some intervention $\cI' \subseteq 2^V$.
Suppose $\cI^*_1$ and $\cI^*_k$ are atomic and bounded size intervention sets minimizing \cref{eq:generalized-cost} such that $\cE_{\cI^*_1}(G^*) = \cE_{\cI^*_k}(G^*) = G^*$, $\cost(\cI^*_1, \alpha, \beta, 1) = \OPT_1$, and $\cost(\cI^*_k, \alpha, \beta, k) = \OPT_k$.
Let $H$ be a chain component of $\cE_{\cI'}(G^*)$ containing a vertex $v \in V(H)$.
Then,
\begin{itemize}
    \item When $k = 1$, \cref{alg:dangling-subroutine-generalized} returns an atomic intervention set $\cI$ such that connected components in $H[V(H) \setminus \{v\}]$ are mutually disjoint in $\cE_{\cI}(H)$ and $\cost(\cI, \alpha, \beta, 1) \in \cO(\log n \cdot \OPT_1)$.
    \item When $k > 1$, \cref{alg:dangling-subroutine-generalized} returns a bounded size intervention set $\cI$ such that connected components in $H[V(H) \setminus \{v\}]$ are mutually disjoint in $\cE_{\cI}(H)$ and $\cost(\cI, \alpha, \beta, k) \in \cO((\log n + \log k) \cdot \OPT_k)$.
\end{itemize}
\end{restatable}

Denote an iteration of the while loop in \cref{alg:weighted-search} as a \emph{phase}.
Since \cref{alg:weighted-search} and \cref{alg:weighted-search-generalized} are essentially the same in terms of how they recurse on smaller chain components of at most half the size in each phase, we can also obtain \cref{lem:logn-phases-suffice-generalized}.

\begin{restatable}{mylemma}{lognphasessufficegeneralized}
\label{lem:logn-phases-suffice-generalized}
\texttt{ALG-GENERALIZED} (\cref{alg:weighted-search-generalized}) terminates after $\cO(\log n)$ phases.
\end{restatable}

Using \cref{thm:interventional-metric-lower-bound-generalized}, we can also obtain \cref{lem:cost-of-each-phase-generalized}.

\begin{restatable}{mylemma}{costofeachphasegeneralized}
\label{lem:cost-of-each-phase-generalized}
Suppose $\cI^*_1$ and $\cI^*_k$ are an atomic and bounded size verifying sets respectively for $G^*$ that minimizes \cref{eq:generalized-cost} with $\cost(\cI^*_1) = \OPT_1$ and $\cost(\cI^*_k) = \OPT_k$.
Each phase in \texttt{ALG-GENERALIZED} (\cref{alg:weighted-search-generalized}) incurs a cost of $\cO(\log n \cdot \OPT_1)$ when $k = 1$ and $\cO \left( (\log n + \log k) \cdot \OPT_k \right)$ when $k > 1$.
\end{restatable}

Then, \cref{thm:weighted-search-generalized} follows directly by combining \cref{lem:logn-phases-suffice-generalized} and \cref{lem:cost-of-each-phase-generalized}.

See \cref{sec:proofs} for the full proofs of \cref{lem:clique-intervention-subroutine-generalized},
\cref{thm:interventional-metric-lower-bound-generalized},
\cref{lem:logn-phases-suffice-generalized}, and \cref{lem:cost-of-each-phase-generalized}.

%% file: appendix-blackbox-combination.tex
\section{Blackbox combination of algorithms}
\label{sec:appendix-blackbox-combination}

In this section, we describe a deterministic naive algorithm (\cref{alg:naive-weighted-search}) which provably incurs a cost of $\cO(n \cdot \overline{\nu}(G^*))$ and show how to combine this algorithm in a blackbox manner with any other deterministic algorithm to augment it with the provable guarantee of incurring a cost of at most $\cO(n \cdot \overline{\nu}(G^*))$.

\begin{algorithm}[tb]
\begin{algorithmic}[1]
\caption{Naive weighted adaptive search.}
\label{alg:naive-weighted-search}
    \Statex \textbf{Input}: Essential graph $\cE(G^*)$ of a moral DAG $G^*$ and weight function $w : V \to \R$.
    \Statex \textbf{Output}: Atomic intervention set $\cI$ s.t.\ $\cE_{\cI}(G^*) = G^*$.
    \State Sort the vertices in non-decreasing weight ordering.
    \While{$\cE_{\cI_{i}}(G^*)$ still has unoriented edges}
        \State Intervene on the next cheapest unintervened vertex and add it to $\cI$.
    \EndWhile
    \State \textbf{Return} $\cI_i$
\end{algorithmic}
\end{algorithm}

\begin{lemma}
Fix an essential graph $\cE(G^*)$ corresponding to an unknown weighted causal DAG $G^*$.
\cref{alg:naive-weighted-search} is a deterministic and adaptive algorithm that computes an atomic intervention set $\cI$ in polynomial time such that $\cE_{\cI}(G^*) = G^*$ and $w(\cI) \in \cO \left( n \cdot \overline{\nu}_1(G^*) \right)$.
\end{lemma}
\begin{proof}
The weight $w(v_{final})$ of the final intervened vertex $v_{final}$ is a lower bound for $\overline{\nu}(G^*)$.
Meanwhile, we intervened at most $n$ vertices before $v_{final}$, each of which has cost lower than $w(v_{final})$.
\end{proof}

\begin{theorem}
Fix an essential graph $\cE(G^*)$ corresponding to an unknown weighted causal DAG $G^*$.
Let $A$ be a deterministic algorithm that that computes an atomic intervention set $\cI$ in polynomial time such that $\cE_{\cI}(G^*) = G^*$ and $w(\cI) \in \cO(C)$.
Then, there is a deterministic and adaptive algorithm that computes an atomic intervention set $\cI$ in polynomial time such that $\cE_{\cI}(G^*) = G^*$ and $w(\cI) \in \cO\left( \min\{ C, n \cdot \overline{\nu}_1(G^*) \} \right)$.
\end{theorem}
\begin{proof}
Let $A_{naive}$ denote \cref{alg:naive-weighted-search}.
We will run both $A$ and $A_{naive}$ in parallel with a budget constraint (that doubles whenever it is exhausted) until we fully orient the causal graph.

More precisely, our new algorithm $A_{new}$ is defined as follows:
\begin{enumerate}
    \item We initialize a budget of $B = \min_{v \in V} w(v)$ to the minimum vertex cost.
    Without loss of generality, we may assume that $B > 0$ by first intervening on all vertices with 0 cost.
    \item ``Simulate'' $A$ until the total accumulated cost is at most $B$.
    If the graph is fully oriented at any point in time, terminate.
    \item ``Simulate'' $A_{naive}$ until the total accumulated cost is at most $B$.
    If the graph is fully oriented at any point in time, terminate.
    \item Double the value of $B$ (or increase it to the next weight so that there is no ``empty iteration'') and return to step 2.
\end{enumerate}
By ``simulate'', we mean that we accumulate the cost of vertices but only intervene on vertices that has not been intervened on previously. We can do this because $A$ and $A_{naive}$ are deterministic.

Since $A_{new}$ only terminates whenever either $A$ or $A_{naive}$ succeeds in fully orienting the causal graph, $A_{new}$ will correctly fully orient any input graph.
Note that we always have $B \in \cO(C)$ and $B \in \cO(n \cdot \overline{\nu}_1(G^*))$ at any point of the modified algorithm whenever neither algorithm terminated.
Since we always double the budget (any constant factor multiplication works), the above asymptotic upper bound also holds for $B_{final}$, where $B_{final}$ is the final value of $B$ when the algorithm terminates.
Furthermore, the cost of $A_{new}$ is at most $2 \cdot B_{final}$ since we pay at most $B_{final}$ for running $A$ and at most $B_{final}$ for running $A_{naive}$.

$A_{new}$ runs in polynomial time because value of $B$ changes at most $\cO \left( \min \left\{ n, \log \frac{\max_{v \in V} w(v)}{\min_{v \in V} w(v)} \right\} \right)$ times and each simulation of $A$ and $A_{naive}$ runs in polynomial time.
\end{proof}

\paragraph{Remark about implementation}
Both $A$ and $A_{naive}$ are actually intervening on the same causal graph, and $A_{new}$ does not actually discard information gained from $A$ when simulating $A_{naive}$ (and vice versa).
This may actually help the algorithm to terminate faster (at a lower cost) than running $A$ or $A_{naive}$ independently.

%% file: appendix-proofs.tex
\section{Deferred proofs}
\label{sec:proofs}

\subsection{Why $\overline{\nu}_1(G^*)$ is not an ideal benchmark}

\napxratio*
\begin{proof}
Let $G^*= (V,E,w)$ be a weighted causal directed tree where $|V| = n$ and $\skel(G^*)$ is a star graph in which $n-1$ vertices have degree 1 (non-center nodes) and a single vertex has degree n – 1 (center node). The weights of the nodes are given as follows,
\[
w(v) =
\begin{cases}
n-1 & \text{$v$ is a center}\\
1 & \text{otherwise}~.
\end{cases}
\]
We let one of the $n-1$ non-center nodes be the root of $G^*$. As intervening on the root suffices to fully orient $G^*$ from its essential graph $\cE(G^*)$, we see that $\overline{\nu}_1(G^*) = 1$.

Observe that any adaptive search algorithm that intervenes on the center of the star immediately incurs $n-1 \in \Omega(n \cdot \overline{\nu}_1(G^*))$.
Meanwhile, intervening on any leaf vertex that is \emph{not} the root will only orient the single edge incident to it so $n-1$ non-center node interventions are needed in the worst case, incurring a cost of $\Omega(n \cdot \overline{\nu}_1(G^*))$.

For randomized algorithms, we will use Yao's lemma (\cref{lem:yao}): the expected worst case performance of a \emph{randomized} algorithm is at least as much the expected performance of the best \emph{deterministic} algorithm over some distribution of inputs.

Consider the distribution of $G^*$ by uniformly picking the root amongst the leaves.
Any deterministic algorithm $A$ can be uniquely mapped to a sequence $\sigma_A$ of vertices that it will intervene on, until $\cE(G^*)$ is fully oriented.
Since the center is never the root in our distribution, any algorithm $A$ that intervene on the center within its first $n-1$ choices strictly perform worse than the alternative algorithm $A'$ that shifts the choice of intervening to the last vertex, i.e.\ if $\sigma_A(j)$ is the center, then
\[
\sigma_{A'}(i) =
\begin{cases}
\sigma_{A}(i) & \text{if $i < j$}\\
\sigma_{A}(i+1) & \text{if $j \leq i < n-1$}\\
\text{center} & \text{if $i = n$}
\end{cases}
\]
Then, for any algorithm $A$ that does \emph{not} intervene on the center within its first $n-1$ choices, we see that the intervention set $\cI$ produced by $A$ has expected cost
\[
\E[w(\cI)] =
\frac{1}{n-1} \cdot \left( 1 + 2 + 3 + \ldots + (n-1) \right)
\in \Omega(n \cdot \overline{\nu}_1(G^*))
\]
\end{proof}

\subsection{Lower bounding the benchmark}

As \cref{thm:interventional-metric-lower-bound} relies on \cref{lem:clique-picking}, we will prove \cref{lem:clique-picking} first.

\cliquepicking*
\begin{proof}
Let $H$ be the chain component containing the clique $C$ of interest.
By \cref{thm:moral-patching}, we can orient $H$ independently of all other chain components and still obtain a DAG that is consistent with $\cE_{\cI}(G^*)$.

Let $C'$ be a \emph{maximal clique} that includes $C$.
By \cref{lem:maximal-clique-picking}, there is an acyclic moral orientation of $H$ such that the vertices of $C'$ appear before all other vertices in $H$.
Let this acyclic moral orientation be the DAG $H'$.
Then, we see that the interventional essential graph $\cE_{V(H) \setminus V(C')}(H')$ by intervening on every single vertex outside of $C'$ has only one chain component, which is precisely the clique $C'$.
By \cref{thm:moral-patching}, we can orient $C'$ independently and still obtain a DAG that is consistent with $\cE_{\cI}(G^*)$.
So, we can order all vertices in $V(C') \setminus V(C)$ \emph{after} the vertices in $V(C)$ and order the vertices within $C$ according to the given desired ordering $\pi$.
\end{proof}

\interventionalmetriclowerbound*
\begin{proof}
Fix an underlying causal graph $G^*$ and consider an arbitrary atomic intervention set $\cI \subseteq V$.
We will prove for $\cI$ and then the claim follows by taking a maximization over all possible atomic intervention sets.
We will prove the two cases separately by mirroring parts of the proof of \cref{lem:strengthened-lb} in how we invoke \cref{lem:hauser-bulmann-strengthened}.

Fix an \emph{arbitrary} atomic intervention set $\cI \subseteq V$ and consider an \emph{arbitrary} DAG $\wt{G}$ that is consistent with $\cE_{\cI}(G^*)$.
That is, $\skel(\wt{G}) = \skel(\cE_{\cI}(G^*))$ and all the oriented edges in $\cE_{\cI}(G^*)$ appear in the same direction in $\wt{G}$.
Fix a chain component $H \in CC(\cE_{\cI}(\wt{G}))$ and let $\cI' \subseteq V$ be any atomic verifying set of $\wt{G}$,
that is, $\cE_{\cI'}(\wt{G}) = \wt{G}$ and $\cE_{\cI'}(\wt{G})[V(H)] = \wt{G}[V(H)]$.
Note that,
\[
\cE_{(\cI' \setminus \cI) \cap V(H)}(\wt{G}[V(H)])
= \cE_{\cI \cup (\cI' \setminus \cI)}(\wt{G})[V(H)]
= \cE_{\cI'}(\wt{G})[V(H)]
= \wt{G}[V(H)]
\]
where the first equality is due to \cref{lem:hauser-bulmann-strengthened} and the last equality is because $\cI'$ is a verifying set of $\wt{G}$.
So, $(\cI' \setminus \cI) \cap V(H)$ is a verifying set for $\wt{G}[V(H)]$, and so is $\cI' \cap V(H)$.
Thus, by minimality of $\overline{\nu}_1$, we have
\begin{equation}
\label{eq:minimality}
\overline{\nu}_1(\wt{G}[V(H)]) \leq w(\cI' \cap V(H))
\end{equation}
for \emph{any} atomic verifying set $\cI' \subseteq V$ of $\wt{G}$.

We now independently lower bound $\overline{\nu}_1(\wt{G}[V(H)])$ by $\zeta^{(1)}_{\cI, H}$ and $\zeta^{(2)}_{\cI, H}$.
To do so, we will construct a DAG $\wt{G}$ that is consistent with the interventional essential graph $\cE_{\cI}(G^*)$ by making vertices of some unoriented clique the prefix of its chain component by using \cref{lem:clique-picking}, and then invoking \cref{eq:minimality} to lower bound the interventional cost in each chain component $H$.
Note that when we fix the ordering of vertices within a chain component, it does not affect the ordering of the vertices outside of that chain component.

\textbf{Lower bounding via $\zeta^{(1)}_{\cI, H}$}:
For each connected component $H \in CC(\cE_{\cI}(G^*))$, fix an arbitrary clique $C$ in $H$.
Suppose the vertices in $C$ are $v_1, \ldots, v_{|C|}$ with $w(v_1) \geq \ldots \geq w(v_{|C|})$.
By \cref{lem:clique-picking}, there exists a valid orientation $\pi$ of $H$ such that all the vertices in $C$ appear at the start of the ordering.
For any such ordering $\pi$, the covered edges are $v_{\pi(1)} \to v_{\pi(2)} \to \ldots \to v_{\pi(|C|)}$ and we know that any atomic verifying set must include a minimum vertex cover of these covered edges due to \cref{thm:necessary-and-sufficient}.
Let $\wt{G}$ be one such DAG which imposes the descending weight ordering $\pi$ on the vertices within $H$, i.e.\ $w(v_{\pi(i)}) = w(v_i)$.
Consider the set of \emph{disjoint} alternating covered edges $\pi^{-1}(1) \to \pi^{-1}(2)$, $\pi^{-1}(3) \to \pi^{-1}(4)$, and so on.
Amongst these disjoint alternating covered edges, at least one endpoint must be intervened upon, incurring a cost of at least $\sum_{\text{even } i} w(v_i)$.
That is, $\overline{\nu}_1(\wt{G}[V(H)]) \geq \sum_{\text{even } i} w(v_i)$.

Since $w(v_1) \geq \ldots \geq w(v_{|C|})$, we see that
\[
w(C)
= w(v_1) + \sum_{\text{even } i} w(v_i) + \sum_{\substack{\text{odd } i\\ i \;\geq\; 3}} w(v_i)
\leq w(v_1) + \sum_{\text{even } i} w(v_i) + \sum_{\substack{\text{odd } i\\ i \;\geq\; 3}} w(v_{i-1})
\leq w(v_1) + 2 \cdot \sum_{\text{even } i} w(v_i) \;.
\]

Therefore,
\[
\overline{\nu}_1(\wt{G}[V(H)])
\geq \sum_{\text{even } i} w(v_i)
\geq \frac{1}{2} \cdot \left( w(V(C)) - w(v_1) \right) \;.
\]

By maximizing amongst the cliques within $H$, we see that $\overline{\nu}_1(\wt{G}[V(H)]) \geq \zeta^{(1)}_{\cI, H}$.

\textbf{Lower bounding via $\zeta^{(2)}_{\cI, H}$}:

For each connected component $H \in CC(\cE_{\cI}(G^*))$, fix an arbitrary vertex $v$ in $H$.
To bound $\gamma_{H,v}$, it suffices to consider \emph{arbitrary} cliques $C_i$ in each disjoint connected components in $H[V \setminus \{v\}]$, and then taking the maximum.

Consider a minimum cost atomic verifying set $\cI$ of $\wt{G}[V(H)]$ with $w(\cI) = \overline{\nu}_1(\wt{G}[V(H)])$.

\textbf{Case 1}: $v \in \cI$.
Then,
$
\overline{\nu}_1(\wt{G}[V(H)])
\geq w(v)
\geq \frac{w(v)}{2}
\geq \frac{1}{2} \cdot \min \left\{ w(v), \sum_{i=1}^t w(V(C_i)) \right\}
$.

By maximizing amongst the cliques within each connected component, we see that $\overline{\nu}_1(\wt{G}[V(H)]) \geq \zeta^{(2)}_{\cI, H}$.

\textbf{Case 2}: $v \not\in \cI$.

By \cref{lem:clique-picking}, there exists DAGs consistent with $\cE(G^*)$ that can be generated by letting $v$ be the first prefix vertex in $\cE(G^*)$, followed by vertices in \emph{descending weight ordering} within each clique $C_i$, across all $t$ components.
Let $\wt{G}$ be one such DAG and suppose the vertices in clique $C_i = \{u_{i,1}, \ldots, u_{i,|C_i|}\}$ have weights $w(u_{i,1}) \geq \ldots w(u_{i,|C_i|})$ and $\pi(v) < \pi(u_{i,1}) < \ldots < \pi(u_{i,|C_i|})$.
We see that the set $\{v \to u_{i,1}, u_{i,1} \to u_{i,2}, \ldots , u_{i,|C_i|-1} \to u_{i,|C_i|}\}_{i=1}^t$ are all covered edges of $\wt{G}$.
By \cref{thm:necessary-and-sufficient}, \emph{any} verification set must include a minimum vertex cover of these edges.
In particular, since $v \not\in \cI$, we must have $\{u_{i,1}\}_{i=1}^t \subseteq \cI$.

Let $A \subseteq E(G^*)$ be the covered edges of $\wt{G}$.
From above, we know that $\{v \to u_{i,1}, u_{i,1} \to u_{i,2}, \ldots , u_{i,|C_i|-1} \to u_{i,|C_i|}\}_{i=1}^t \subseteq A$.
Define $B = A \setminus \{v \to u_{i,1}, u_{i,1} \to u_{i,2} \}_{i=1}^t$ as the remaining covered edges in the above discussion after removing edges covered by $\{u_{i,1}\}_{i=1}^t$.
That is, \emph{conditioned on not using $v$}, $A$'s minimum cost vertex cover has cost $\sum_{i=1}^t w(u_{i,1})$ plus the cost $B$'s minimum cost vertex cover.

For each clique $C_i = \{u_{i,1}, \ldots, u_{i,|C_i|}\}$ amongst the disjoint cliques, consider the set of \emph{disjoint} alternating covered edges $u_{i,2} \to u_{i,3}$, $u_{i,4} \to u_{i,5}$, and so on.
Amongst these disjoint alternating covered edges, at least one endpoint must be chosen for any vertex cover of $B$, incurring a cost of at least $\sum_{\substack{\text{odd } i\\ i \;\geq\; 3}} w(u_{i,j})$.

Since $w(u_{i,1}) \geq \ldots \geq w(u_{i,|C_i|})$, we see that
\[
w(V(C_i))
= w(u_{i,1}) + w(u_{i,2}) + \sum_{\substack{\text{even } i\\ i \;\geq\; 4}} w(u_{i,j}) + \sum_{\substack{\text{odd } i\\ i \;\geq\; 3}} w(u_{i,j})
\leq 2 \cdot \left( w(u_{i,1}) + \sum_{\substack{\text{odd } i\\ i \;\geq\; 3}} w(u_{i,j}) \right) \;.
\]

So, the minimum cost vertex cover of $B$ is at least $\frac{1}{2} \sum_{i=1}^t (w(V(C_i)) - 2 \cdot w(u_{i,1}))$ and
\begin{align*}
\overline{\nu}_1(\wt{G}[V(H)])
&\geq \sum_{i=1}^t w(u_{i,1}) + \frac{1}{2} \sum_{i=1}^t (w(V(C_i)) - 2 \cdot w(u_{i,1}))\\
&\geq \frac{1}{2} \sum_{i=1}^t w(V(C_i))\\
&\geq \frac{1}{2} \cdot \min \left\{ w(v), \sum_{i=1}^t w(V(C_i)) \right\} \;.    
\end{align*}

By maximizing amongst the cliques within each connected component, we see that $\overline{\nu}_1(\wt{G}[V(H)]) \geq \zeta^{(2)}_{\cI, H}$.

\textbf{Putting together}:

Since $\cI^*$ is the minimum cost verifying set,
\begin{multline*}
\overline{\nu}^{\max}_1(G^*)
= \max_{G \in [G^*]} \overline{\nu}_1(G)
\geq \overline{\nu}_1(\wt{G})
= w(\cI^*)\\
\stackrel{(\ast)}{\geq} \sum_{\substack{H \in CC(\cE_{\cI}(G^*))\\|V(H)| \geq 2}} w(\cI^* \cap V(H))
\geq \sum_{\substack{H \in CC(\cE_{\cI}(G^*))\\|V(H)| \geq 2}} \overline{\nu}_1(\wt{G}[V(H)])
\geq \sum_{\substack{H \in CC(\cE_{\cI}(G^*))\\|V(H)| \geq 2}} \max\{\zeta^{(1)}_{\cI,H}, \zeta^{(2)}_{\cI,H}\}
\end{multline*}
where the inequality $(\ast)$ is because some edges may have already been oriented by $\cI$.

Finally, the claim follows by taking the maximum over all possible atomic interventions $\cI \subseteq V$.
\end{proof}

\relatek*
\begin{proof}
\textbf{Proof for $\overline{\nu}^{\max}_k(G^*) \geq \overline{\nu}^{\max}_1(G^*)$}:

Observe that intervening on all vertices in a bounded size intervention one-by-one in an atomic fashion will not increase the cost and can only recover more information about the causal graph.
Let us formalize this:
Suppose $\cI^*_k \subseteq 2^V$ is a minimum cost bounded size verifying set.
Define $\cI = \cup_{S \in \cI^*_k} S$ as an atomic intervention set that involves all vertices in $\cI^*_k$ exactly once.
So, by construction, $w(\cI) \leq w(\cI^*_k)$.
By \cref{thm:necessary-and-sufficient}, we know that $\cI^*_k$ must separate all covered edges of $G^*$.
Meanwhile, by construction, $\cI$ also separates all covered edges of $G^*$ while having $w(\cI) \leq \sum_{S \in \cI^*_k} \sum_{v \in S} w(v) = w(\cI^*_k)$.
Thus, $\overline{\nu}^{\max}_1(G^*) \leq \overline{\nu}^{\max}_k(G^*)$.

\textbf{Proof for $\nu^{\max}_k(G^*) \geq \nu^{\max}_1(G^*)$}:

Observe that
\[
\nu^{\max}_k(G^*)
= \max_{G \in [G^*]} \nu_k(G)
\geq \max_{G \in [G^*]} \left\lceil \nu_1(G) / k \right\rceil
= \left\lceil \max_{G \in [G^*]} \nu_1(G) / k \right\rceil
= \left\lceil \nu^{\max}_1(G^*) / k \right\rceil
\]
where the inequality is due to \cref{thm:nu-k-to-1}.
\end{proof}

\subsection{A competitive adaptive search algorithm}

\danglingsubroutine*
\begin{proof}
Since the underlying graph is a moral DAG, intervening on $v$ or $\{\argmin_{u \in V(H_i) \cap N_H(v)} \pi(u) \}_{i \in [t]}$ ensures that all the outgoing edges of $v$ in $H$ are oriented (\cref{lem:middle}).
Suppose $u_i = \argmin_{u \in V(H_i) \cap N_H(v)} \pi(u)$.
If $\pi(v) > \min_{i \in \{1, \ldots, t\}} \pi(u_i)$, then intervening on $u_1, \ldots, u_t$ will disconnect\footnote{Every path between $H_i$ and $H_j$, for $i \neq j$ will involve an oriented arc. Such arcs will be removed when considering chain components, disconnecting the path.} $H_i$'s from each other\footnote{Without loss of generality, suppose $\pi(u_1) = \min_{i \in \{1, \ldots, t\}} \pi(u_i)$. Orienting the arc $u_1 \to v_H$ triggers Meek rule R1 to orient all $v \to z$ arcs for $z \not\in H_1$, thus disconnecting $H_i$'s from each other.}.
Otherwise, if $\pi(v) < \min_{i \in \{1, \ldots, t\}} \pi(u_i)$, \cref{lem:middle} tells us that intervening on $u_i$ will orient all $v \to z$ arcs for $z \in H_i$. In both cases, we orient all the outgoing edges of $v$ within $H$.

The if-case of \texttt{ResolveDangling} directly intervenes on $v$ while the else-case of \texttt{ResolveDangling} repeatedly recurses on a connected subgraph of $H_i[V']$, towards the source $u_i$.
Since $K_{H_i}$ is a 1/2-clique separator of $H_i$ at each iteration, the size of $V'$ is at least halved in each iteration of the while-loop, and there will be at most $\cO(\log n)$ iterations.
By the $\zeta^{(2)}$ term of \cref{thm:interventional-metric-lower-bound}, we see that for each iteration, the cost of finding all the $\{u_i\}_{i \in [t]}$ has a cost at most $2 \cdot \overline{\nu}^{\max}_1(G^*)$.
Put together, we see that $w(\cI) \in \cO(\log n \cdot \overline{\nu}^{\max}_1(G^*))$.
\end{proof}

\atmostoneincoming*
\begin{proof}
Apply \cref{lem:at-most-one-incoming-generalized} with $K = \{v\}$.
\end{proof}

\lognphasessuffice*
\begin{proof}
In each phase, we are essentially breaking up the graph into small subgraphs using \cref{thm:chordal-separator} where the size of the chain components decreases by a factor of two.

Note that we do not intervene on \emph{all} the vertices in the clique separator $K_H$, but only intervene on $V(K_H) \setminus \{v_H\}$.
So, we need to argue that partites $A$ and $B$ (with respect to the 1/2-clique separator $K_H$) are disconnected before we recurse in the next phase.
To do so, we use \cref{lem:dangling-subroutine}: invoking $\texttt{ResolveDangling}$ on $(Z_{v_H}, w, v_H)$ ensures that all outgoing edges from $v_H$ will be oriented, so we obtain two disconnected chain component partites $A$ and $B$.

Since the maximum chain component size initially at most $n$ and is always halved after a phase, \cref{alg:weighted-search} terminates after $\cO(\log n)$ phases.
\end{proof}

\costofeachphase*
\begin{proof}
By the $\zeta^{(1)}$ term of \cref{thm:interventional-metric-lower-bound}, intervening on $V(K_H) \setminus \{v_H\}$ across all chain components $H \in CC(\cE_{\cI_i}(G^*))$ incurs a cost of at most $2 \cdot \overline{\nu}^{\max}_1(G^*)$.
By \cref{lem:dangling-subroutine}, \texttt{ResolveDangling} incurs returns intervention set $\cI$ of weight $w(\cI) \in \cO(\log n \cdot \overline{\nu}^{\max}_1(G^*))$.
\end{proof}

\weightedsearch*
\begin{proof}
Direct consequence of combining \cref{lem:logn-phases-suffice} and \cref{lem:cost-of-each-phase}.

To analyze the running time, let us consider the running time of the subroutines:
\begin{itemize}
    \item \cref{alg:weighted-search} has $\cO(\log n)$ phases where each phase may execute the \texttt{ResolveDangling} subroutine.
    \item There are at most $t \leq n$ components within the \texttt{ResolveDangling} subroutine and the while loops terminates after $\cO(\log n)$ iterations.
    \item Throughout, computing clique separators can be done in $\cO(m)$ time (\cref{thm:chordal-separator}, \cite{gilbert1984separatorchordal}).
    \item Throughout, executing Meek rules after performing an intervention can be done in $\cO(d \cdot m)$ time (\cref{sec:appendix-meek-rules}, \cite{wienobst2021extendability}).
    \item Within the \texttt{ResolveDangling} subroutine, finding the chain component $Q$ can be done in $\cO(m)$ time.
\end{itemize}
Thus, \cref{alg:weighted-search} runs in $\cO(n \cdot \log^2(n) \cdot d \cdot m)$ time.
Since $d \leq n$ and $m \leq n^2$, the overall running time is polynomial in $n$.
\end{proof}

\subsection{Handling the generalized cost objective}

\ssource*
\begin{proof}
By definition of a source node, all edges in $G$ will point \emph{away} from $v_1$.
Meanwhile, since $G$ is a clique, every other vertex $v_i$ will have an arc $v_1 \to v_i$.
So, $S_{source}$ is the unique set in $\mathcal{S}$ that has a vertex without any incoming arcs from the other sets.
\end{proof}

To prove \cref{lem:at-most-one-incoming-generalized}, we rely on the next lemma (\cref{lem:middle-bounded}) which generalizes \cref{lem:middle}: the latter is the special case of the former where $S$ is a single vertex.
Given a moral DAG, \cref{lem:middle} of \cite{choo2023subset} tells us that intervening on a single vertex $w$ will split up the graph into separate chain components such that all ancestors of $w$ will belong in a single chain component.
\cref{lem:middle-bounded} generalizes this fact to the setting of bounded size interventions.

\begin{lemma}
\label{lem:middle-bounded}
Let $G = (V,E)$ be a moral DAG and $\pi$ be an arbitrary consistent ordering of $G$.
Intervening on vertex set $S = \{s_1, s_2, \ldots, s_k\} \subseteq V$ orients all edges $u \to v$ with $s_1 \in \Des(u) \cap \Anc(v)$, where $\pi(s_1) < \pi(s_2) < \ldots < \pi(s_k)$.
\end{lemma}
\begin{proof}
Note that $u \not\in S$ as $s_1 \in \Des(u)$, but $v$ could possibly be a vertex in $S$.

By \cref{lem:intermediate-direct-arcs-exist}, we know that there are arcs $u \to w$ for all $w \in \Des(u) \cap \Anc(v)$.

Let
\[
s_i = \argmax_{\substack{z \in S;\\ z \in \Des(u) \cap (\Anc(v) \cup \{v\})}} \{ \pi(z) \}
\]
be a vertex in $S$ that lies between $u$ and $v$, with the largest ordering.
The vertex $s_i$ is well-defined because $s_1 \in \Des(u) \cap \Anc(v) \subseteq \Des(u) \cap (\Anc(v) \cup \{v\})$.

If $s_i = v$, then $u \to v$ is trivially oriented when we intervene on $S$ because $u \not\in S$.
In the rest of the proof, we may assume that $s_i \neq v$, i.e.\ $s_i \in \Anc(v)$.
Let
\[
w = \argmax_{\substack{z \in \Des(s_i) \cap (\Anc(v) \cup \{v\});\\ (s_i \to z) \in E}} \{ \pi(z) \}
\]
denote the vertex with incoming arc from $s_i$ and ancestral to $v$, with the largest ordering.
The vertex $w$ is well-defined because $s_i \in \Anc(v)$ and thus there is a sequence of directed arcs from $s_i$ to $v$.
Note that $w$ could be $v$ and $w \not\in S$ by maximality of $s_i$.

When we intervene on $S$, we recover all arc directions incident to $s_i$, except maybe the arcs internal within $S$.
In particular, we will recover the arcs $u \to s_i$ and $s_i \to w$.

If $w = v$, then Meek rule R2 recovers $u \to w = v$ via $u \to s_i \to w \sim u$.

Otherwise, if $w \neq v$, then let $w = w_0 \to w_1 \to \ldots \to w_\ell = v$ be the sequence of directed arcs from $w$ to $v$ in $G$.
By maximality of $w$, there is no arc from $s_i$ to any of the vertices $\{w_1, \ldots, w_\ell\}$.
So, by repeatedly applying Meek R1, we recover
\begin{itemize}
    \item $w_0 \to w_1$ via $s_i \to w_0 \sim w_1$
    \item $w_1 \to w_2$ via $w_0 \to w_1 \sim w_2$
    \item $\ldots$
    \item $w_{\ell-1} \to w_{\ell}$ via $w_{\ell-2} \to w_{\ell-1} \sim w_{\ell}$
\end{itemize}
Furthermore, we know that the arcs $u \to w_0, u \to w_1, \ldots u \to w_{\ell}$ exist due to \cref{lem:intermediate-direct-arcs-exist}.
So, by repeatedly applying Meek R2, we recover
\begin{itemize}
    \item $u \to w_0$ via $u \to s_i \to w_0 \sim u$
    \item $u \to w_1$ via $u \to w_0 \to w_1 \sim u$
    \item $\ldots$
    \item $u \to w_{\ell}$ via $u \to w_{\ell-1} \to w_{\ell} \sim u$
\end{itemize}
That is, the arc $u \to w_{\ell} = v$ will be oriented.
\end{proof}

\atmostoneincominggeneralized*
\begin{proof}
Let us denote $H_{source}$ as the source node of $H$ and $K_{source}$ as the source node of $K$.

\textbf{Case 1:} $K_{source} = H_{source}$

Suppose, for a contradiction, that there was a chain component in $\cE_{\cI \cup \{V(K)\}}(H)$ with incoming arcs into $K$ in $G$.
Since $G$ is moral, this chain component must have an edge with $K_{source}$.
However, since $K_{source} = H_{source}$, this arc must be \emph{outgoing} from $K_{source}$.
Contradiction.

\textbf{Case 2:} $K_{source} \neq H_{source}$, i.e.\ $K_{source} \in \Des(H_{source})$

Recall that chain components do not have oriented arcs, so $H$ must be moral.
Since $K$ is a clique in the chain component $H$, there was an unoriented directed path from $H_{source} \to u_1 \to \ldots \to u_{last} \to K_{source}$ before intervening on $K$.
Since Meek rules can only orient arcs with an endpoint that is a descendant of vertices in $K$, we see that the arcs $H_{source} \to u_1 \to \ldots \to u_{last}$ remain unoriented after intervening on $K$.

\emph{Claim 2.1: There exists one such chain component.}
Let $A$ be the chain component containing $H_{source}$ after intervening on $K$.
From the above discussion, $A$ has an arc into $K$ in $G$, namely $u_{last} \to K_{source}$.
For $A$ to have any incoming arcs from $K$, $A$ must contain some descendant of $K_{source}$.
However, by \cref{lem:middle-bounded}, any arc joining an ancestor and descendant of $K_{source}$ would be oriented, thus ancestors and descendants of $K_{source}$ will belong in different chain components in $\cE_{\cI \cup \{V(K)\}}(H)$.
Thus, $A$ only has incoming arcs into $K$ in $G$.

\emph{Claim 2.2: There does not exist two such chain components.}
Suppose, for a contradiction, that there is another chain components $B$ in $\cE_{\cI \cup \{V(K)\}}(H)$ with incoming arcs into $K$ in $G$.
Since $G$ is moral, $B$ must have an edge into $K_{source}$, say $b \to K_{source}$.
Again, since $G$ is moral, there must be an edge between $b$ and $u_{last}$.
Since Meek rules can only orient arcs with an endpoint the arc $b \sim u_{last}$ remains unoriented after intervening on $K$, so $A$ and $B$ are actually the same chain component.
Contradiction.

\textbf{Running time}
We can enumerate over all chain components of $H$ and checking each edge at most twice in order to determine whether there is a chain component in $\cE_{\cI \cup \{V(K)\}}(H)$ with incoming arcs into $K$ in $G$, and if so find it.
\end{proof}

\cliqueinterventionsubroutinegeneralized*
\begin{proof}
By construction and \cref{lem:labelling-scheme}, each partite in $S$ has at most $k$ vertices.

When $k = 1$, the output $|S| = |V(C)|$ and each vertex appears exactly once in $S$.

When $k > 1$, the output $|S| \leq \left\lceil \frac{|V(C)|}{k'} \right\rceil \cdot \left\lceil \log_{\left\lceil \frac{|V(C)|}{k'} \right\rceil} |V(C)| \right\rceil$ and each vertex appears $\left\lceil \log_{\left\lceil \frac{|V(C)|}{k'} \right\rceil} |V(C)| \right\rceil$ times in $S$, where $k' = \min\{k, |V(C)|/2\} > 1$.
Since $k' \leq k$, we have $\left\lceil \frac{|V(C)|}{k'} \right\rceil \in \cO \left( \frac{|V(C)|}{k} \right)$.
So, it remains to bound $\left\lceil \log_{\left\lceil \frac{|V(C)|}{k'} \right\rceil} |V(C)| \right\rceil$.

When $1 < k \leq \frac{|V(C)|}{2}$, we see that $k' = k$.
So,
\[
\left\lceil \log_{\left\lceil \frac{|V(C)|}{k'} \right\rceil} |V(C)| \right\rceil
= \left\lceil \log_{\left\lceil \frac{|V(C)|}{k} \right\rceil} |V(C)| \right\rceil
= \left\lceil \frac{\log |V(C)|}{\log \left\lceil \frac{|V(C)|}{k} \right\rceil} \right\rceil
\in \cO(\log k)
\]

For the final asymptotic inclusion, consider the following argument with $\log$ being base 2 and $1 < k \leq x/2$:
\begin{align*}
&\; \frac{\log x}{\log(x/k)} \leq \log k + 1\\
\iff &\; \log x \leq \log k \cdot \log(x/k) + \log(x/k)\\
\iff &\; \log k \leq \log k \cdot \log(x/k)\\
\iff &\; 1 \leq \log(x/k)\\
\iff &\; 2 \leq x/k
\end{align*}

When $k > \frac{|V(C)|}{2}$, we see that $k' = \frac{|V(C)|}{2}$.
So,
\[
\left\lceil \log_{\left\lceil \frac{|V(C)|}{k'} \right\rceil} |V(C)| \right\rceil
= \left\lceil \log_2 |V(C)| \right\rceil
\leq \left\lceil \log_2 2k \right\rceil
\in \cO(\log k)
\]

The claim follows since we always have
\[
\left\lceil \log_{\left\lceil \frac{|V(C)|}{k'} \right\rceil} |V(C)| \right\rceil
\in \cO(\log k)
\]
\end{proof}

\interventionalmetriclowerboundgeneralized*
\begin{proof}
The proof is similar to \cref{thm:interventional-metric-lower-bound} but we specialize the bounds to take into account of \cref{eq:generalized-cost}.

\textbf{Common argument}

Fix an \emph{arbitrary} intervention set $\cI \subseteq 2^V$.
We will prove the two cases separately by mirroring parts of the proof of \cref{lem:strengthened-lb} in how we invoke \cref{lem:hauser-bulmann-strengthened}.

Consider an arbitrary DAG $\wt{G} \in [G^*]$.
Let $\cI' \subseteq V$ be any atomic verifying set of $\wt{G}$ and fix a chain component $H \in CC(\cE_{\cI}(G^*))$.
That is, suppose $\cE_{\cI'}(G^*) = \wt{G}$ and $\cE_{\cI'}(G^*)[V(H)] = \wt{G}[V(H)]$.
Then,
\[
\cE_{(\cI' \setminus \cI) \cap V(H)}(\wt{G}[V(H)])
= \cE_{\cI \cup (\cI' \setminus \cI)}(\wt{G})[V(H)]
= \cE_{\cI'}(\wt{G})[V(H)]
= \wt{G}[V(H)]
\]
where the first equality is due to \cref{lem:hauser-bulmann-strengthened} and the last equality is because $\cI'$ is a verifying set of $\wt{G}$.
So, $(\cI' \setminus \cI) \cap V(H)$ is a verifying set for $\wt{G}[V(H)]$, and so is $\cI' \cap V(H)$.
Thus, by minimality of $\nu_1$ and $\nu_k$, we have
\begin{equation}
\label{eq:minimality-generalized-atomic}
\nu_1(\wt{G}[V(H)]) \leq |\cI' \cap V(H)|
\qquad \text{and} \qquad
\overline{\nu}_1(\wt{G}[V(H)]) \leq w(\cI' \cap V(H))
\end{equation}
for \emph{any} atomic verifying set $\cI' \subseteq V$ of $\wt{G}$.

Repeating the exact same argument for bounded size verifying sets, we have
\begin{equation}
\label{eq:minimality-generalized-bounded-size}
\nu_k(\wt{G}[V(H)]) \leq |\cI' \cap V(H)|
\qquad \text{and} \qquad
\overline{\nu}_k(\wt{G}[V(H)]) \leq w(\cI' \cap V(H))
\end{equation}
for \emph{any} bounded size verifying set $\cI' \subseteq 2^V$ of $\wt{G}$.

We now independently lower bound via $\zeta^{(3)}_{\cI, H}$, $\zeta^{(4)}_{\cI, H}$, $\zeta^{(5)}_{\cI, H}$, and $\zeta^{(6)}_{\cI, H}$ by using \cref{lem:clique-picking}: in any interventional essential graph, we can always pick a consistent ordering by making any unoriented clique the prefix of its chain component.

\textbf{Case A: Lower bounding via $\zeta^{(3)}_{\cI, H}$ when $k = 1$}:

Fix an arbitrary clique $C$ in $H$.
Suppose the vertices in $C$ are $v_1, \ldots, v_{|C|}$ with $w(v_1) \geq \ldots \geq w(v_{|C|})$.
By \cref{lem:clique-picking}, there exists a valid orientation $\pi$ of $H$ such that all the vertices in $C$ appear at the start of the ordering.
For any such ordering $\pi$, the covered edges are $v_{\pi(1)} \to v_{\pi(2)} \to \ldots \to v_{\pi(|C|)}$ and we know that any atomic verifying set must include a minimum vertex cover of these covered edges due to \cref{thm:necessary-and-sufficient}.

Fix the ordering $\pi$ where $w(v_{\pi(i)}) = w(v_i)$ and let the DAG $\wt{G} \in [G^*]$ correspond to this ordering, i.e.\ $\pi$ is in descending weight ordering.
Consider the set of \emph{disjoint} alternating covered edges $\pi^{-1}(1) \to \pi^{-1}(2)$, $\pi^{-1}(3) \to \pi^{-1}(4)$, and so on.
Amongst these disjoint alternating covered edges, at least one endpoint must be intervened upon, incurring a cost of at least $\sum_{\text{even } i} w(v_i)$.
That is, $\nu_1(\wt{G}) \geq \sum_{\text{even } i} w(v_i)$.
From the proof of \cref{thm:interventional-metric-lower-bound}, we know that
\[
\overline{\nu}_1(\wt{G}[V(H)]) \geq \frac{1}{2} \cdot \left( w(V(C)) - \max_{v \in V(C)} w(v) \right) \;.
\]
Meanwhile, \cref{thm:clique-covered-edges-and-lower-bound} tells us that orienting $C$ requires at least $|V(C)|/2$ atomic interventions even if we allow randomization and adaptivity.
So,
\[
\nu_1(\wt{G}[V(H)]) \geq |V(C)|/2 \;.
\]

Therefore, for \emph{any} atomic verifying set $\cI$ of $\wt{G}[V(H)]$,
\begin{align*}
\alpha \cdot w(\cI) + \beta \cdot |\cI|
&\geq \alpha \cdot \overline{\nu}_1(\wt{G}[V(H)]) + \beta \cdot \nu_1(\wt{G}[V(H)])\\
&\geq \alpha \cdot \left( \frac{1}{2} \cdot \left( w(V(C)) - \max_{v \in V(C)} w(v) \right) \right) + \beta \cdot \left( |V(C)|/2 \right)\\
&= \frac{1}{2} \cdot \left\{ \alpha \cdot \left( w(V(C)) - \max_{v \in V(C)} w(v) \right) + \beta \cdot |V(C)| \right\} \;.    
\end{align*}
By maximizing amongst the cliques within $H$, we see that $\alpha \cdot w(\cI) + \beta \cdot |\cI| \geq \zeta^{(3)}_{\cI, H}$.

\textbf{Case B: Lower bounding via $\zeta^{(4)}_{\cI, H}$ when $k = 1$}:

It suffices to prove this for \emph{arbitrary} cliques $C_i$ in each disjoint connected components in $H[V \setminus \{v\}]$, and then taking the maximum.
Consider a minimum cost atomic verifying set $\cI$ of $\wt{G}[V(H)]$.

\textbf{Case 1}: $v \in \cI$.
Then,
\begin{align*}
\alpha \cdot w(\cI) + \beta \cdot |\cI|
&\geq \alpha \cdot \overline{\nu}_1(\wt{G}[V(H)]) + \beta \cdot \nu_1(\wt{G}[V(H)])\\
&\geq \alpha \cdot w(v) + \beta\\
&\geq \frac{1}{2} \cdot \left\{ \alpha \cdot w(v) + \beta, \sum_{i=1}^t \alpha \cdot w(V(C_i)) + \beta \cdot |V(C_i)| \right\}
\end{align*}

By maximizing amongst the cliques within each connected component, we see that $\alpha \cdot w(\cI) + \beta \cdot |\cI| \geq \zeta^{(4)}_{\cI, H}$.

\textbf{Case 2}: $v \not\in \cI$.

By \cref{lem:clique-picking}, there exists DAGs consistent with $\cE(G^*)$ that can be generated by letting $v$ be the first prefix vertex in $\cE(G^*)$, followed by vertices in \emph{descending weight ordering} within each clique $C_i$, across all $t$ components.
Let $\wt{G}$ be one such DAG and suppose the vertices in clique $C_i = \{u_{i,1}, \ldots, u_{i,|C_i|}\}$ have weights $w(u_{i,1}) \geq \ldots w(u_{i,|C_i|})$ and $\pi(v) < \pi(u_{i,1}) < \ldots < \pi(u_{i,|C_i|})$.
We see that the set $\{v \to u_{i,1}, u_{i,1} \to u_{i,2}, \ldots , u_{i,|C_i|-1} \to u_{i,|C_i|}\}_{i=1}^t$ are all covered edges of $\wt{G}$.
By \cref{thm:necessary-and-sufficient}, \emph{any} verification set must include a minimum vertex cover of these edges.
In particular, since $v \not\in \cI$, we must have $\{u_{i,1}\}_{i=1}^t \subseteq \cI$.

\emph{Conditioned on not using $v$}, we know, from the proof of \cref{thm:interventional-metric-lower-bound}, that
\[
\overline{\nu}_1(\wt{G}[V(H)])
\geq \frac{1}{2} \cdot \sum_{i=1}^t w(V(C_i)) \;.
\]
Meanwhile, \cref{thm:clique-covered-edges-and-lower-bound} tells us that orienting all the $C_i$'s require at least $\sum_{i=1}^t |V(C_i)|/2$ atomic interventions, even if we allow randomization and adaptivity.
So,
\[
\nu_1(\wt{G}[V(H)])
\geq \frac{1}{2} \cdot \sum_{i=1}^t |V(C_i)| \;.
\]
Therefore,
\begin{align*}
\alpha \cdot w(\cI) + \beta \cdot |\cI|
&\geq \alpha \cdot \overline{\nu}_1(\wt{G}[V(H)]) + \beta \cdot \nu_1(\wt{G}[V(H)])\\
&\geq \alpha \cdot \left( \frac{1}{2} \cdot \sum_{i=1}^t w(V(C_i)) \right) + \beta \cdot \left( \frac{1}{2} \cdot \sum_{i=1}^t |V(C_i)| \right)\\
&= \frac{1}{2} \cdot \left( \alpha \cdot \sum_{i=1}^t w(V(C_i)) + \beta \cdot |V(C_i)| \right)
\end{align*}

By maximizing amongst the cliques within each connected component, we see that $\alpha \cdot w(\cI) + \beta \cdot |\cI| \geq \zeta^{(4)}_{\cI, H}$.

\textbf{Case C: Lower bounding via $\zeta^{(5)}_{\cI, H}$ when $k > 1$}:

We use the exact same proof outline as $\zeta^{(3)}_{\cI, H}$ while invoking \cref{thm:relate-k}.
This gives the following inequalities:
\[
\overline{\nu}_k(\wt{G}[V(H)]) 
\geq \overline{\nu}_1(\wt{G}[V(H)])
\geq \frac{1}{2} \cdot \left( w(V(C)) - \max_{v \in V(C)} w(v) \right) \;.
\]
and
\[
\nu_k(\wt{G}[V(H)])
\geq \left\lceil \frac{\nu_1(\wt{G}[V(H)])}{k} \right\rceil
\geq \left\lceil \frac{|V(C)|}{2k} \right\rceil \;.
\]

Therefore, for \emph{any} bounded size verifying set $\cI$ of $\wt{G}[V(H)]$,
\begin{align*}
\alpha \cdot w(\cI) + \beta \cdot |\cI|
&= \frac{1}{2} \cdot \left\{ \alpha \cdot \left( w(V(C)) - \max_{v \in V(C)} w(v) \right) + \beta \cdot \frac{|V(C)|}{k} \right\} \;.    
\end{align*}
By maximizing amongst the cliques within $H$, we see that $\alpha \cdot w(\cI) + \beta \cdot |\cI| \geq \zeta^{(5)}_{\cI, H}$.

\textbf{Case D: Lower bounding via $\zeta^{(6)}_{\cI, H}$ when $k > 1$}:

We use the exact same proof outline as $\zeta^{(4)}_{\cI, H}$ while invoking \cref{thm:relate-k}.
Let $\cI$ be an \emph{arbitrary} bounded size verifying set of $\wt{G}[V(H)]$.

\emph{Conditioned on using $v$}, we trivially get $\alpha \cdot w(\cI) + \beta \cdot |\cI| \geq \alpha \cdot w(v) + \beta$ like before.
By maximizing amongst the cliques within $H$, we see that $\alpha \cdot w(\cI) + \beta \cdot |\cI| \geq \zeta^{(6)}_{\cI, H}$.

Meanwhile, \emph{conditioned on not using $v$}, we get the following inequalities:
\[
\overline{\nu}_k(\wt{G}[V(H)])
\geq \overline{\nu}_1(\wt{G}[V(H)])
\geq \frac{1}{2} \cdot \sum_{i=1}^t w(V(C_i)) \;.
\]
and
\[
\nu_k(\wt{G}[V(H)])
\geq \left\lceil \frac{\nu_1(\wt{G}[V(H)])}{k} \right\rceil
\geq \left\lceil \frac{1}{2} \cdot \sum_{i=1}^t \frac{|V(C_i)|}{k} \right\rceil
\geq \frac{1}{2} \cdot \sum_{i=1}^t \frac{|V(C_i)|}{k} \;.
\]

Therefore,
\begin{align*}
\alpha \cdot w(\cI) + \beta \cdot |\cI|
&= \frac{1}{2} \cdot \left\{ \alpha \cdot \sum_{i=1}^t w(V(C_i)) + \beta \cdot \frac{|V(C_i)|}{k} \right\} \;.
\end{align*}
By maximizing amongst the cliques within $H$, we see that $\alpha \cdot w(\cI) + \beta \cdot |\cI| \geq \zeta^{(6)}_{\cI, H}$.

\textbf{Putting together}

For $k = 1$, recall that $\cI^*_1 \subseteq V$ is the atomic intervention set optimizing \cref{eq:generalized-cost} such that $\cE_{\cI^*_1}(G^*) = G^*$.
So,
\begin{align*}
\OPT_1
&= \alpha \cdot w(\cI^*_1) + \beta \cdot |\cI^*_1|\\
&\stackrel{(\ast)}{\geq} \sum_{\substack{H \in CC(\cE_{\cI}(G^*))\\|V(H)| \geq 2}} \alpha \cdot w(\cI^*_1 \cap V(H)) + \beta \cdot |\cI^*_1 \cap V(H)|\\
&\geq \sum_{\substack{H \in CC(\cE_{\cI}(G^*))\\|V(H)| \geq 2}} \alpha \cdot \overline{\nu}_1(\wt{G}[V(H)]) + \beta \cdot \nu_1(\wt{G}[V(H)])\\
&\geq \sum_{\substack{H \in CC(\cE_{\cI}(G^*))\\|V(H)| \geq 2}} \max\{\zeta^{(3)}_{\cI,H}, \zeta^{(4)}_{\cI,H}\}
\end{align*}
where the inequality $(\ast)$ is because some edges may have already been oriented by $\cI$ and the last two inequalities follow from arguments in cases A and B.
Finally, the claim follows by taking the maximum over all possible atomic interventions $\cI \subseteq V$.

For $k > 1$, recall that $\cI^*_k \subseteq 2^V$ is the bounded size intervention set optimizing \cref{eq:generalized-cost} such that $\cE_{\cI^*_k}(G^*) = G^*$.
So,
\begin{align*}
\OPT_k
&= \sum_{S \in \cI^*_k} \alpha \cdot w(S) + \beta \cdot |S|\\
&\stackrel{(\ast)}{\geq} \sum_{\substack{H \in CC(\cE_{\cI}(G^*))\\|V(H)| \geq 2}} \sum_{S \in \cI^*_k} \alpha \cdot w(S \cap V(H)) + \beta \cdot |S \cap V(H)|\\
&\geq \sum_{\substack{H \in CC(\cE_{\cI}(G^*))\\|V(H)| \geq 2}} \alpha \cdot \overline{\nu}_k(\wt{G}[V(H)]) + \beta \cdot \nu_k(\wt{G}[V(H)])\\
&\geq \sum_{\substack{H \in CC(\cE_{\cI}(G^*))\\|V(H)| \geq 2}} \max\{\zeta^{(5)}_{\cI,H}, \zeta^{(6)}_{\cI,H}\}
\end{align*}
where the inequality $(\ast)$ is because some edges may have already been oriented by $\cI$ and the last two inequalities follow from arguments in cases C and D.
Finally, the claim follows by taking the maximum over all possible bounded size interventions $\cI \subseteq 2^V$.
\end{proof}

\danglingsubroutinegeneralized*
\begin{proof}
The proof strategy exactly follows \cref{lem:dangling-subroutine} except we have to account for the subroutine call to \texttt{CliqueIntervention} via \cref{lem:clique-intervention-subroutine-generalized}.

Since the underlying graph is a moral DAG, intervening on $v$ or $\argmin_{u \in ( \cup_{i=1}^t V(H_1) ) \cap N_H(v)} \pi(u)$ ensures that the partites will indeed become separated.
Suppose $u_i = \argmin_{u \in V(H_i) \cap N_H(v)} \pi(u)$.
If $\pi(v_H) > \min_{i \in \{1, \ldots, t\}} \pi(u_i)$, then intervening on $u_1, \ldots, u_t$ will disconnect\footnote{Every path between $H_i$ and $H_j$, for $i \neq j$ will involve an oriented arc. Such arcs will be removed when considering chain components, disconnecting the path.} $H_i$'s from each other\footnote{Without loss of generality, suppose $\pi(u_1) = \min_{i \in \{1, \ldots, t\}} \pi(u_i)$. Orienting the arc $u_1 \to v_H$ triggers Meek rule R1 to orient all $v \to z$ arcs for $z \not\in H_1$, thus disconnecting $H_i$'s from each other.}.
Otherwise, if $\pi(v_H) < \min_{i \in \{1, \ldots, t\}} \pi(u_i)$, \cref{lem:middle} tells us that intervening on $u_i$ will orient all $v_H \to z$ arcs for $z \in H_i$.

The if-case of \texttt{ResolveDanglingGeneralized} directly intervenes on $v$ while the else-case of \texttt{ResolveDanglingGeneralized} repeatedly recurses on a connected subgraph of $H_i[V']$, towards the source $\argmin_{u \in ( \cup_{i=1}^t V(H_1) ) \cap N_H(v)} \pi(u)$.
Since the size of $V'$ is at least halved in each iteration of the while-loop, it can have at most $\cO(\log n)$ iterations.

Note that, in each iteration (out of $O(\log n)$ iterations) of \texttt{ResolveDanglingGeneralized} except the last one, we partition the clique seperators into sets of size at most $k$ and intervene on them.  Suppose $S \subseteq \cI$ is the intervention set output of that iteration, by \cref{thm:interventional-metric-lower-bound-generalized} the cost of this step, that is $\cost(S, \alpha, \beta, k)\in \cO(\OPT_k)$ for all $k \geq 1$. In the last step of our prcoedure \texttt{ResolveDanglingGeneralized}, we invoke \texttt{CliqueIntervention} and in the remainder of the proof we bound the cost incurred by this subroutine.

\textbf{Accounting for \texttt{CliqueIntervention} subroutine calls}

Suppose $S \subseteq \cI$ is the intervention set output of \texttt{CliqueIntervention} on some clique $C$ in the last step.

When $k = 1$, we know from \cref{lem:clique-intervention-subroutine-generalized} that $|S| = |V(C)|$ and each vertex appears exactly once in $S$.
By $\zeta^{(4)}$ term of \cref{thm:interventional-metric-lower-bound-generalized}, $\cost(S, \alpha, \beta, 1) \in \cO(\OPT_1)$.
So, across all $\cO(\log n)$ iterations, $\cost(\cI, \alpha, \beta, 1) \in \cO(\log n \cdot \OPT_1)$.

When $k > 1$, we know from \cref{lem:clique-intervention-subroutine-generalized} that $|S| \in \cO(\log k \cdot |V(C)| / k)$ and each vertex appears at most $\cO(\log k)$ in $S$.
By $\zeta^{(6)}$ term of \cref{thm:interventional-metric-lower-bound-generalized}, $\cost(S, \alpha, \beta, k) \in \cO(\log k \cdot \OPT_k)$, where the $\cO(\log k)$ multiplicity of each vertex occurrence increases the $\alpha$ term while the $\cO(\log k)$ multiplicative factor size overhead increases the $\beta$ term.
As we invoke, \texttt{CliqueIntervention} only in the last step and as we incur only a cost of $\cO(\OPT_k)$ in all the remaining steps, our total cost across all $\cO(\log n)$ iterations is, $\cost(\cI, \alpha, \beta, k) \in \cO(\log n \cdot \OPT_k + \log k \cdot \OPT_k)$. We conclude the proof.
\end{proof}

\lognphasessufficegeneralized*
\begin{proof}
The proof exactly follows \cref{lem:logn-phases-suffice} except we use \texttt{ResolveDanglingGeneralized} instead of \texttt{ResolveDangling} to ensure that the partites $A$ and $B$ (with respect to 1/2-clique separator $K_H$) are separated before we recurse in the next phase.
For completeness, we repeat the entire argument below.

In each phase, we are essentially breaking up the graph into small subgraphs using \cref{thm:chordal-separator} where the size of the chain components decrease by a factor of two.

Note that we do not intervene on \emph{all} the vertices in the clique separator $K_H$, but only intervene on $V(K_H) \setminus \{v_H\}$ on line 8, we need to argue that partites $A$ and $B$ (with respect to the 1/2-clique separator $K_H$) are separated before we recurse in the next phase.
\cref{lem:dangling-subroutine-generalized} ensures that $\texttt{ResolveDanglingGeneralized}$ on $(Z_{v_H}, w, v_H)$ separates any connected components that may be ``dangling'' from $v_H$ after intervening on $V(K_H) \setminus \{v_H\}$.

Since the maximum chain component size initially at most $n$ and is always halved after a phase, \cref{alg:weighted-search-generalized} terminates after $\cO(\log n)$ phases.
\end{proof}

\costofeachphasegeneralized*
\begin{proof}
Fix an arbitrary phase $i$ with an intermediate interventional essential graph $\cE_{\cI}(G^*)$.
Suppose that $\cJ_i \subseteq 2^V$ is the intervention set computed by \texttt{ALG-GENERALIZED} in phase $i$.
By construction, $\cJ_i$ is made up by at most two calls from \texttt{CliqueIntervention} -- one from line 9 and one within \texttt{ResolveDanglingGeneralized}.

\textbf{Case $k = 1$:}
By the $\zeta^{(3)}$ term of \cref{thm:interventional-metric-lower-bound-generalized}, a function call to \texttt{CliqueIntervention} from line 9 incurs a cost of $\cO(\OPT_1)$.
By \cref{lem:dangling-subroutine-generalized}, \texttt{ResolveDanglingGeneralized} incurs a cost of $\cO(\log n \cdot \OPT_1)$.

\textbf{Case $k > 1$:}
By the $\zeta^{(5)}$ term of \cref{thm:interventional-metric-lower-bound-generalized}, a function call to \texttt{CliqueIntervention} from line 9 incurs a cost of $\cO(\OPT_k)$.
By \cref{lem:dangling-subroutine-generalized}, \texttt{ResolveDanglingGeneralized} incurs a cost of $\cO((\log n + \log k) \cdot \OPT_k)$.
\end{proof}

\weightedsearchgeneralized*
\begin{proof}
Direct consequence of combining \cref{lem:logn-phases-suffice-generalized} and \cref{lem:cost-of-each-phase-generalized}.

The algorithm runs in polynomial time because the following it uses a polynomial number of phases and each phase can be computed in polynomial time:
\begin{itemize}
    \item Computation of 1/2-clique separators run in polynomial time (\cref{thm:chordal-separator})
    \item Enumerating all maximal cliques in chordal graph can be done in polynomial time\footnote{e.g.\ see \url{https://en.wikipedia.org/wiki/Chordal_graph}}
    \item Labelling scheme computation of \cref{lem:labelling-scheme} can be computed in polynomial time
    \item Applying Meek rules till convergence can be made to run in polynomial time \cite{wienobst2021extendability}
\end{itemize}
\end{proof}

%% file: appendix-experiments.tex
\section{Experiments}
\label{sec:appendix-experiments}

In this section, we provide more details about our experiments.

All our experiments are conducted on an Ubuntu server with two AMD EPYC 7532 CPU and 256GB DDR4 RAM.
Source code implementation and experimental scripts are available at \url{https://github.com/cxjdavin/new-metrics-and-search-algorithms-for-weighted-causal-DAGs}.

We base our evaluation on the experimental framework of \cite{choo2022verification}\footnote{Available at \url{https://github.com/cxjdavin/subset-verification-and-search-algorithms-for-causal-DAGs}}, which in turn is based on \cite{squires2020active}\footnote{Available at \url{https://github.com/csquires/dct-policy}}.
In the following, we replicate some of the experimental setup details from Appendix H of \cite{choo2022verification}.

\subsection{Synthetic graph classes}

The synthetic graphs are random connected moral DAGs.

\textbf{1. Erd\H{o}s-R\'{e}nyi styled graphs}\quad
These graphs are parameterized by 2 parameters: $n$ and density $\rho$.
Generate a random ordering $\sigma$ over $n$ vertices.
Then, set the in-degree of the $n^{th}$ vertex (i.e.\ last vertex in the ordering) in the order to be $X_n = \max\{1, \texttt{Binomial}(n-1, \rho)\}$, and sample $X_n$ parents uniformly form the nodes earlier in the ordering.
Finally, chordalize the graph by running the elimination algorithm of \cite{koller2009probabilistic} with elimination ordering equal to the reverse of $\sigma$.

\textbf{2. Tree-like graphs}\quad
These graphs are parameterized by 4 parameters: $n$, degree $d$, $e_{\min}$, and $e_{\max}$. First, generate a complete directed $d$-ary tree on $n$ nodes.
Then, add $\texttt{Uniform}(e_{\min}, e_{\max})$ edges to the tree.
Finally, compute a topological order of the graph by DFS and triangulate the graph using that order.

\subsection{Weights and generalized cost parameters}

We ran experiments for $\alpha \in \{0,1\}$ and $\beta = 1$ on two different types of weight classes for a graph on $n$ vertices:

\textbf{Type 1}\quad
The weight of each vertex is independently sampled from an exponential distribution $\exp(n^2)$ with parameter $n^2$. This is to simulate the setting where there is a spread in the costs of the vertices.

\textbf{Type 2}\quad
A randomly chosen $p=0.1$ fraction of vertices are assigned weight $n^2$ while the others are assigned weight $1$. This is to simulate the setting where there are a few randomly chosen high cost vertices.

So, we have 4 sets of experiments in total, where each set follows the 5 experiments performed in \cite{choo2022verification}.

\textbf{Experiment 1}\quad
Graph class 1 with $n \in \{10, 15, 20, 25\}$ and density $\rho = 0.1$.

\textbf{Experiment 2}\quad
Graph class 1 with $n \in \{8, 10, 12, 14\}$ and density $\rho = 0.1$.

\textbf{Experiment 3}\quad
Graph class 2 with $n \in \{100, 200, 300, 400, 500\}$ and $(\text{degree}, e_{\min}, e_{\max}) = (4, 2, 5)$.

\textbf{Experiment 4}\quad
Graph class 1 with $n \in \{10, 15, 20, 25\}$ and density $\rho = 0.1$.

\textbf{Experiment 5}\quad
Graph class 2 with $n \in \{100, 200, 300, 400, 500\}$ and $(\text{degree}, e_{\min}, e_{\max}) = (40, 20, 50)$.

\subsection{Algorithms benchmarked}

The following algorithms perform \emph{atomic interventions}.
Our algorithm \texttt{weighted\_separator} perform atomic interventions when given $k=1$ and \emph{bounded size interventions} when given $k > 1$.

\texttt{random}:\quad
A baseline algorithm that repeatedly picks a random non-dominated node (a node that is incident to some unoriented edge) from the interventional essential graph

\texttt{dct}:\quad
\texttt{DCT Policy} of \cite{squires2020active}

\texttt{coloring}:\quad
\texttt{Coloring} of \cite{shanmugam2015learning}

\texttt{opt\_single}:\quad
\texttt{OptSingle} of \cite{hauser2014two}

\texttt{greedy\_minmax}:\quad
\texttt{MinmaxMEC} of \cite{he2008active}

\texttt{greedy\_entropy}:\quad
\texttt{MinmaxEntropy} of \cite{he2008active}

\texttt{separator}:\quad
Algorithm of \cite{choo2022verification}; Takes in a parameter $k$ for bounded-size interventions.

\texttt{weighted\_separator}:\quad
Our \cref{alg:weighted-search-generalized}; Takes in a parameter $k$ for bounded-size interventions.

\subsection{Experimental results}

In all experiments, \texttt{ALG-GENERALIZED} has a similar run time\footnote{\texttt{ALG-GENERALIZED} is faster than all benchmarked algorithms except \cite{choo2022verification}. This is expected as both use an approach based on 1/2-clique separators but \texttt{ALG-GENERALIZED} has additional computational overhead to handle dangling components.}.
When $\alpha = 0$ and $\beta = 1$, the generalized cost function is simply the number of interventions used, and \texttt{ALG-GENERALIZED} incurs a similar cost to the existing state-of-the-art algorithms.
Meanwhile, when $\alpha = 1$ and $\beta = 1$, the generalized cost function is affected by the vertex weights, and \texttt{ALG-GENERALIZED} incurs noticeably less generalized cost.

We also tested the bounded size implementation for $k \in \{1,3,5\}$ and observe that the lines ``flip'', for both weight types.
When $(\alpha, \beta) = (0,1)$, $k = 1$ is worst and $k = 3$ is best. 
When $(\alpha, \beta) = (1,1)$, $k = 3$ is worst and $k = 1$ is best.
This matches what we expect from our theoretical analyses.

\begin{figure}[htbp]
\centering
\begin{subfigure}[t]{0.23\linewidth}
    \centering
    \includegraphics[width=\linewidth]{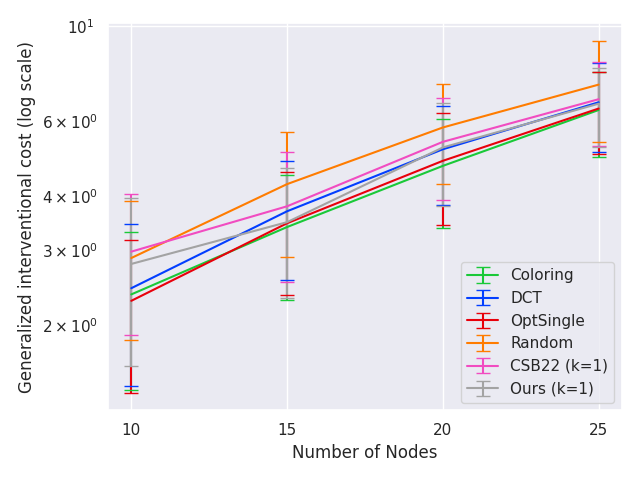}
    \caption{Generalized cost (log scale)}
\end{subfigure}
\begin{subfigure}[t]{0.23\linewidth}
    \centering
    \includegraphics[width=\linewidth]{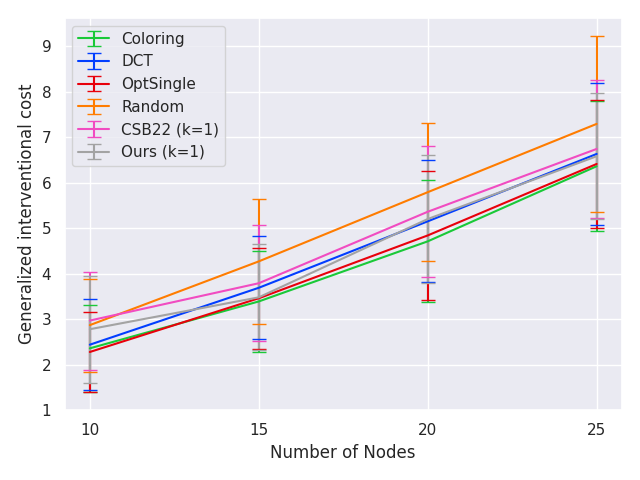}
    \caption{Generalized cost}
\end{subfigure}
\begin{subfigure}[t]{0.23\linewidth}
    \centering
    \includegraphics[width=\linewidth]{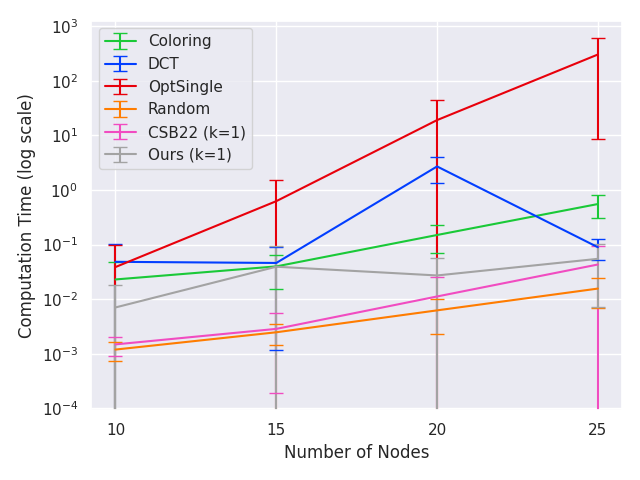}
    \caption{Time taken, in secs (log scale)}
\end{subfigure}
\begin{subfigure}[t]{0.23\linewidth}
    \centering
    \includegraphics[width=\linewidth]{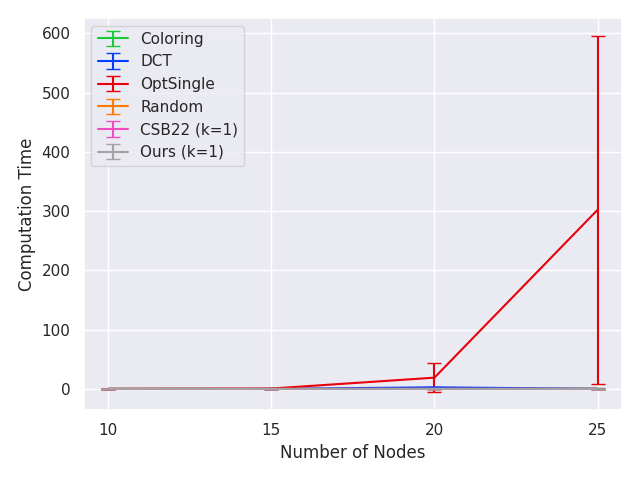}
    \caption{Time taken, in secs}
\end{subfigure}
\caption{Experiment 1, Type 1, $\alpha = 0$, $\beta = 1$}
\label{fig:exp1_type1_alpha0_beta1}
\end{figure}

\begin{figure}[htbp]
\centering
\begin{subfigure}[t]{0.23\linewidth}
    \centering
    \includegraphics[width=\linewidth]{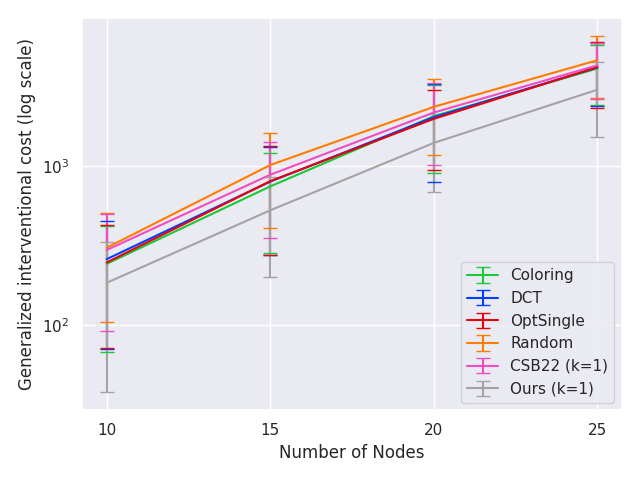}
    \caption{Generalized cost (log scale)}
\end{subfigure}
\begin{subfigure}[t]{0.23\linewidth}
    \centering
    \includegraphics[width=\linewidth]{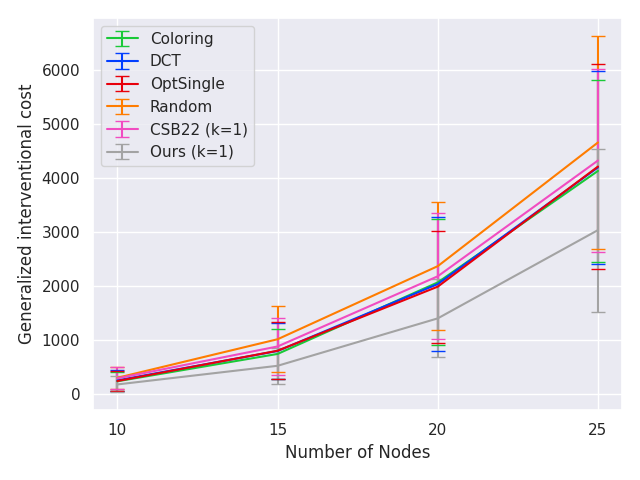}
    \caption{Generalized cost}
\end{subfigure}
\begin{subfigure}[t]{0.23\linewidth}
    \centering
    \includegraphics[width=\linewidth]{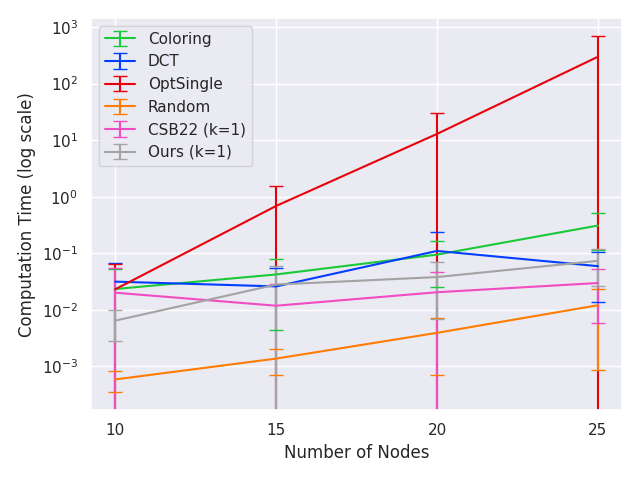}
    \caption{Time taken, in secs (log scale)}
\end{subfigure}
\begin{subfigure}[t]{0.23\linewidth}
    \centering
    \includegraphics[width=\linewidth]{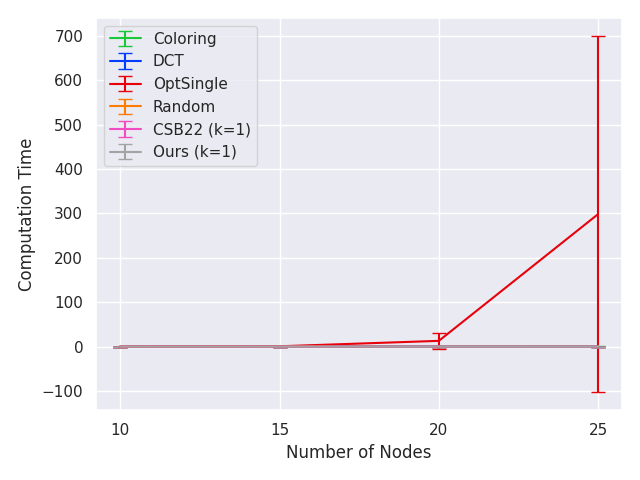}
    \caption{Time taken, in secs}
\end{subfigure}
\caption{Experiment 1, Type 1, $\alpha = 1$, $\beta = 1$}
\label{fig:exp1_type1_alpha1_beta1}
\end{figure}

\begin{figure}[htbp]
\centering
\begin{subfigure}[t]{0.23\linewidth}
    \centering
    \includegraphics[width=\linewidth]{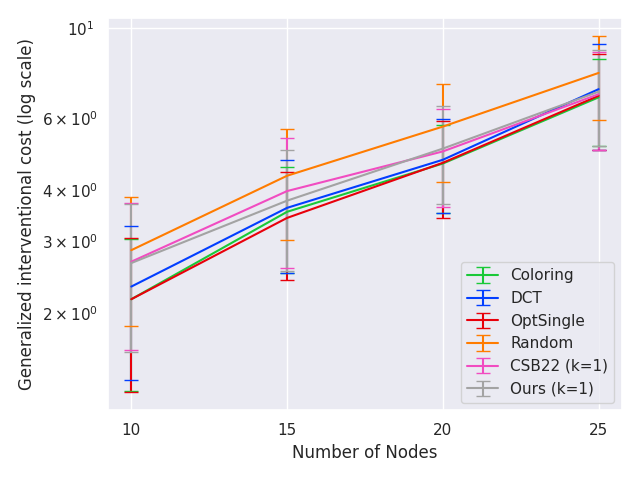}
    \caption{Generalized cost (log scale)}
\end{subfigure}
\begin{subfigure}[t]{0.23\linewidth}
    \centering
    \includegraphics[width=\linewidth]{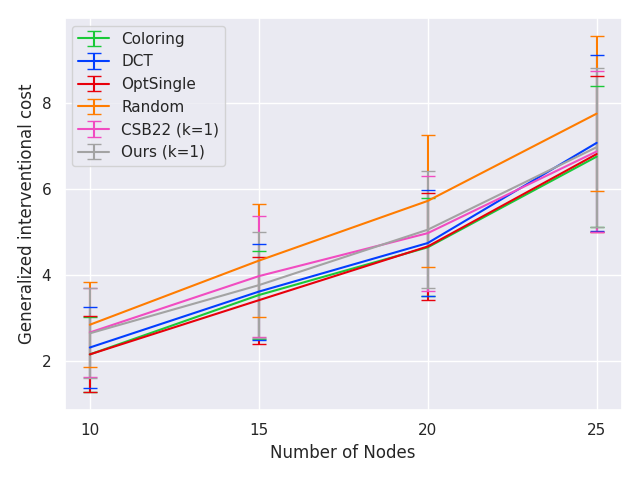}
    \caption{Generalized cost}
\end{subfigure}
\begin{subfigure}[t]{0.23\linewidth}
    \centering
    \includegraphics[width=\linewidth]{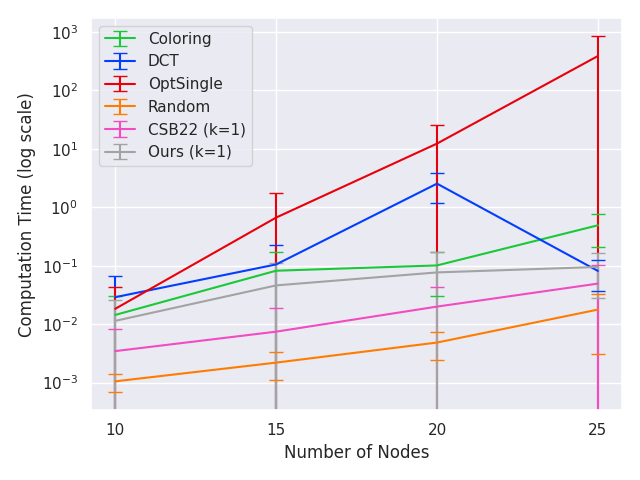}
    \caption{Time taken, in secs (log scale)}
\end{subfigure}
\begin{subfigure}[t]{0.23\linewidth}
    \centering
    \includegraphics[width=\linewidth]{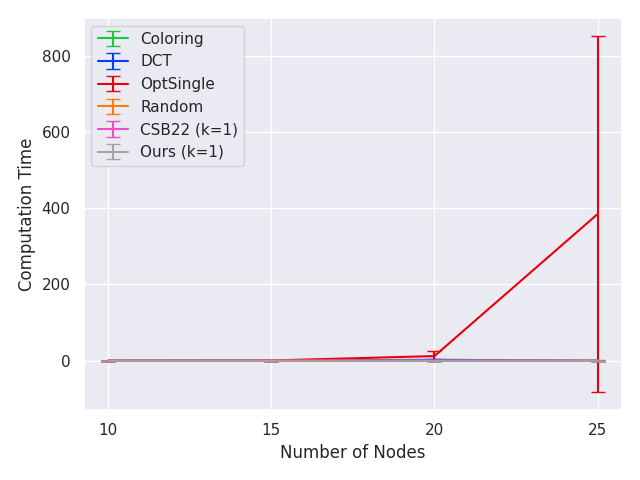}
    \caption{Time taken, in secs}
\end{subfigure}
\caption{Experiment 1, Type 2, $\alpha = 0$, $\beta = 1$}
\label{fig:exp1_type2_alpha0_beta1}
\end{figure}

\begin{figure}[htbp]
\centering
\begin{subfigure}[t]{0.23\linewidth}
    \centering
    \includegraphics[width=\linewidth]{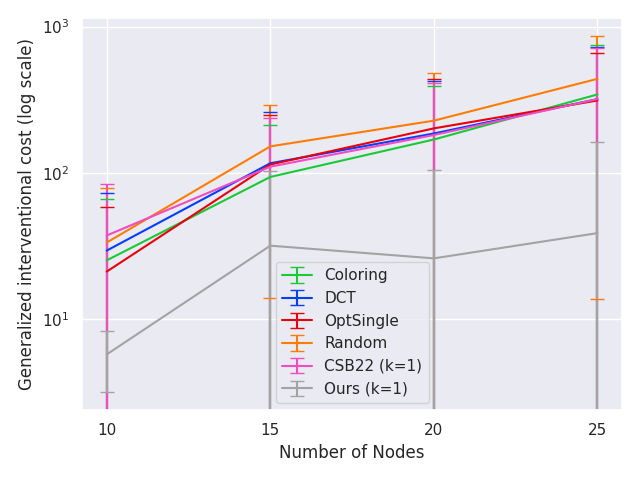}
    \caption{Generalized cost (log scale)}
\end{subfigure}
\begin{subfigure}[t]{0.23\linewidth}
    \centering
    \includegraphics[width=\linewidth]{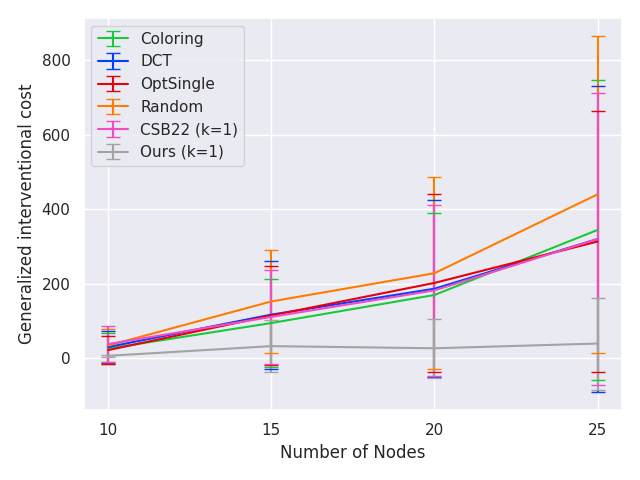}
    \caption{Generalized cost}
\end{subfigure}
\begin{subfigure}[t]{0.23\linewidth}
    \centering
    \includegraphics[width=\linewidth]{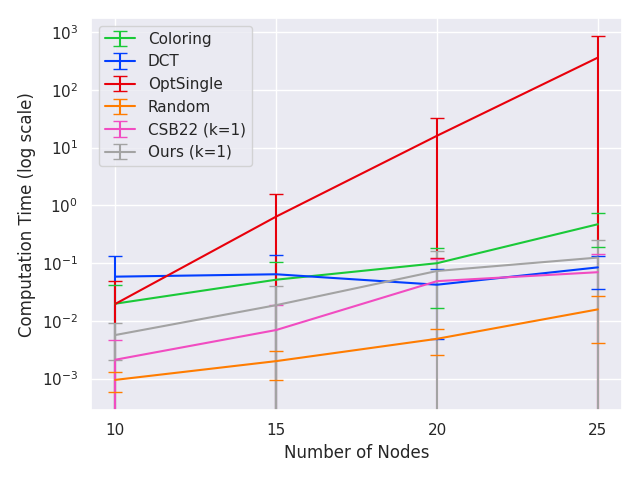}
    \caption{Time taken, in secs (log scale)}
\end{subfigure}
\begin{subfigure}[t]{0.23\linewidth}
    \centering
    \includegraphics[width=\linewidth]{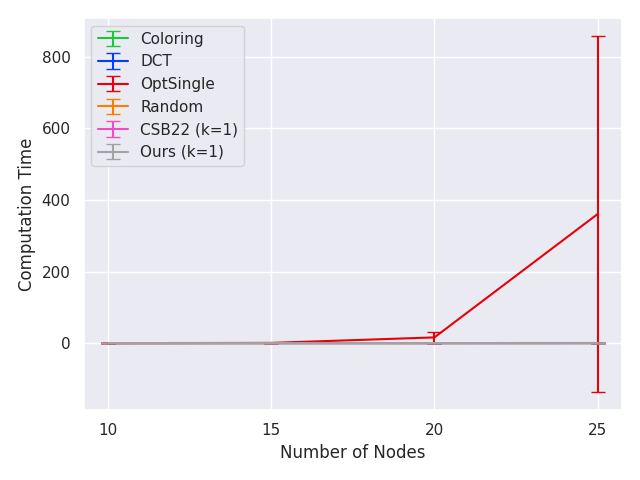}
    \caption{Time taken, in secs}
\end{subfigure}
\caption{Experiment 1, Type 2, $\alpha = 1$, $\beta = 1$}
\label{fig:exp1_type2_alpha1_beta1}
\end{figure}

\begin{figure}[htbp]
\centering
\begin{subfigure}[t]{0.23\linewidth}
    \centering
    \includegraphics[width=\linewidth]{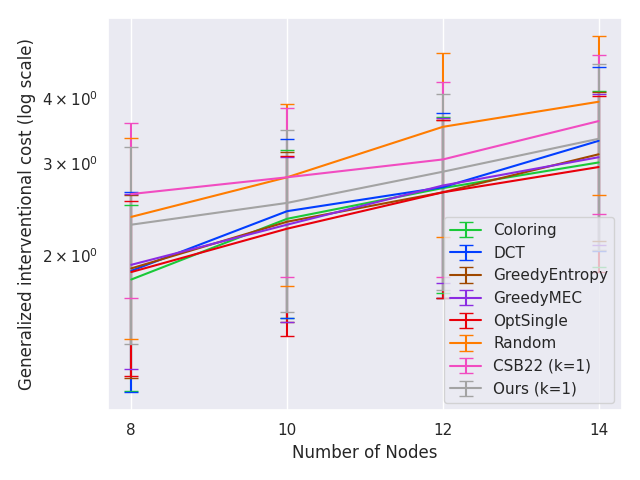}
    \caption{Generalized cost (log scale)}
\end{subfigure}
\begin{subfigure}[t]{0.23\linewidth}
    \centering
    \includegraphics[width=\linewidth]{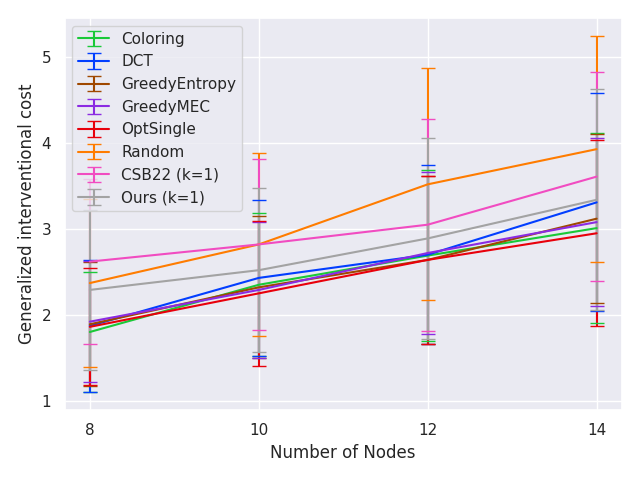}
    \caption{Generalized cost}
\end{subfigure}
\begin{subfigure}[t]{0.23\linewidth}
    \centering
    \includegraphics[width=\linewidth]{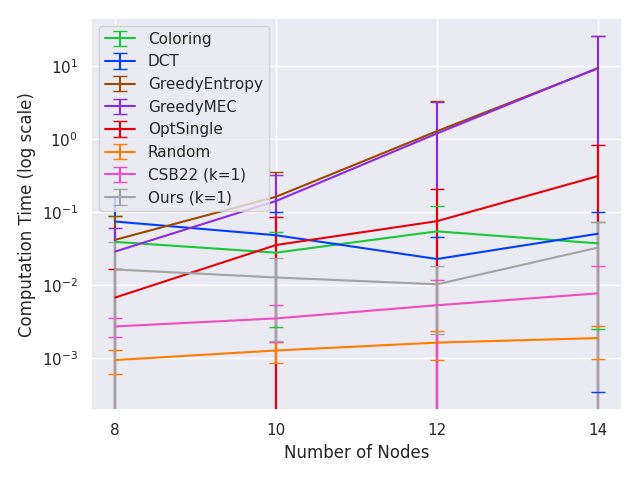}
    \caption{Time taken, in secs (log scale)}
\end{subfigure}
\begin{subfigure}[t]{0.23\linewidth}
    \centering
    \includegraphics[width=\linewidth]{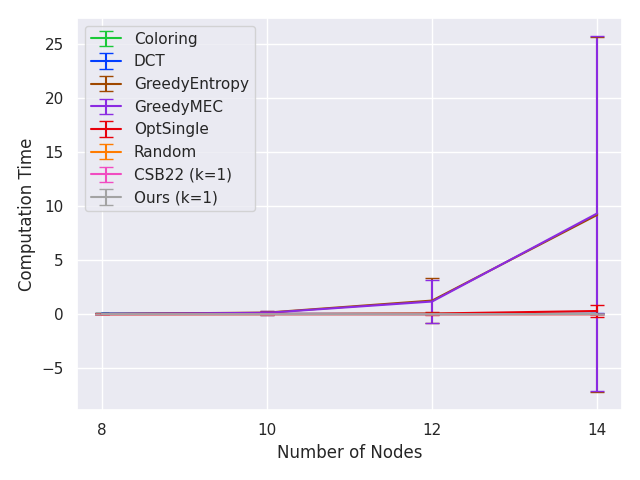}
    \caption{Time taken, in secs}
\end{subfigure}
\caption{Experiment 2, Type 1, $\alpha = 0$, $\beta = 1$}
\label{fig:exp2_type1_alpha0_beta1}
\end{figure}

\begin{figure}[htbp]
\centering
\begin{subfigure}[t]{0.23\linewidth}
    \centering
    \includegraphics[width=\linewidth]{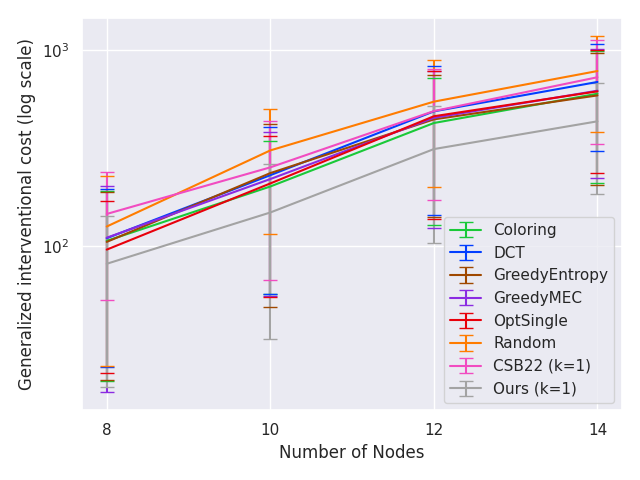}
    \caption{Generalized cost (log scale)}
\end{subfigure}
\begin{subfigure}[t]{0.23\linewidth}
    \centering
    \includegraphics[width=\linewidth]{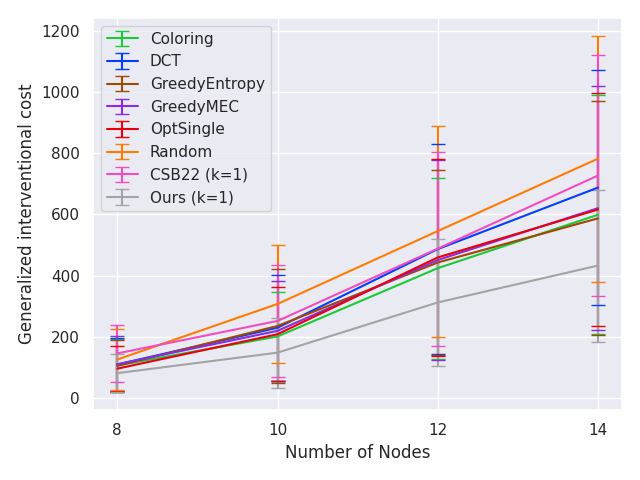}
    \caption{Generalized cost}
\end{subfigure}
\begin{subfigure}[t]{0.23\linewidth}
    \centering
    \includegraphics[width=\linewidth]{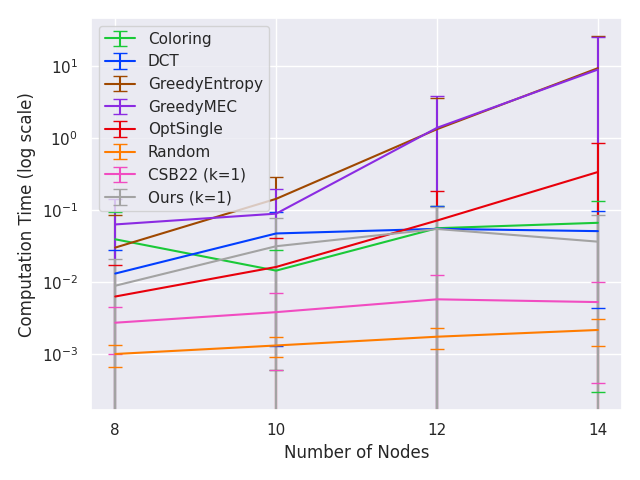}
    \caption{Time taken, in secs (log scale)}
\end{subfigure}
\begin{subfigure}[t]{0.23\linewidth}
    \centering
    \includegraphics[width=\linewidth]{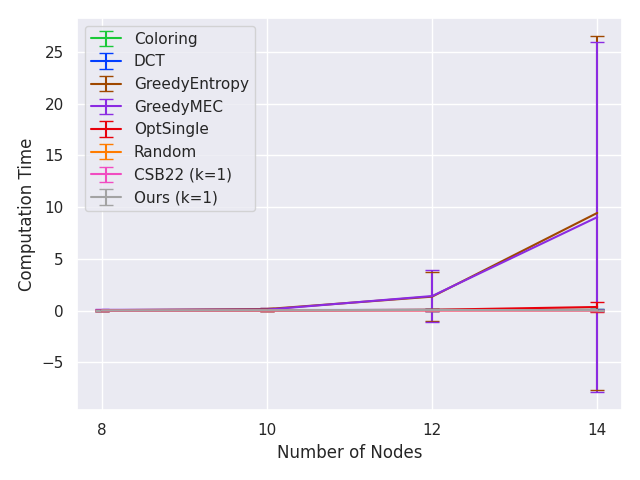}
    \caption{Time taken, in secs}
\end{subfigure}
\caption{Experiment 2, Type 1, $\alpha = 1$, $\beta = 1$}
\label{fig:exp2_type1_alpha1_beta1}
\end{figure}

\begin{figure}[htbp]
\centering
\begin{subfigure}[t]{0.23\linewidth}
    \centering
    \includegraphics[width=\linewidth]{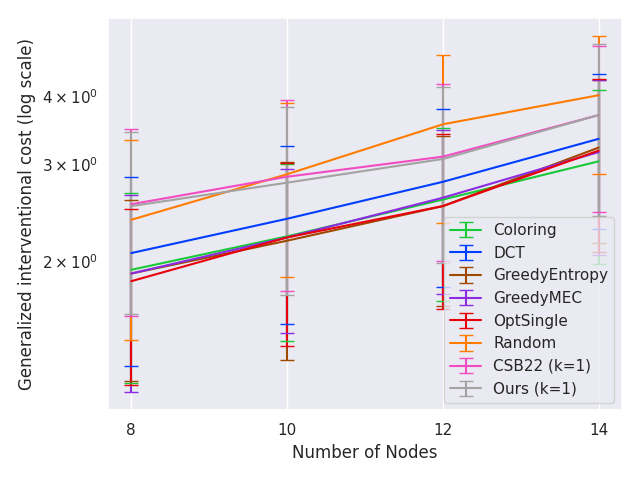}
    \caption{Generalized cost (log scale)}
\end{subfigure}
\begin{subfigure}[t]{0.23\linewidth}
    \centering
    \includegraphics[width=\linewidth]{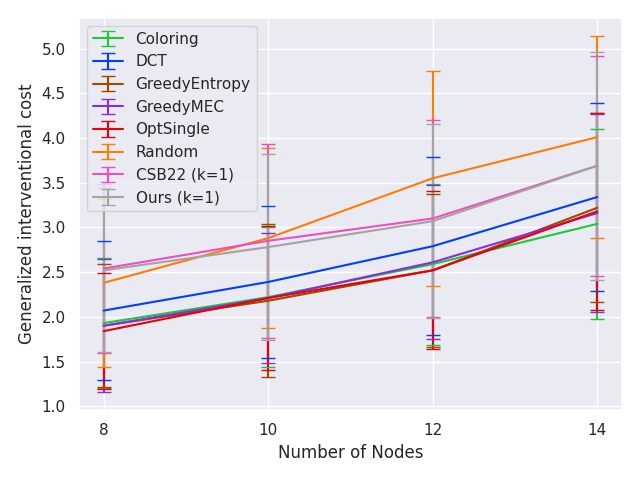}
    \caption{Generalized cost}
\end{subfigure}
\begin{subfigure}[t]{0.23\linewidth}
    \centering
    \includegraphics[width=\linewidth]{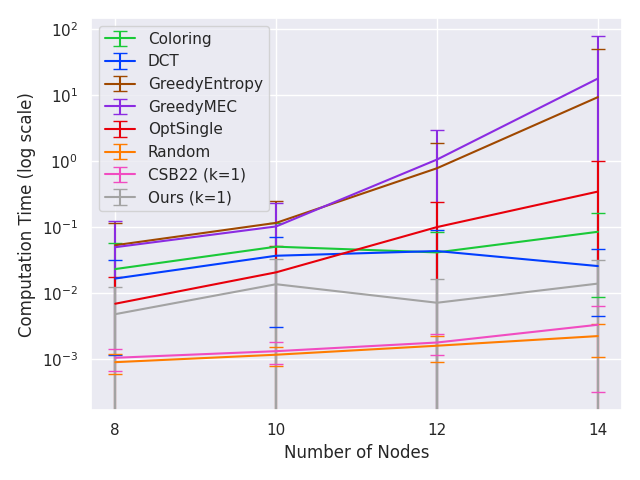}
    \caption{Time taken, in secs (log scale)}
\end{subfigure}
\begin{subfigure}[t]{0.23\linewidth}
    \centering
    \includegraphics[width=\linewidth]{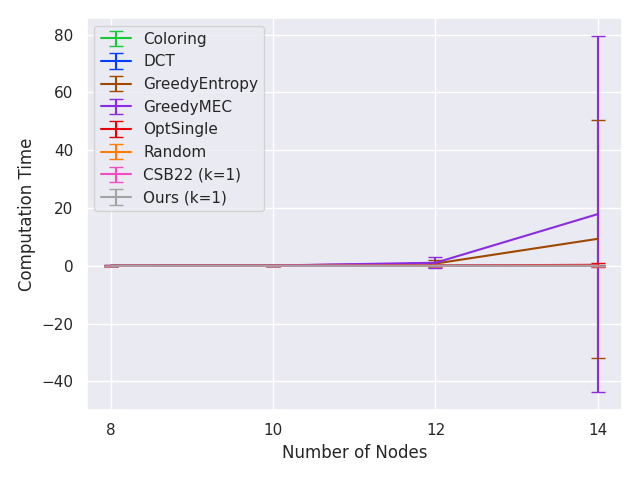}
    \caption{Time taken, in secs}
\end{subfigure}
\caption{Experiment 2, Type 2, $\alpha = 0$, $\beta = 1$}
\label{fig:exp2_type2_alpha0_beta1}
\end{figure}

\begin{figure}[htbp]
\centering
\begin{subfigure}[t]{0.23\linewidth}
    \centering
    \includegraphics[width=\linewidth]{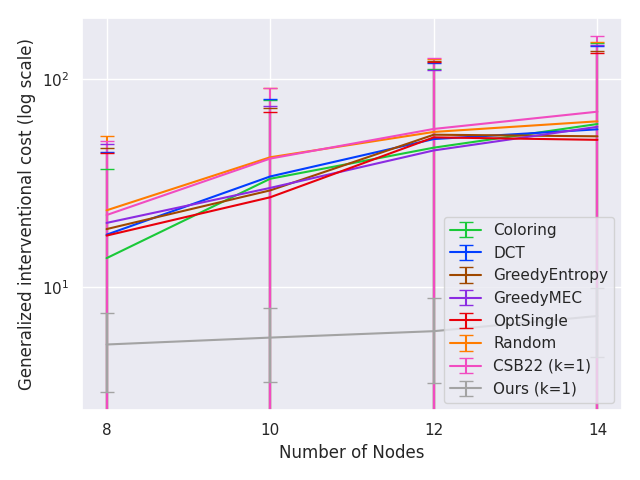}
    \caption{Generalized cost (log scale)}
\end{subfigure}
\begin{subfigure}[t]{0.23\linewidth}
    \centering
    \includegraphics[width=\linewidth]{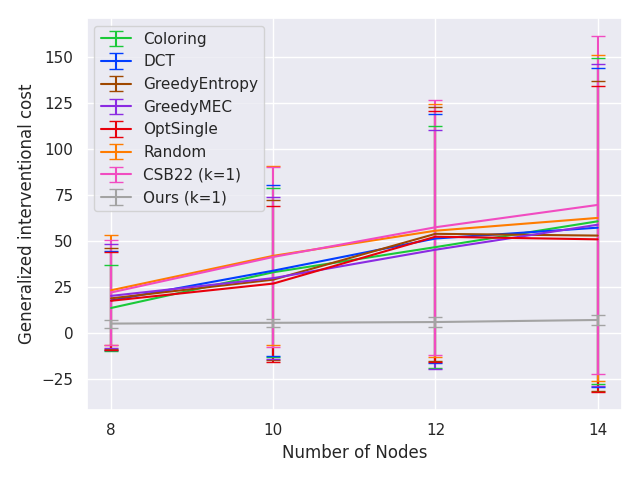}
    \caption{Generalized cost}
\end{subfigure}
\begin{subfigure}[t]{0.23\linewidth}
    \centering
    \includegraphics[width=\linewidth]{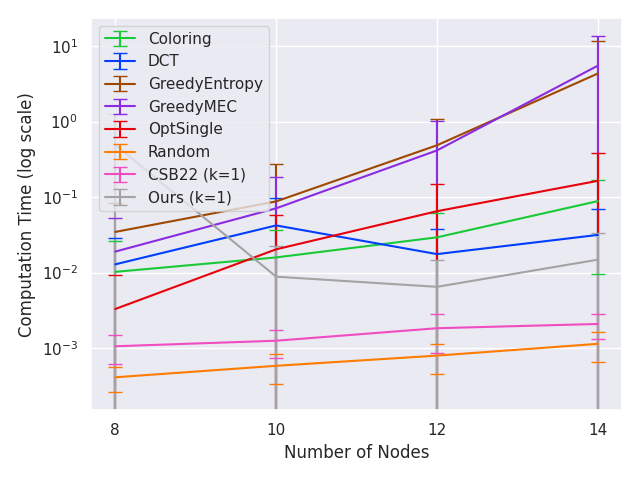}
    \caption{Time taken, in secs (log scale)}
\end{subfigure}
\begin{subfigure}[t]{0.23\linewidth}
    \centering
    \includegraphics[width=\linewidth]{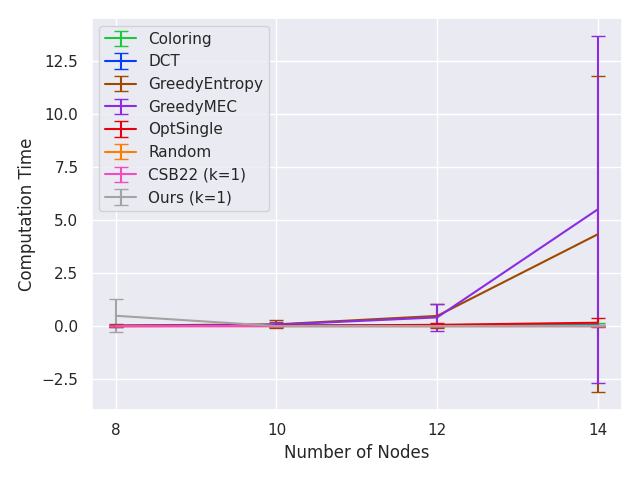}
    \caption{Time taken, in secs}
\end{subfigure}
\caption{Experiment 2, Type 2, $\alpha = 1$, $\beta = 1$}
\label{fig:exp2_type2_alpha1_beta1}
\end{figure}

\begin{figure}[htbp]
\centering
\begin{subfigure}[t]{0.23\linewidth}
    \centering
    \includegraphics[width=\linewidth]{plots/exp3_type1_alpha0_beta1_generalized_cost_log.png}
    \caption{Generalized cost (log scale)}
\end{subfigure}
\begin{subfigure}[t]{0.23\linewidth}
    \centering
    \includegraphics[width=\linewidth]{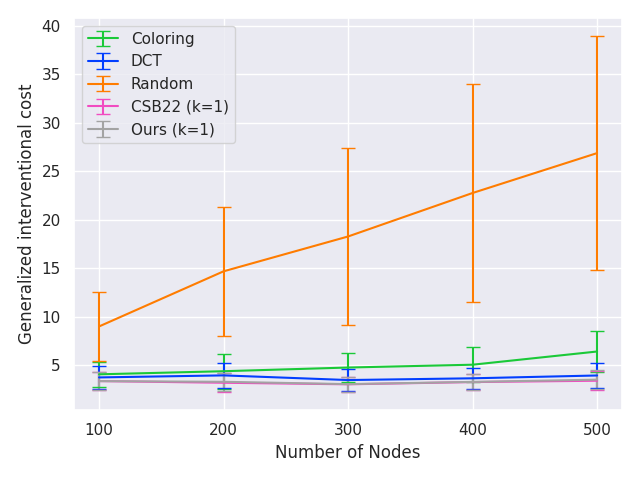}
    \caption{Generalized cost}
\end{subfigure}
\begin{subfigure}[t]{0.23\linewidth}
    \centering
    \includegraphics[width=\linewidth]{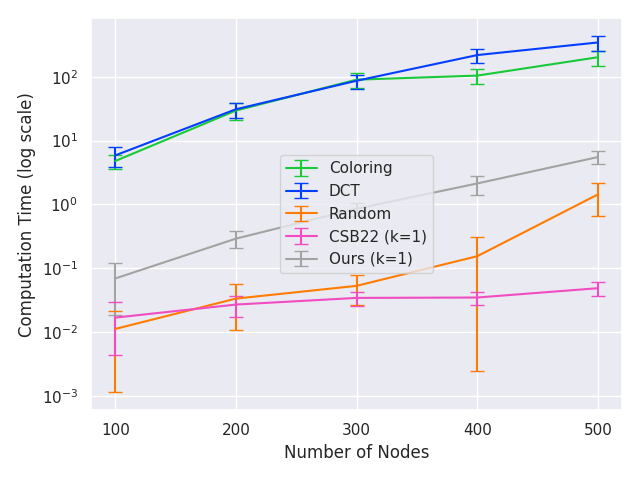}
    \caption{Time taken, in secs (log scale)}
\end{subfigure}
\begin{subfigure}[t]{0.23\linewidth}
    \centering
    \includegraphics[width=\linewidth]{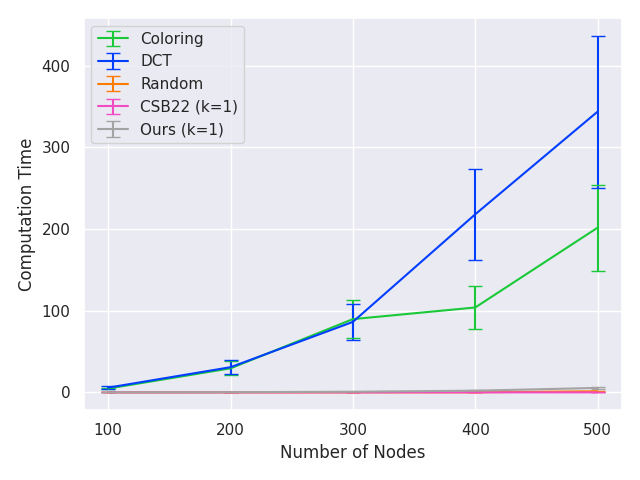}
    \caption{Time taken, in secs}
\end{subfigure}
\caption{Experiment 3, Type 1, $\alpha = 0$, $\beta = 1$}
\label{fig:exp3_type1_alpha0_beta1}
\end{figure}

\begin{figure}[htbp]
\centering
\begin{subfigure}[t]{0.23\linewidth}
    \centering
    \includegraphics[width=\linewidth]{plots/exp3_type1_alpha1_beta1_generalized_cost_log.png}
    \caption{Generalized cost (log scale)}
\end{subfigure}
\begin{subfigure}[t]{0.23\linewidth}
    \centering
    \includegraphics[width=\linewidth]{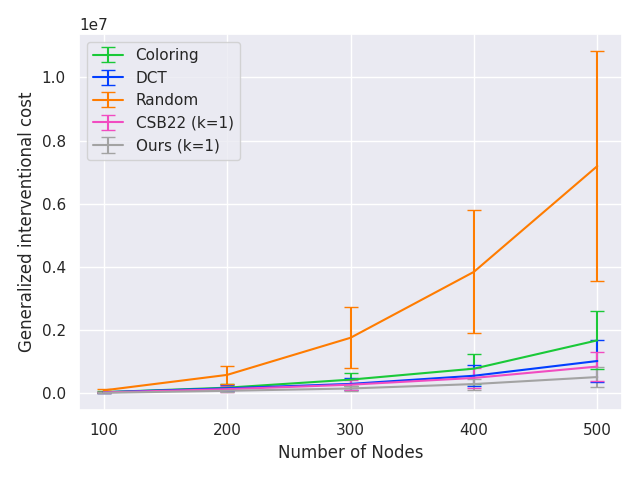}
    \caption{Generalized cost}
\end{subfigure}
\begin{subfigure}[t]{0.23\linewidth}
    \centering
    \includegraphics[width=\linewidth]{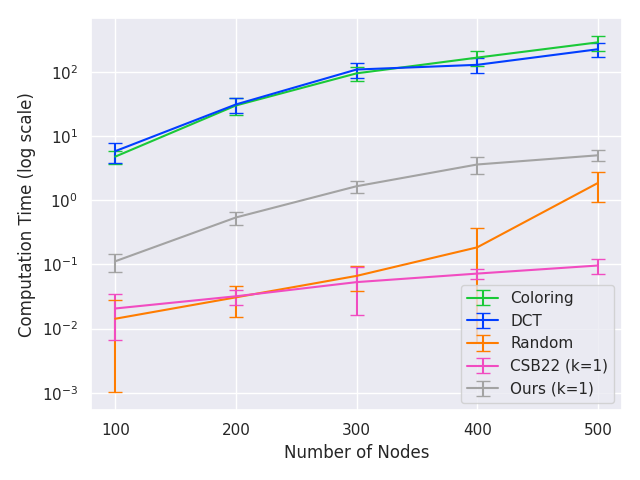}
    \caption{Time taken, in secs (log scale)}
\end{subfigure}
\begin{subfigure}[t]{0.23\linewidth}
    \centering
    \includegraphics[width=\linewidth]{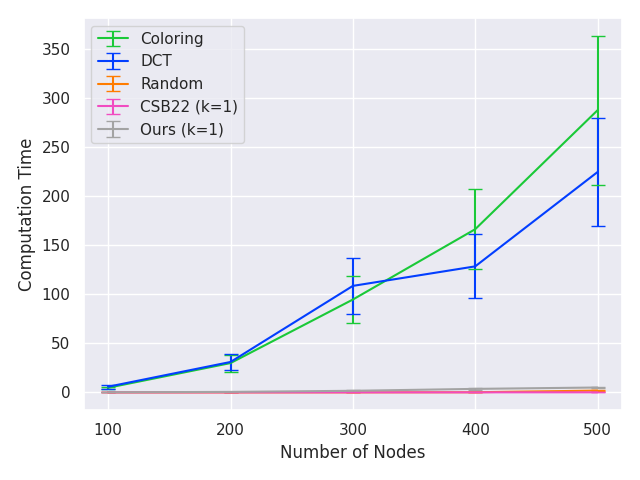}
    \caption{Time taken, in secs}
\end{subfigure}
\caption{Experiment 3, Type 1, $\alpha = 1$, $\beta = 1$}
\label{fig:exp3_type1_alpha1_beta1}
\end{figure}

\begin{figure}[htbp]
\centering
\begin{subfigure}[t]{0.23\linewidth}
    \centering
    \includegraphics[width=\linewidth]{plots/exp3_type2_alpha0_beta1_generalized_cost_log.png}
    \caption{Generalized cost (log scale)}
\end{subfigure}
\begin{subfigure}[t]{0.23\linewidth}
    \centering
    \includegraphics[width=\linewidth]{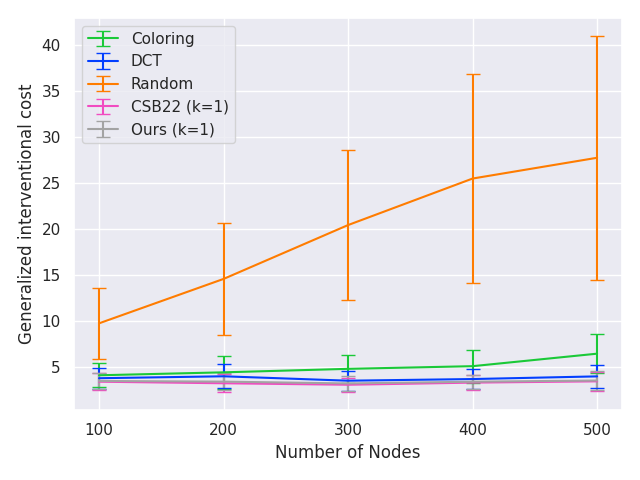}
    \caption{Generalized cost}
\end{subfigure}
\begin{subfigure}[t]{0.23\linewidth}
    \centering
    \includegraphics[width=\linewidth]{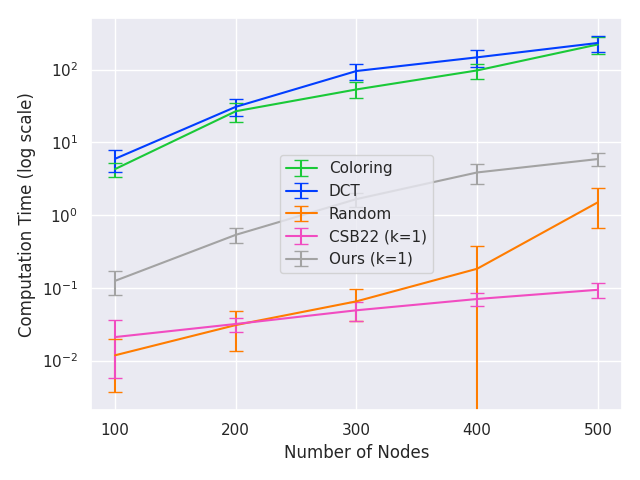}
    \caption{Time taken, in secs (log scale)}
\end{subfigure}
\begin{subfigure}[t]{0.23\linewidth}
    \centering
    \includegraphics[width=\linewidth]{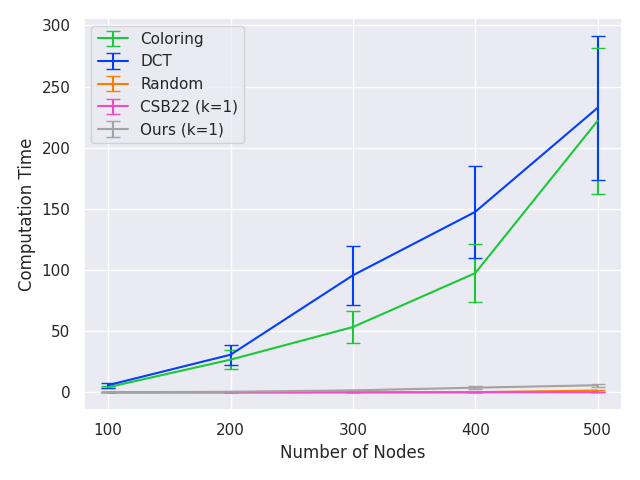}
    \caption{Time taken, in secs}
\end{subfigure}
\caption{Experiment 3, Type 2, $\alpha = 0$, $\beta = 1$}
\label{fig:exp3_type2_alpha0_beta1}
\end{figure}

\begin{figure}[htbp]
\centering
\begin{subfigure}[t]{0.23\linewidth}
    \centering
    \includegraphics[width=\linewidth]{plots/exp3_type2_alpha1_beta1_generalized_cost_log.png}
    \caption{Generalized cost (log scale)}
\end{subfigure}
\begin{subfigure}[t]{0.23\linewidth}
    \centering
    \includegraphics[width=\linewidth]{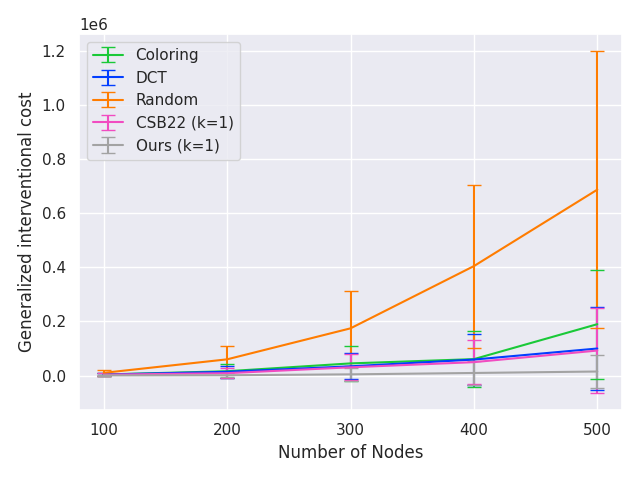}
    \caption{Generalized cost}
\end{subfigure}
\begin{subfigure}[t]{0.23\linewidth}
    \centering
    \includegraphics[width=\linewidth]{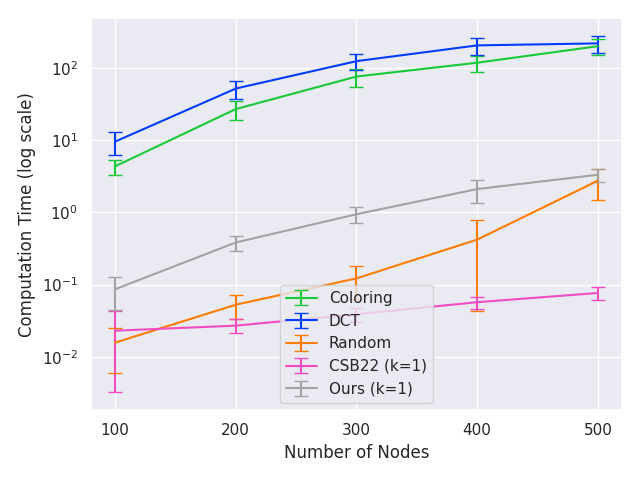}
    \caption{Time taken, in secs (log scale)}
\end{subfigure}
\begin{subfigure}[t]{0.23\linewidth}
    \centering
    \includegraphics[width=\linewidth]{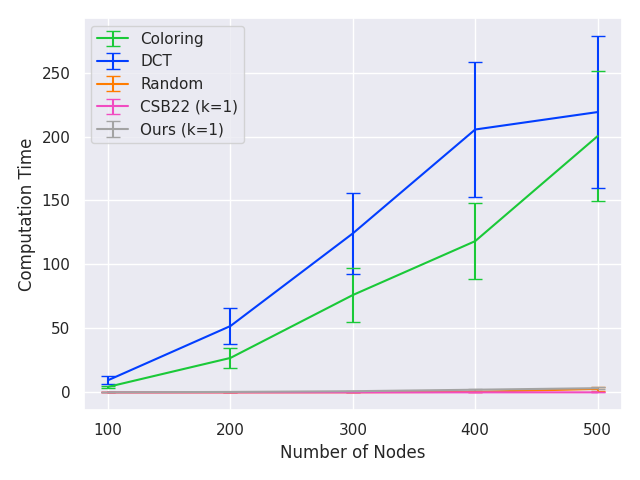}
    \caption{Time taken, in secs}
\end{subfigure}
\caption{Experiment 3, Type 2, $\alpha = 1$, $\beta = 1$}
\label{fig:exp3_type2_alpha1_beta1}
\end{figure}

\begin{figure}[htbp]
\centering
\begin{subfigure}[t]{0.23\linewidth}
    \centering
    \includegraphics[width=\linewidth]{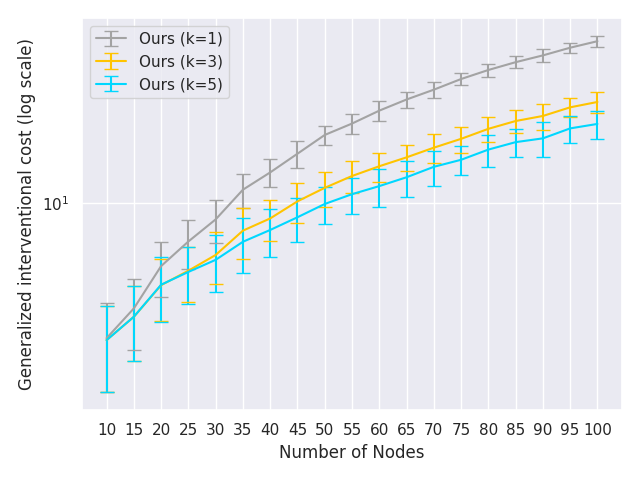}
    \caption{Generalized cost (log scale)}
\end{subfigure}
\begin{subfigure}[t]{0.23\linewidth}
    \centering
    \includegraphics[width=\linewidth]{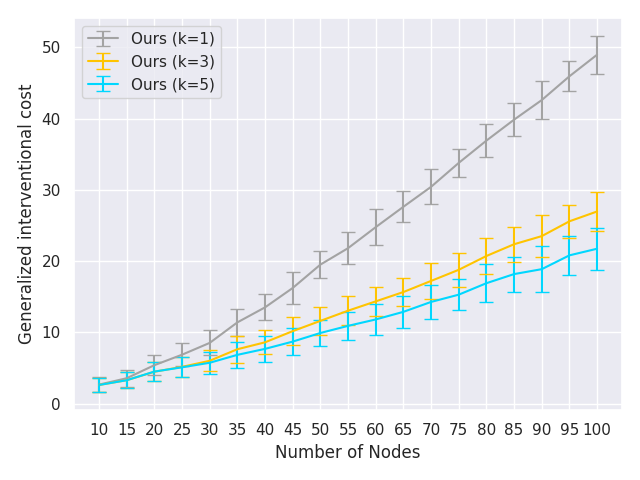}
    \caption{Generalized cost}
\end{subfigure}
\begin{subfigure}[t]{0.23\linewidth}
    \centering
    \includegraphics[width=\linewidth]{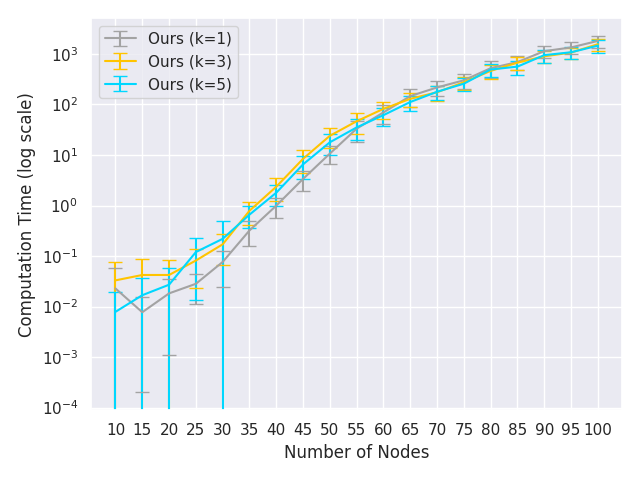}
    \caption{Time taken, in secs (log scale)}
\end{subfigure}
\begin{subfigure}[t]{0.23\linewidth}
    \centering
    \includegraphics[width=\linewidth]{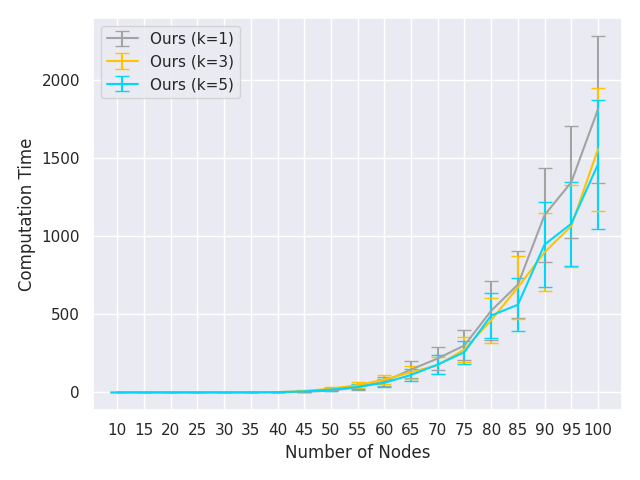}
    \caption{Time taken, in secs}
\end{subfigure}
\caption{Experiment 4, Type 1, $\alpha = 0$, $\beta = 1$}
\label{fig:exp4_type1_alpha0_beta1}
\end{figure}

\begin{figure}[htbp]
\centering
\begin{subfigure}[t]{0.23\linewidth}
    \centering
    \includegraphics[width=\linewidth]{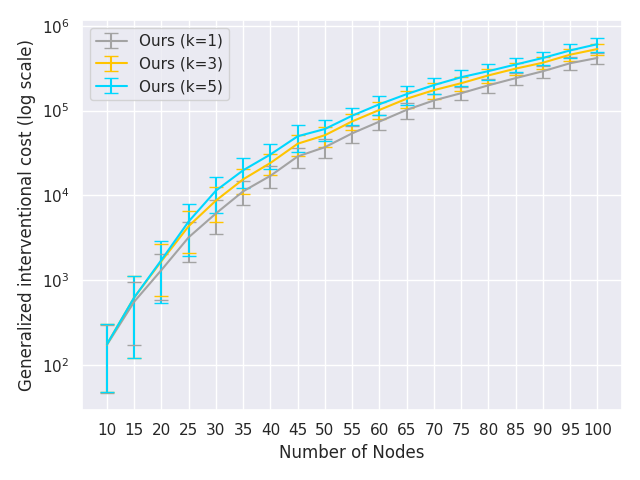}
    \caption{Generalized cost (log scale)}
\end{subfigure}
\begin{subfigure}[t]{0.23\linewidth}
    \centering
    \includegraphics[width=\linewidth]{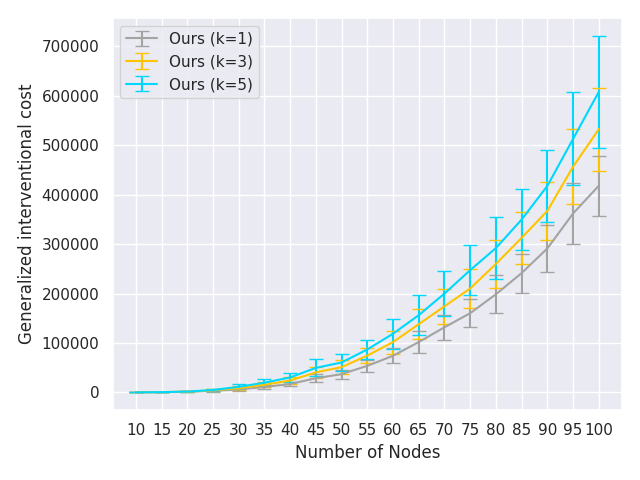}
    \caption{Generalized cost}
\end{subfigure}
\begin{subfigure}[t]{0.23\linewidth}
    \centering
    \includegraphics[width=\linewidth]{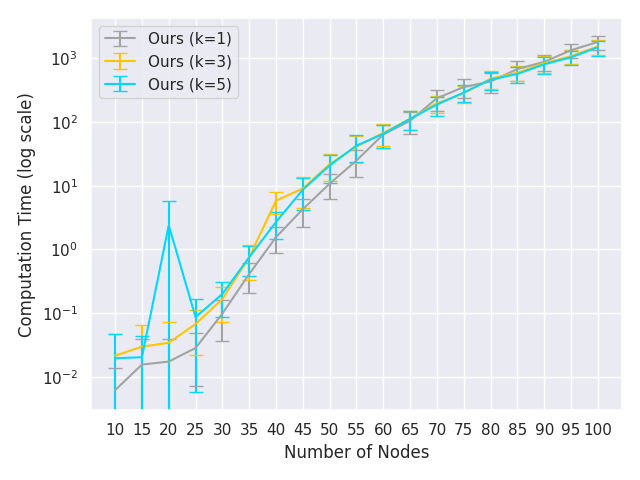}
    \caption{Time taken, in secs (log scale)}
\end{subfigure}
\begin{subfigure}[t]{0.23\linewidth}
    \centering
    \includegraphics[width=\linewidth]{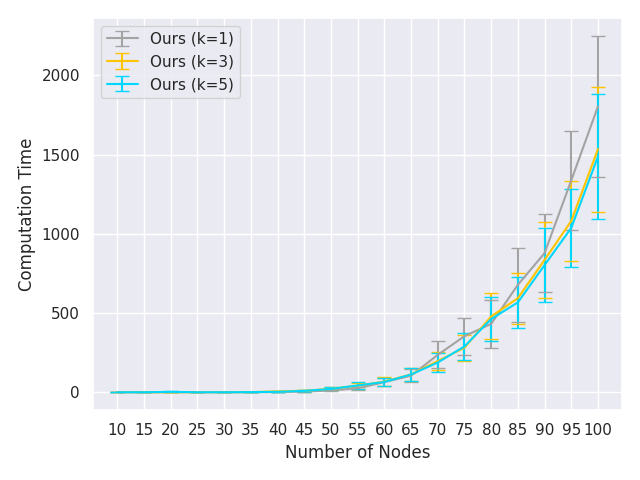}
    \caption{Time taken, in secs}
\end{subfigure}
\caption{Experiment 4, Type 1, $\alpha = 1$, $\beta = 1$}
\label{fig:exp4_type1_alpha1_beta1}
\end{figure}

\begin{figure}[htbp]
\centering
\begin{subfigure}[t]{0.23\linewidth}
    \centering
    \includegraphics[width=\linewidth]{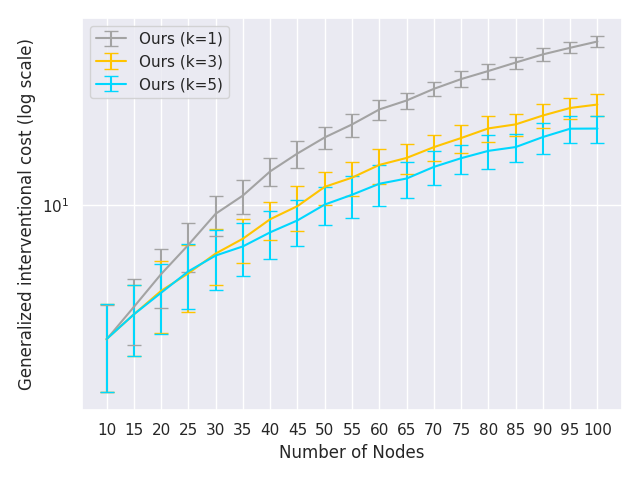}
    \caption{Generalized cost (log scale)}
\end{subfigure}
\begin{subfigure}[t]{0.23\linewidth}
    \centering
    \includegraphics[width=\linewidth]{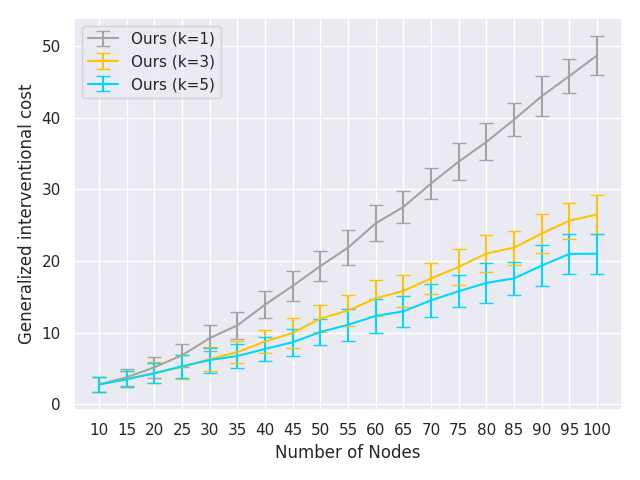}
    \caption{Generalized cost}
\end{subfigure}
\begin{subfigure}[t]{0.23\linewidth}
    \centering
    \includegraphics[width=\linewidth]{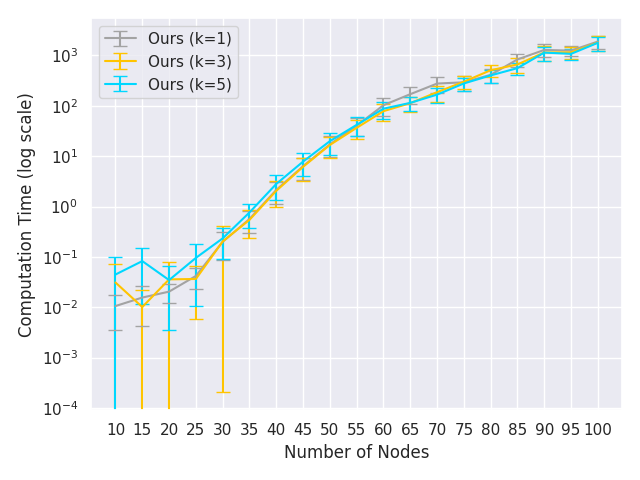}
    \caption{Time taken, in secs (log scale)}
\end{subfigure}
\begin{subfigure}[t]{0.23\linewidth}
    \centering
    \includegraphics[width=\linewidth]{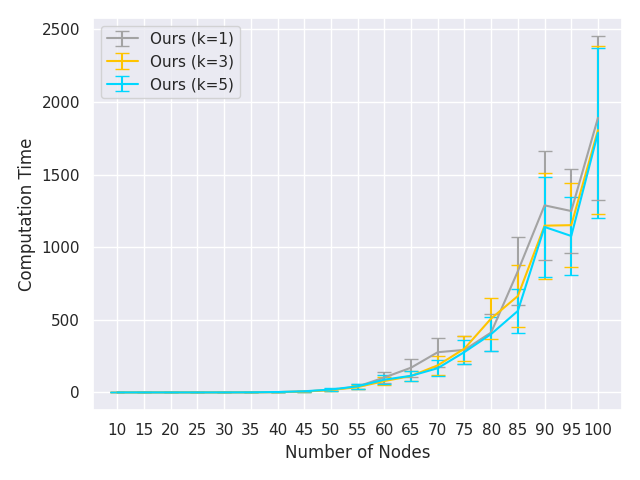}
    \caption{Time taken, in secs}
\end{subfigure}
\caption{Experiment 4, Type 2, $\alpha = 0$, $\beta = 1$}
\label{fig:exp4_type2_alpha0_beta1}
\end{figure}

\begin{figure}[htbp]
\centering
\begin{subfigure}[t]{0.23\linewidth}
    \centering
    \includegraphics[width=\linewidth]{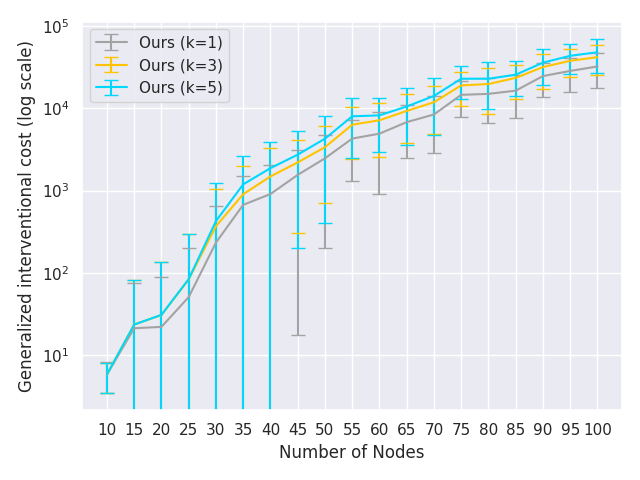}
    \caption{Generalized cost (log scale)}
\end{subfigure}
\begin{subfigure}[t]{0.23\linewidth}
    \centering
    \includegraphics[width=\linewidth]{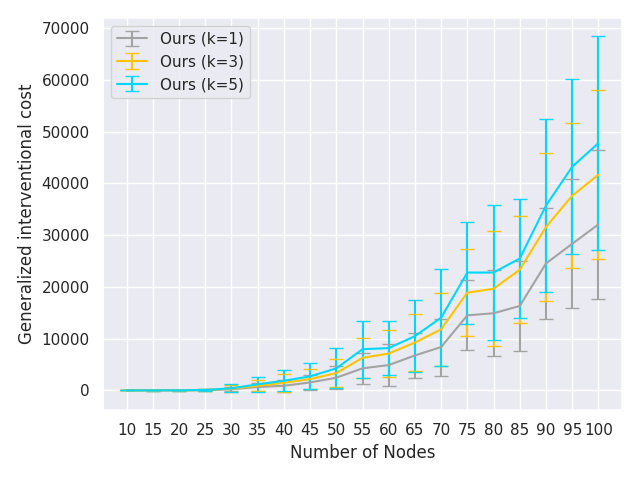}
    \caption{Generalized cost}
\end{subfigure}
\begin{subfigure}[t]{0.23\linewidth}
    \centering
    \includegraphics[width=\linewidth]{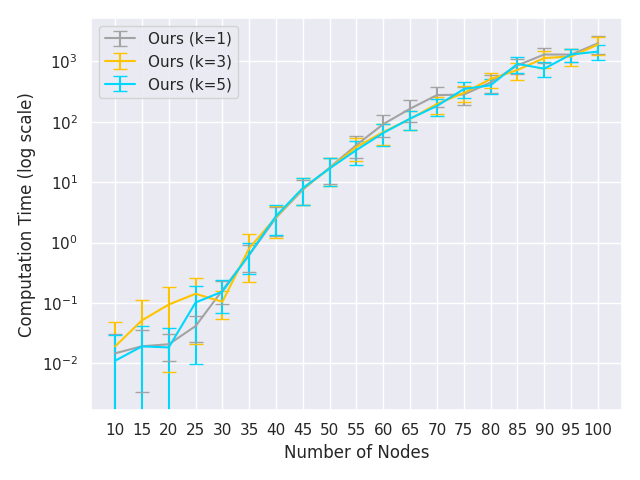}
    \caption{Time taken, in secs (log scale)}
\end{subfigure}
\begin{subfigure}[t]{0.23\linewidth}
    \centering
    \includegraphics[width=\linewidth]{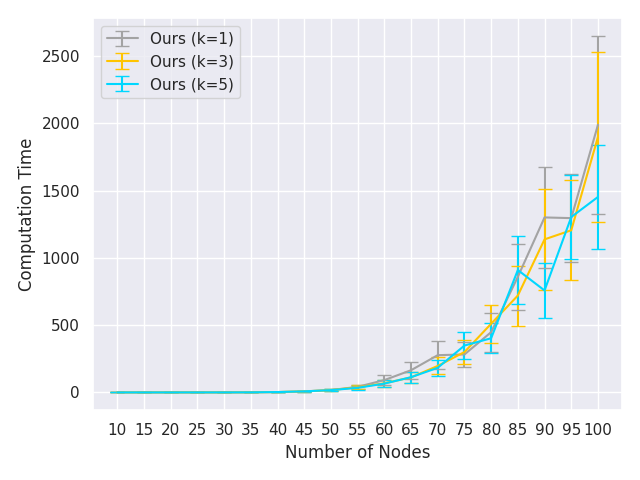}
    \caption{Time taken, in secs}
\end{subfigure}
\caption{Experiment 4, Type 2, $\alpha = 1$, $\beta = 1$}
\label{fig:exp4_type2_alpha1_beta1}
\end{figure}

\begin{figure}[htbp]
\centering
\begin{subfigure}[t]{0.23\linewidth}
    \centering
    \includegraphics[width=\linewidth]{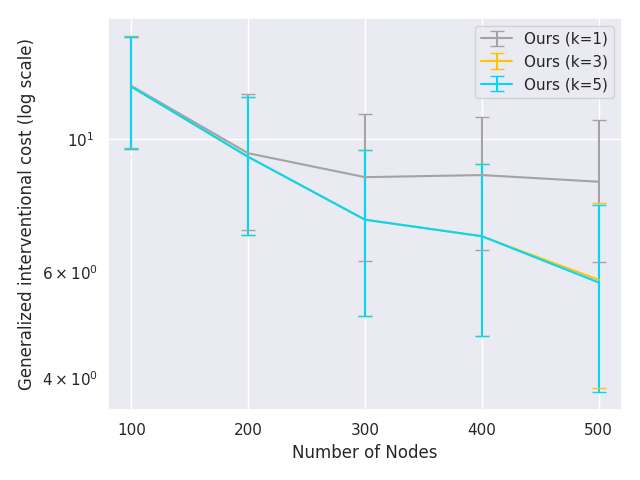}
    \caption{Generalized cost (log scale)}
\end{subfigure}
\begin{subfigure}[t]{0.23\linewidth}
    \centering
    \includegraphics[width=\linewidth]{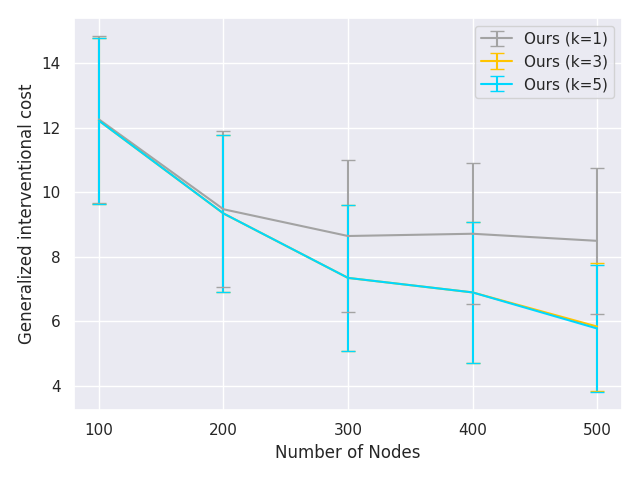}
    \caption{Generalized cost}
\end{subfigure}
\begin{subfigure}[t]{0.23\linewidth}
    \centering
    \includegraphics[width=\linewidth]{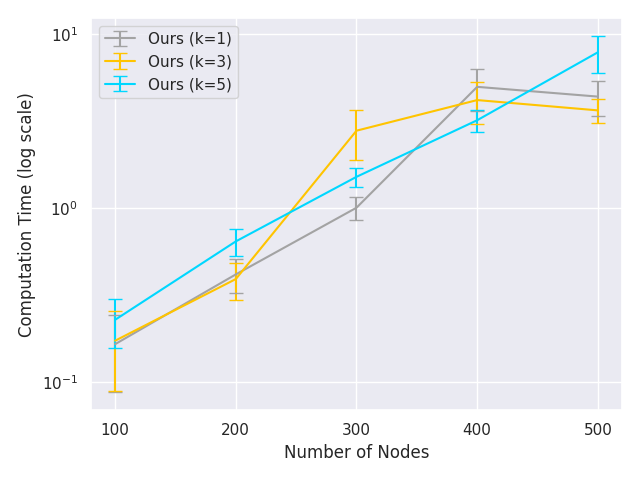}
    \caption{Time taken, in secs (log scale)}
\end{subfigure}
\begin{subfigure}[t]{0.23\linewidth}
    \centering
    \includegraphics[width=\linewidth]{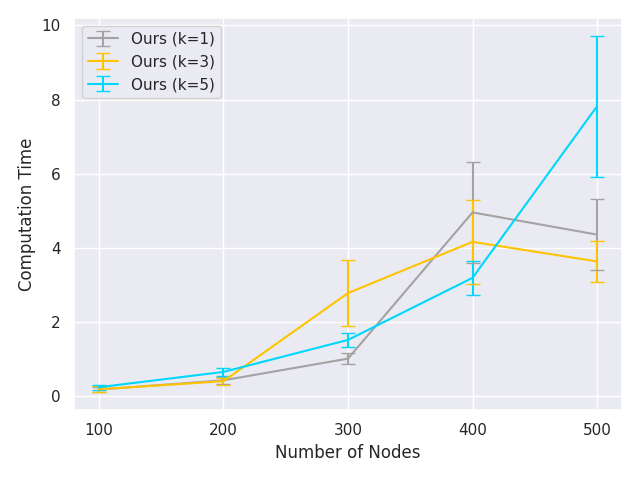}
    \caption{Time taken, in secs}
\end{subfigure}
\caption{Experiment 5, Type 1, $\alpha = 0$, $\beta = 1$}
\label{fig:exp5_type1_alpha0_beta1}
\end{figure}

\begin{figure}[htbp]
\centering
\begin{subfigure}[t]{0.23\linewidth}
    \centering
    \includegraphics[width=\linewidth]{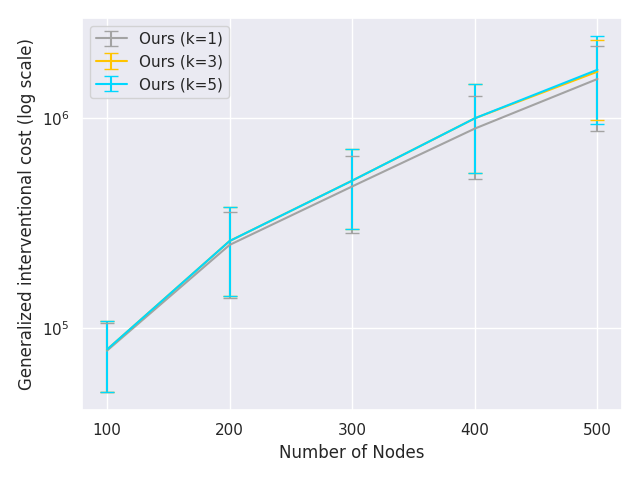}
    \caption{Generalized cost (log scale)}
\end{subfigure}
\begin{subfigure}[t]{0.23\linewidth}
    \centering
    \includegraphics[width=\linewidth]{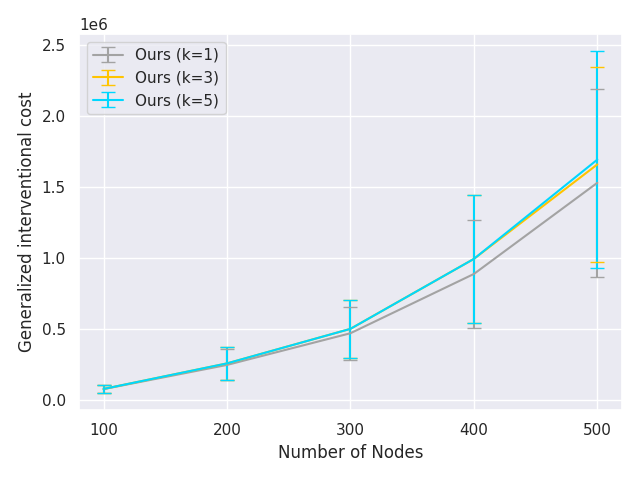}
    \caption{Generalized cost}
\end{subfigure}
\begin{subfigure}[t]{0.23\linewidth}
    \centering
    \includegraphics[width=\linewidth]{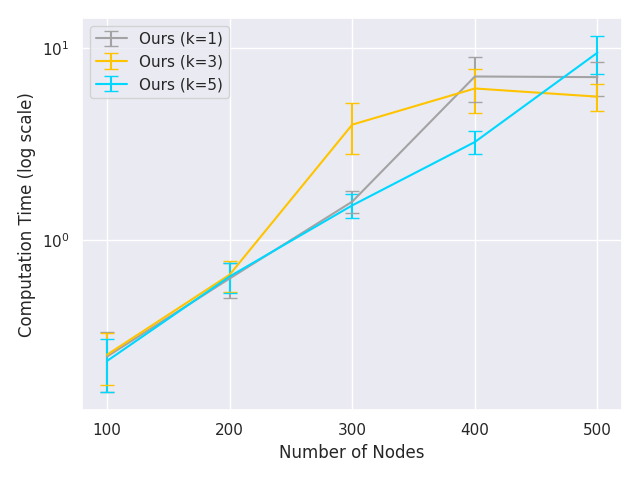}
    \caption{Time taken, in secs (log scale)}
\end{subfigure}
\begin{subfigure}[t]{0.23\linewidth}
    \centering
    \includegraphics[width=\linewidth]{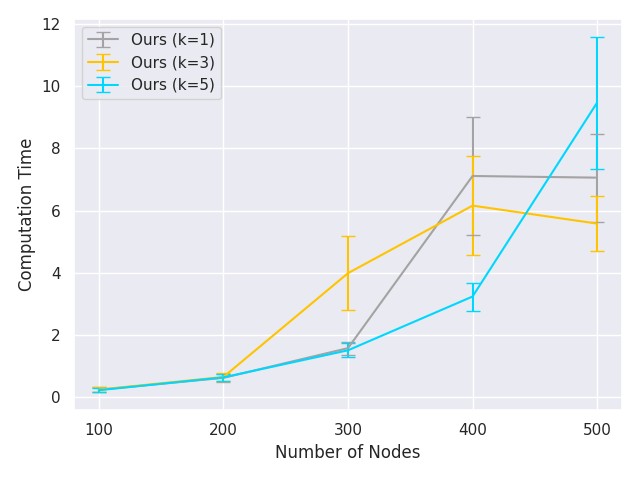}
    \caption{Time taken, in secs}
\end{subfigure}
\caption{Experiment 5, Type 1, $\alpha = 1$, $\beta = 1$}
\label{fig:exp5_type1_alpha1_beta1}
\end{figure}

\begin{figure}[htbp]
\centering
\begin{subfigure}[t]{0.23\linewidth}
    \centering
    \includegraphics[width=\linewidth]{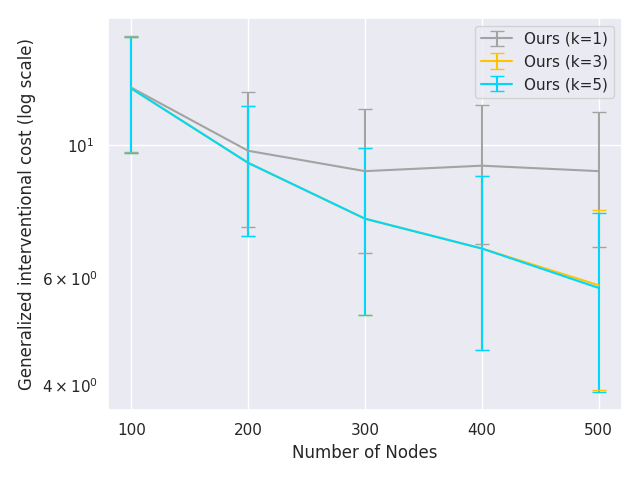}
    \caption{Generalized cost (log scale)}
\end{subfigure}
\begin{subfigure}[t]{0.23\linewidth}
    \centering
    \includegraphics[width=\linewidth]{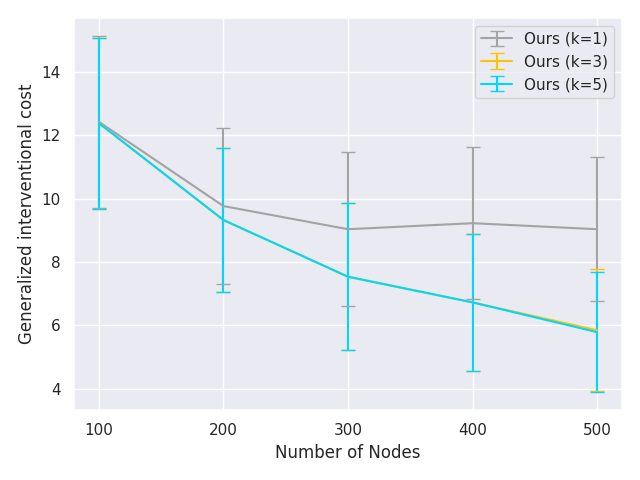}
    \caption{Generalized cost}
\end{subfigure}
\begin{subfigure}[t]{0.23\linewidth}
    \centering
    \includegraphics[width=\linewidth]{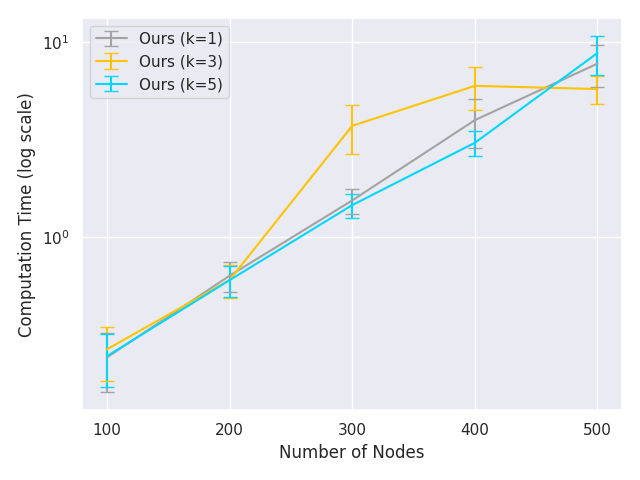}
    \caption{Time taken, in secs (log scale)}
\end{subfigure}
\begin{subfigure}[t]{0.23\linewidth}
    \centering
    \includegraphics[width=\linewidth]{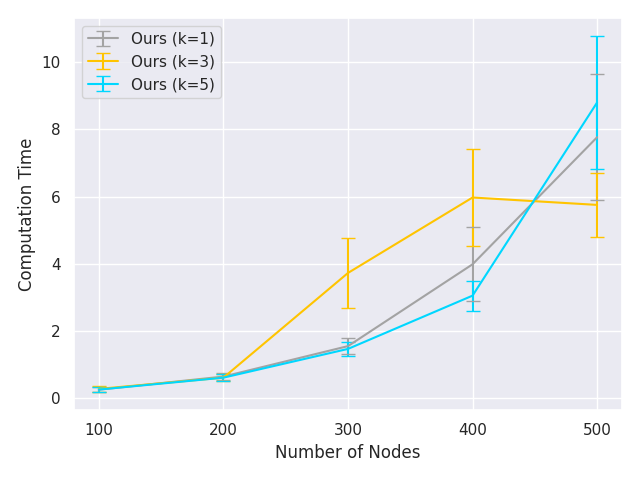}
    \caption{Time taken, in secs}
\end{subfigure}
\caption{Experiment 5, Type 2, $\alpha = 0$, $\beta = 1$}
\label{fig:exp5_type2_alpha0_beta1}
\end{figure}

\begin{figure}[htbp]
\centering
\begin{subfigure}[t]{0.23\linewidth}
    \centering
    \includegraphics[width=\linewidth]{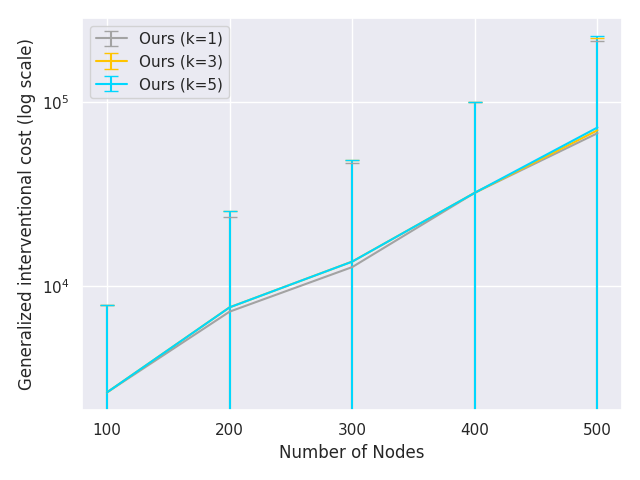}
    \caption{Generalized cost (log scale)}
\end{subfigure}
\begin{subfigure}[t]{0.23\linewidth}
    \centering
    \includegraphics[width=\linewidth]{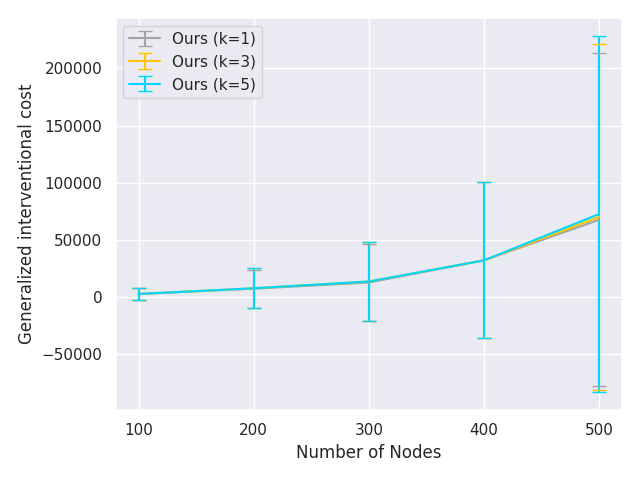}
    \caption{Generalized cost}
\end{subfigure}
\begin{subfigure}[t]{0.23\linewidth}
    \centering
    \includegraphics[width=\linewidth]{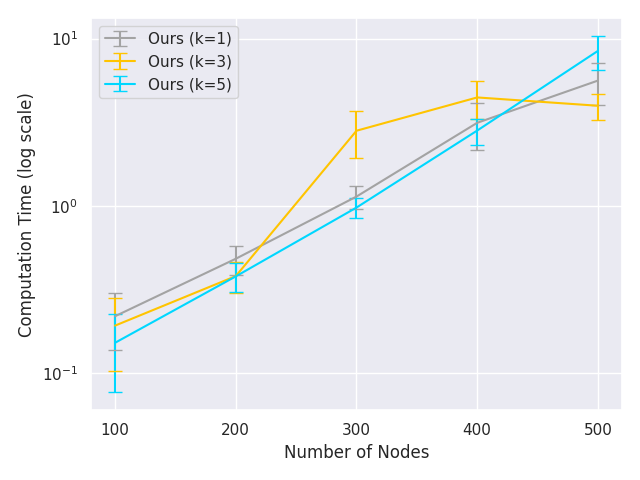}
    \caption{Time taken, in secs (log scale)}
\end{subfigure}
\begin{subfigure}[t]{0.23\linewidth}
    \centering
    \includegraphics[width=\linewidth]{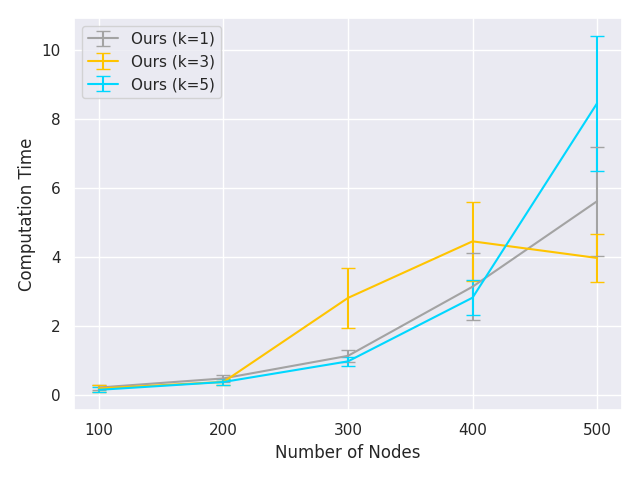}
    \caption{Time taken, in secs}
\end{subfigure}
\caption{Experiment 5, Type 2, $\alpha = 1$, $\beta = 1$}
\label{fig:exp5_type2_alpha1_beta1}
\end{figure}

%% file: main.bbl
\newcommand{\etalchar}[1]{$^{#1}$}
\begin{thebibliography}{dCCCM19}

\bibitem[AKMM20]{addanki2020efficient}
Raghavendra Addanki, Shiva Kasiviswanathan, Andrew McGregor, and Cameron Musco.
\newblock {Efficient Intervention Design for Causal Discovery with Latents}.
\newblock In {\em International Conference on Machine Learning}, pages 63--73.
  PMLR, 2020.

\bibitem[AMP97]{andersson1997characterization}
Steen~A. Andersson, David Madigan, and Michael~D. Perlman.
\newblock {A characterization of Markov equivalence classes for acyclic
  digraphs}.
\newblock {\em The Annals of Statistics}, 25(2):505--541, 1997.

\bibitem[CBP16]{cho2016reconstructing}
Hyunghoon Cho, Bonnie Berger, and Jian Peng.
\newblock {Reconstructing Causal Biological Networks through Active Learning}.
\newblock {\em {PLoS ONE}}, 11(3):e0150611, 2016.

\bibitem[Chi95]{chickering2013transformational}
David~Maxwell Chickering.
\newblock {A Transformational Characterization of Equivalent Bayesian Network
  Structures}.
\newblock In {\em Proceedings of the Eleventh Conference on Uncertainty in
  Artificial Intelligence}, UAI'95, page 87–98, San Francisco, CA, USA, 1995.
  Morgan Kaufmann Publishers Inc.

\bibitem[CS23]{choo2023subset}
Davin Choo and Kirankumar Shiragur.
\newblock {Subset verification and search algorithms for causal DAGs}.
\newblock In {\em International Conference on Artificial Intelligence and
  Statistics}, 2023.

\bibitem[CSB22]{choo2022verification}
Davin Choo, Kirankumar Shiragur, and Arnab Bhattacharyya.
\newblock {Verification and search algorithms for causal DAGs}.
\newblock {\em Advances in Neural Information Processing Systems}, 35, 2022.

\bibitem[dCCCM19]{de2019combining}
Luis~M. de~Campos, Andr{\'e}s Cano, Javier~G. Castellano, and Seraf{\'\i}n
  Moral.
\newblock {Combining gene expression data and prior knowledge for inferring
  gene regulatory networks via Bayesian networks using structural
  restrictions}.
\newblock {\em Statistical Applications in Genetics and Molecular Biology},
  18(3), 2019.

\bibitem[Ebe07]{eberhardt2007causation}
Frederick Eberhardt.
\newblock {Causation and Intervention}.
\newblock {\em Unpublished doctoral dissertation, Carnegie Mellon University},
  page~93, 2007.

\bibitem[Ebe10]{eberhardt2010causal}
Frederick Eberhardt.
\newblock {Causal Discovery as a Game}.
\newblock In {\em Causality: Objectives and Assessment}, pages 87--96. PMLR,
  2010.

\bibitem[EGS05]{eberhardt2005number}
Frederick Eberhardt, Clark Glymour, and Richard Scheines.
\newblock {On the number of experiments sufficient and in the worst case
  necessary to identify all causal relations among N variables}.
\newblock In {\em Proceedings of the Twenty-First Conference on Uncertainty in
  Artificial Intelligence}, pages 178--184, 2005.

\bibitem[EGS06]{eberhardt2006n}
Frederick Eberhardt, Clark Glymour, and Richard Scheines.
\newblock {N-1 Experiments Suffice to Determine the Causal Relations Among N
  Variables}.
\newblock In {\em Innovations in machine learning}, pages 97--112. Springer,
  2006.

\bibitem[ES07]{eberhardt2007interventions}
Frederick Eberhardt and Richard Scheines.
\newblock {Interventions and Causal Inference}.
\newblock {\em Philosophy of science}, 74(5):981--995, 2007.

\bibitem[GKS{\etalchar{+}}19]{greenewald2019sample}
Kristjan Greenewald, Dmitriy Katz, Karthikeyan Shanmugam, Sara Magliacane,
  Murat Kocaoglu, Enric Boix-Adser\`{a}, and Guy Bresler.
\newblock {Sample Efficient Active Learning of Causal Trees}.
\newblock {\em Advances in Neural Information Processing Systems}, 32, 2019.

\bibitem[GRE84]{gilbert1984separatorchordal}
John~R. Gilbert, Donald~J. Rose, and Anders Edenbrandt.
\newblock {A Separator Theorem for Chordal Graphs}.
\newblock {\em SIAM Journal on Algebraic Discrete Methods}, 5(3):306--313,
  1984.

\bibitem[GSKB18]{ghassami2018budgeted}
AmirEmad Ghassami, Saber Salehkaleybar, Negar Kiyavash, and Elias Bareinboim.
\newblock {Budgeted Experiment Design for Causal Structure Learning}.
\newblock In {\em International Conference on Machine Learning}, pages
  1724--1733. PMLR, 2018.

\bibitem[HB14]{hauser2014two}
Alain Hauser and Peter B\"{u}hlmann.
\newblock {Two Optimal Strategies for Active Learning of Causal Models from
  Interventions}.
\newblock {\em International Journal of Approximate Reasoning}, 55(4):926--939,
  2014.

\bibitem[HEH13]{hyttinen2013experiment}
Antti Hyttinen, Frederick Eberhardt, and Patrik~O. Hoyer.
\newblock {Experiment Selection for Causal Discovery}.
\newblock {\em Journal of Machine Learning Research}, 14:3041--3071, 2013.

\bibitem[HG08]{he2008active}
Yang-Bo He and Zhi Geng.
\newblock {Active Learning of Causal Networks with Intervention Experiments and
  Optimal Designs}.
\newblock {\em Journal of Machine Learning Research}, 9:2523--2547, 2008.

\bibitem[HLV14]{hu2014randomized}
Huining Hu, Zhentao Li, and Adrian Vetta.
\newblock {Randomized Experimental Design for Causal Graph Discovery}.
\newblock {\em Advances in Neural Information Processing Systems}, 27, 2014.

\bibitem[Hoo90]{hoover1990logic}
Kevin~D Hoover.
\newblock {The logic of causal inference: Econometrics and the Conditional
  Analysis of Causation}.
\newblock {\em Economics \& Philosophy}, 6(2):207--234, 1990.

\bibitem[KDV17]{kocaoglu2017cost}
Murat Kocaoglu, Alex Dimakis, and Sriram Vishwanath.
\newblock {Cost-Optimal Learning of Causal Graphs}.
\newblock In {\em International Conference on Machine Learning}, pages
  1875--1884. PMLR, 2017.

\bibitem[KF09]{koller2009probabilistic}
Daphne Koller and Nir Friedman.
\newblock {\em Probabilistic graphical models: principles and techniques}.
\newblock MIT press, 2009.

\bibitem[KWJ{\etalchar{+}}04]{king2004functional}
Ross~D. King, Kenneth~E. Whelan, Ffion~M. Jones, Philip G.~K. Reiser,
  Christopher~H. Bryant, Stephen~H. Muggleton, Douglas~B. Kell, and Stephen~G.
  Oliver.
\newblock {Functional genomic hypothesis generation and experimentation by a
  robot scientist}.
\newblock {\em Nature}, 427(6971):247--252, 2004.

\bibitem[LKDV18]{lindgren2018experimental}
Erik~M. Lindgren, Murat Kocaoglu, Alexandros~G. Dimakis, and Sriram Vishwanath.
\newblock {Experimental Design for Cost-Aware Learning of Causal Graphs}.
\newblock {\em Advances in Neural Information Processing Systems}, 31, 2018.

\bibitem[Mee95]{meek1995}
Christopher Meek.
\newblock {Causal Inference and Causal Explanation with Background Knowledge}.
\newblock In {\em Proceedings of the Eleventh Conference on Uncertainty in
  Artificial Intelligence}, UAI'95, page 403–410, San Francisco, CA, USA,
  1995. Morgan Kaufmann Publishers Inc.

\bibitem[NSMV18]{ness2018bayesian}
Robert~Osazuwa Ness, Karen Sachs, Parag Mallick, and Olga Vitek.
\newblock {A Bayesian Active Learning Experimental Design for Inferring
  Signaling Networks}.
\newblock {\em Journal of Computational Biology: a Journal of Computational
  Molecular Cell Biology}, 25(7):709--725, 2018.

\bibitem[Pea09]{pearl2009causality}
Judea Pearl.
\newblock {\em {Causality: Models, Reasoning and Inference}}.
\newblock Cambridge University Press, USA, 2nd edition, 2009.

\bibitem[POS{\etalchar{+}}18]{pingault2018using}
Jean-Baptiste Pingault, Paul~F O'reilly, Tabea Schoeler, George~B Ploubidis,
  Fr{\"u}hling Rijsdijk, and Frank Dudbridge.
\newblock {Using genetic data to strengthen causal inference in observational
  research}.
\newblock {\em Nature Reviews Genetics}, 19(9):566--580, 2018.

\bibitem[PSS22]{porwal2021almost}
Vibhor Porwal, Piyush Srivastava, and Gaurav Sinha.
\newblock {Almost Optimal Universal Lower Bound for Learning Causal DAGs with
  Atomic Interventions}.
\newblock In {\em International Conference on Artificial Intelligence and
  Statistics}, 2022.

\bibitem[Rei56]{reichenbach1956direction}
Hans Reichenbach.
\newblock {\em {The Direction of Time}}, volume~65.
\newblock University of California Press, 1956.

\bibitem[RHT{\etalchar{+}}17]{rotmensch2017learning}
Maya Rotmensch, Yoni Halpern, Abdulhakim Tlimat, Steven Horng, and David
  Sontag.
\newblock {Learning a Health Knowledge Graph from Electronic Medical Records}.
\newblock {\em Scientific reports}, 7(1):1--11, 2017.

\bibitem[RS02]{richardson2002ancestral}
Thomas Richardson and Peter Spirtes.
\newblock Ancestral graph markov models.
\newblock {\em The Annals of Statistics}, 30(4):962--1030, 2002.

\bibitem[RW06]{rubin2006estimating}
Donald~B Rubin and Richard~P Waterman.
\newblock {Estimating the Causal Effects of Marketing Interventions Using
  Propensity Score Methodology}.
\newblock {\em Statistical Science}, pages 206--222, 2006.

\bibitem[SC17]{sverchkov2017review}
Yuriy Sverchkov and Mark Craven.
\newblock {A review of active learning approaches to experimental design for
  uncovering biological networks}.
\newblock {\em PLoS computational biology}, 13(6):e1005466, 2017.

\bibitem[SKDV15]{shanmugam2015learning}
Karthikeyan Shanmugam, Murat Kocaoglu, Alexandros~G. Dimakis, and Sriram
  Vishwanath.
\newblock {Learning Causal Graphs with Small Interventions}.
\newblock {\em Advances in Neural Information Processing Systems}, 28, 2015.

\bibitem[SMG{\etalchar{+}}20]{squires2020active}
Chandler Squires, Sara Magliacane, Kristjan Greenewald, Dmitriy Katz, Murat
  Kocaoglu, and Karthikeyan Shanmugam.
\newblock {Active Structure Learning of Causal DAGs via Directed Clique Trees}.
\newblock {\em Advances in Neural Information Processing Systems},
  33:21500--21511, 2020.

\bibitem[Tia16]{tian2016bayesian}
Tianhai Tian.
\newblock {Bayesian Computation Methods for Inferring Regulatory Network Models
  Using Biomedical Data}.
\newblock {\em Translational Biomedical Informatics: A Precision Medicine
  Perspective}, pages 289--307, 2016.

\bibitem[VP90]{verma1990}
Thomas Verma and Judea Pearl.
\newblock {Equivalence and Synthesis of Causal Models}.
\newblock In {\em Proceedings of the Sixth Annual Conference on Uncertainty in
  Artificial Intelligence}, UAI '90, page 255–270, USA, 1990. Elsevier
  Science Inc.

\bibitem[WBL21a]{wienobst2021extendability}
Marcel Wien{\"o}bst, Max Bannach, and Maciej Li{\'s}kiewicz.
\newblock {Extendability of Causal Graphical Models: Algorithms and
  Computational Complexity}.
\newblock In {\em Uncertainty in Artificial Intelligence}, pages 1248--1257.
  PMLR, 2021.

\bibitem[WBL21b]{wienobst2021polynomial}
Marcel Wien{\"o}bst, Max Bannach, and Maciej Li{\'s}kiewicz.
\newblock {Polynomial-Time Algorithms for Counting and Sampling Markov
  Equivalent DAGs}.
\newblock In {\em Proccedings of the 35th Conference on Artificial
  Intelligence, AAAI}, 2021.

\bibitem[Woo05]{woodward2005making}
James Woodward.
\newblock {\em {Making Things Happen: A Theory of Causal Explanation}}.
\newblock Oxford University Press, 2005.

\bibitem[Yao77]{yao1977probabilistic}
Andrew Chi-Chin Yao.
\newblock {Probabilistic computations: Toward a unified measure of complexity}.
\newblock In {\em 2013 IEEE 54th Annual Symposium on Foundations of Computer
  Science}, pages 222--227. IEEE Computer Society, 1977.

\end{thebibliography}
